\documentclass[11pt]{article}

\usepackage{epsfig,epsf,fancybox}
\usepackage{amsmath}
\usepackage{mathrsfs}
\usepackage{amssymb}
\usepackage{graphicx}
\usepackage{color}
\usepackage{multirow}
\usepackage{paralist}
\usepackage{verbatim}
\usepackage{galois}
\usepackage{algorithm}
\usepackage[noend]{algorithmic}
\usepackage{boxedminipage}
\usepackage{booktabs}
\usepackage{accents}
\usepackage{stmaryrd}
\usepackage[table]{xcolor}
\usepackage{hhline}

\usepackage{subfig}

\usepackage{natbib}

\usepackage{url}
\usepackage[colorlinks,linkcolor=magenta,citecolor=blue, pagebackref=true,backref=true]{hyperref}
\renewcommand*{\backrefalt}[4]{%
    \ifcase #1 \footnotesize{(Not cited.)}%
    \or        \footnotesize{(Cited on page~#2.)}%
    \else      \footnotesize{(Cited on pages~#2.)}%
    \fi}

\newtheorem{theorem}{Theorem}[section]

\newtheorem{lemma}[theorem]{Lemma}
\newtheorem{proposition}[theorem]{Proposition}

\newtheorem{definition}[theorem]{Definition}

\newtheorem{remark}[theorem]{Remark}
\newtheorem{assumption}[theorem]{Assumption}

\numberwithin{equation}{section}

\newcommand{\EE}{\mathbb{E}}

\newcommand{\grad}{\nabla}
\newcommand{\gradx}{\nabla_{\mathbf x}}
\newcommand{\grady}{\nabla_{\mathbf y}}
\newcommand{\subg}{\partial}
\newcommand{\subgx}{\partial_{\mathbf x}}
\newcommand{\subgy}{\partial_{\mathbf y}}
\newcommand{\gx}{\textbf{g}_{\mathbf x}}
\newcommand{\gy}{\textbf{g}_{\mathbf y}}
\newcommand{\stocgx}{\hat{\bf g}_{\mathbf x}}
\newcommand{\stocgy}{\hat{\bf g}_{\mathbf y}}

\newcommand{\proj}{\mathcal{P}}

\newcommand{\x}{\mathbf x}
\newcommand{\y}{\mathbf y}

\newcommand{\sa}{\mathbf a}
\newcommand{\g}{\mathbf g}

\newcommand{\w}{\mathbf w}

\newcommand{\argmin}{\mathop{\rm argmin}}
\newcommand{\argmax}{\mathop{\rm argmax}}

\newcommand{\OCal}{\mathcal{O}}

\newcommand{\XCal}{\mathcal{X}}
\newcommand{\YCal}{\mathcal{Y}}

\newcommand{\prox}{\textnormal{prox}}

\newcommand{\br}{\mathbb{R}}

\newcommand{\ba}{\begin{array}}
\newcommand{\ea}{\end{array}}

\begin{document}

\begin{center}

{\bf{\LARGE{Two-Timescale Gradient Descent Ascent Algorithms for \\ [.2cm] Nonconvex Minimax Optimization}}}

\vspace*{.2in}
{\large{
\begin{tabular}{c}
Tianyi Lin$^\ddagger$ \and Chi Jin$^\square$ \and Michael I. Jordan$^{\diamond, \dagger}$ \\
\end{tabular}
}}

\vspace*{.2in}

\begin{tabular}{c}
Department of Electrical Engineering and Computer Sciences$^\diamond$ \\
Department of Statistics$^\dagger$ \\ 
University of California, Berkeley \\
Department of Industrial Engineering and Operations Research, Columbia University$^\ddagger$ \\
Department of Electrical and Computer Engineering, Princeton University$^\square$
\end{tabular}

\vspace*{.2in}

\today

\vspace*{.2in}

\begin{abstract}
We provide a unified analysis of \textit{two-timescale} gradient descent ascent (TTGDA) for solving structured nonconvex minimax optimization problems in the form of $\min_\x \max_{\y \in \YCal} f(\x, \y)$, where the objective function $f(\x, \y)$ is nonconvex in $\x$ and concave in $\y$, and the constraint set $\YCal \subseteq \br^n$ is convex and bounded. In the convex-concave setting, the single-timescale gradient descent ascent (GDA) algorithm is widely used in applications and has been shown to have strong convergence guarantees. In more general settings, however, it can fail to converge. Our contribution is to design TTGDA algorithms that are effective beyond the convex-concave setting, efficiently finding a stationary point of the function $\Phi(\cdot) := \max_{\y\in\YCal} f(\cdot, \y)$. We also establish theoretical bounds on the complexity of solving both \emph{smooth} and \emph{nonsmooth} nonconvex-concave minimax optimization problems. To the best of our knowledge, this is the first systematic analysis of TTGDA for nonconvex minimax optimization, shedding light on its superior performance in training generative adversarial networks (GANs) and in other real-world application problems. 
\end{abstract}
\end{center}

\section{Introduction}
In the wake of von Neumann's seminal work~\citep{Neumann-1928-Theorie}, the problem of finding a saddle-point solution to minimax optimization problems, under various assumptions on the objective functions, has been a major focus of research in many areas, including economics, control theory, operations research and computer science~\citep{Basar-1999-Dynamic,Nisan-2007-Algorithmic,Von-2007-Theory}. Recently, minimax optimization theory has begun to see applications in machine learning, with examples including generative adversarial networks (GANs)~\citep{Goodfellow-2014-Generative}, robust statistics~\citep{Xu-2009-Robustness, Abadeh-2015-Distributionally}, online adversarial learning~\citep{Cesa-2006-Prediction}, robust training of deep neural networks~\citep{Sinha-2018-Certifiable} and distributed computing over networks~\citep{Shamma-2008-Cooperative, Mateos-2010-Distributed}. Current trends in multi-agent learning with large-scale economic systems seem likely to further increase the need for efficient algorithms for computing minimax optima and related equilibrium solutions.  
\begin{table}[!t]
\begin{tabular}{cc}
\begin{minipage}{.48\textwidth}
\begin{algorithm}[H]\small
\caption{TTGDA}\label{Algorithm:TTGDA}
\begin{algorithmic}
\STATE \textbf{Input:} $(\x_0, \y_0)$, $(\eta_\x^t, \eta_\y^t)_{t \geq 0}$ with $\eta_\x^t \ll \eta_\y^t$. 
\FOR{$t = 1, 2, \ldots, T$}
\STATE $(\gx^{t-1}, \gy^{t-1}) \in \subg f(\x_{t-1}, \y_{t-1})$. 
\STATE $\x_t \leftarrow \x_{t-1} - \eta_\x^{t-1} \gx^{t-1}$,
\STATE $\y_t \leftarrow \proj_\YCal(\y_{t-1} + \eta_\y^{t-1} \gy^{t-1})$. 
\ENDFOR
\STATE Randomly draw $\hat{\x}$ from $\{\x_t\}_{t=0}^T$ at uniform. 
\STATE \textbf{Return:} $\hat{\x}$. 
\end{algorithmic}
\end{algorithm}
\end{minipage} &
\begin{minipage}{.52\textwidth}
\begin{algorithm}[H]\small
\caption{TTSGDA}\label{Algorithm:TTSGDA}
\begin{algorithmic}
\STATE \textbf{Input:} $(\x_0, \y_0)$, $M$, $(\eta_\x^t, \eta_\y^t)_{t \geq 0}$ with $\eta_\x^t \ll \eta_\y^t$.  
\FOR{$t = 1, 2, \ldots, T$}
\STATE $(\stocgx^{t-1}, \stocgy^{t-1}) = \textsf{SG}(G, \{\xi_i^{t-1}\}_{i=1}^M, \x_{t-1}, \y_{t-1})$. 
\STATE $\x_t \leftarrow \x_{t-1} - \eta_\x^{t-1} \stocgx^{t-1}$. 
\STATE $\y_t \leftarrow \proj_\YCal(\y_{t-1} + \eta_\y^{t-1} \stocgy^{t-1})$. 
\ENDFOR
\STATE Randomly draw $\hat{\x}$ from $\{\x_t\}_{t=0}^T$ at uniform. 
\STATE \textbf{Return:} $\hat{\x}$. 
\end{algorithmic}
\end{algorithm}
\end{minipage}
\end{tabular}
\end{table}

Formally, we study the following minimax optimization problem:
\begin{equation}\label{prob:main}
\min_{\x \in \br^m} \max_{\y \in \YCal} \ f(\x, \y), 
\end{equation}
where $f: \br^m \times \br^n \mapsto \br$ is the objective function and $\YCal$ is a constraint set. One of simplest methods for solving the minimax optimization problems is to generalize gradient descent (GD) so that each loop comprises both a descent step and an ascent step. The algorithm, referred to as \emph{gradient descent ascent} (GDA), performs one-step gradient descent over the min variable $\x$ with a learning rate $\eta_\x > 0$ and one-step gradient ascent over the max variable $\y$ with another learning rate $\eta_\y > 0$. 

On the positive side, when $f$ is convex in $\x$ and concave in $\y$, there is a vast literature establishing asymptotic and nonasymptotic convergence for average iterates generated by classical single-timescale GDA (i.e., $\eta_\x = \eta_\y$)~\citep[see, e.g.,][]{Korpelevich-1976-Extragradient, Chen-1997-Convergence, Nemirovski-2004-Prox, Auslender-2009-Projected, Nedic-2009-Subgradient}. In addition, local linear convergence can be established under an additional assumption that $f$ is locally strongly convex in $\x$ and locally strongly concave in $\y$~\citep{Cherukuri-2017-Saddle, Adolphs-2019-Local}. In contrast, there has been no shortage of negative results that highlight the fact that in the general settings, the single-timescale GDA algorithm might converge to limit cycles or even diverge~\citep{Mertikopoulos-2018-Cycles}.  

A recent line of research has focused on alternative gradient-based algorithms that have guarantees beyond the convex-concave setting~\citep{Daskalakis-2018-Training, Heusel-2017-Gans, Mertikopoulos-2019-Optimistic}. Notably, the two-timescale GDA (TTGDA), which incorporates unequal learning rates ($\eta_\x \neq \eta_\y$),  has been shown empirically to alleviate problems with limit circles; indeed, it attains theoretical support in terms of local asymptotic convergence (see~\citet[Theorem~2]{Heusel-2017-Gans} for the formal statements and proofs). For sequential problems such as robust learning, where the natural order of min-max is important (i.e., min-max is not equal to max-min), practitioners often prefer the faster convergence for the max problem. As such, there can be certain practical motivations for choosing $\eta_\x \ll \eta_\y$ in addition to those relating to nonasymptotic convergence guarantees.

We provide detailed pseudocode for TTGDA and its stochastic counterpart (TTSGDA) in Algorithms~\ref{Algorithm:TTGDA} and~\ref{Algorithm:TTSGDA}. Note that Algorithm~\ref{Algorithm:TTSGDA} repeatedly samples while remaining at the current iterate and the subroutine $(\stocgx^t, \stocgy^t) = \textsf{SG}(G, \{\xi_i^t\}_{i=1}^M, \x_t, \y_t)$ can be implemented in general by the updates 
\begin{equation}
\stocgx^t = \tfrac{1}{M}\left(\sum_{i=1}^M G_\x(\x_t, \y_t, \xi_i^t)\right), \quad \stocgy^t = \tfrac{1}{M}\left(\sum_{i=1}^M G_\y(\x_t, \y_t, \xi_i^t)\right), 
\end{equation}
where $G = (G_\x, G_\y)$ is unbiased and has bounded variance such that
\begin{equation}\label{Assumption:SGDA}
\EE[G(\x, \y, \xi) \mid \x, \y] \in \subg f(\x, \y), \qquad \EE[\|G(\x, \y, \xi) - \EE[G(\x, \y, \xi)]\|^2 \mid \x, \y] \leq \sigma^2. 
\end{equation}
In the current paper, we aim to provide a theoretical analysis of these algorithms, extending the asymptotic analysis in~\citet[Theorem~2]{Heusel-2017-Gans} to a nonasymptotic analysis that characterizes algorithmic efficiency.  Notably, we pursue our analysis in a general setting in which $f(\x, \cdot)$ is concave for any $\x$ and $\YCal$ is a convex and bounded set. This setting arises in many practical applications, including robust training of a deep neural network and robust learning of a classifier from multiple distributions. Both of these models can be reformulated as nonconvex-concave minimax optimization problems. 

Our main results in this paper are presented in four theorems. In particular, the first and second theorems (Theorems~\ref{Thm:nsc-smooth} and~\ref{Thm:nc-smooth}) provide a complete characterization of TTGDA and TTSGDA when applied to solve \textit{smooth} minimax optimization problems. In the nonconvex-strongly-concave setting, TTGDA and TTSGDA require $O(\kappa^2\epsilon^{-2})$ and $O(\kappa^3\epsilon^{-4})$ (stochastic) gradient evaluations to find an $\epsilon$-stationary point (cf. Definition~\ref{def:notion-smooth}) where $\kappa$ is a condition number. In the nonconvex-concave setting, TTGDA and TTSGDA require $O(\epsilon^{-6})$ and $O(\epsilon^{-8})$ (stochastic) gradient evaluations to find an $\epsilon$-stationary point (cf. Definition~\ref{def:notion-nonsmooth}). While these results have appeared in our preliminary work~\citep{Lin-2020-Gradient}, the analysis that we provide here is simpler and more direct. 

Another two theorems (Theorems~\ref{Thm:nsc-nonsmooth} and~\ref{Thm:nc-nonsmooth}) are the more substantive contributions of this paper. These theorems show that TTGDA and TTSGDA can efficiently find an $\epsilon$-stationary point (cf. Definition~\ref{def:notion-nonsmooth}) even when applied to solve \textit{nonsmooth} minimax optimization problems. The complexity bounds of $\tilde{O}(\epsilon^{-6})$ and $O(\epsilon^{-8})$ in terms of (stochastic) gradient evaluations can be established for the nonconvex-strongly-concave and nonconvex-concave settings. These results are useful since the aforementioned machine learning applications often have the nonsmooth objective functions (e.g., the use of ReLU neural network in GANs). Our theorems demonstrate that, when lack of smoothness occurs, TTGDA and TTSGDA may slow down, but the nonasymptotic complexity bounds will still be attained as long as a certain Lipschitz continuity property is retained and adaptive stepsizes are chosen carefully (i.e., our algorithms use $(\eta_\x^t, \eta_\y^t)_{t \geq 0}$ instead of $(\eta_\x, \eta_\y)$). 

Our proof technique in this paper is worth being highlighted. It is different from existing techniques for analyzing nested-loop algorithms as exemplified in~\citet{Nouiehed-2019-Solving} and~\citet{Jin-2020-Local}. Indeed, the previous frameworks are based on inexact gradient descent algorithms for minimizing a function $\Phi(\cdot) = \max_{\y \in \YCal} f(\cdot, \y)$, where a subroutine is designed for the inner problem of $\max_{\y \in \YCal} f(\x, \y)$ and which requires $\y_t$ to be close to $\y^\star(\x_t) = \argmax_{\y \in \YCal} f(\x_t, \y)$ at each iteration. Applying this approach to analyze TTGDA and TTSGDA is challenging since $\y_t$ might not be close enough to $\y^\star(\x_t)$, so that $\gradx f(\x_t, \y_t)$ does not approximate a descent direction well. In order to overcome this issue, we have developed a new proof technique that directly analyzes the concave optimization problem with a slowly changing objective function. 

Subsequent to our work in~\citet{Lin-2020-Gradient}, several new algorithms were developed for solving smooth and nonconvex-(strongly)-concave minimax optimization problems~\citep{Lin-2020-Near, Yang-2020-Catalyst, Ostrovskii-2021-Efficient, Yan-2020-Optimal, Zhang-2020-Single, Zhang-2021-Complexity, Li-2021-Complexity, Yang-2022-Faster, Zhao-2023-Primal, Boct-2023-Alternating, Xu-2023-Unified}. In the smooth and nonconvex-strongly-concave setting, a lower bound has been derived for gradient-based algorithms~\citep{Zhang-2021-Complexity, Li-2021-Complexity} and there exist various nested-loop \textit{near-optimal} deterministic algorithms~\citep{Lin-2020-Near, Ostrovskii-2021-Efficient, Kong-2021-Accelerated, Zhang-2021-Complexity}. In the smooth and nonconvex-concave setting, no lower bound is known, and the best-known upper bound has been achieved by various nested-loop algorithms~\citep{Thekumparampil-2019-Efficient, Lin-2020-Near, Yang-2020-Catalyst, Ostrovskii-2021-Efficient, Kong-2021-Accelerated, Zhao-2023-Primal}. Compared to deterministic counterparts, investigations of stochastic algorithms are scarce; indeed, there exists a gap between the best-known lower bound~\citep{Li-2021-Complexity} and upper bound~\citep{Yan-2020-Optimal}. 

A growing body of research has focused on \emph{single-loop} algorithms for nonconvex minimax optimization problems~\citep{Zhang-2020-Single, Yang-2022-Faster, Boct-2023-Alternating, Xu-2023-Unified}. For example,~\citet{Xu-2023-Unified} proposed an alternating GDA which converges in both nonconvex-concave and convex-nonconcave settings and showed that the required number of (stochastic) gradient evaluations to find an $\epsilon$-stationary point of $f$ is the best among all single-loop algorithms.~\citet{Boct-2023-Alternating} studied the alternating TTGDA and TTSGDA in a more general setting where $f(\cdot, \y)$ is Lipschitz and weakly convex but $f(\x, \cdot)$ is smooth. They showed that the algorithms can achieve the same complexity bound in this setting as in the smooth setting. These two papers extended our proof techniques to analyze the alternating TTGDA and TTSGDA but still assume the smoothness of $f(\x, \cdot)$. By contrast, the results in the current paper are the first to demonstrate the efficiency of TTGDA and TTSGDA in the nonsmooth setting. Both~\citet{Zhang-2020-Single} and~\citet{Yang-2022-Faster} proposed single-loop algorithms for solving other structured nonconvex-concave minimax optimization problems with theoretical guarantee. 

\paragraph{Notation.}  For $n \geq 2$, we let $[n]$ denote the set $\{1, 2, \ldots, n\}$ and $\br_+^n$ denote the set of all vectors in $\br^n$ with nonnegative coordinates. We also use bold lower-case letters to denote vectors (i.e., $\x$ and $\y$), and calligraphic upper-case letters to denote sets (i.e., $\XCal$ and $\YCal$). For a Lipschitz function $f$, we let $\subg f(\cdot)$ denote its Clarke subdifferential~\citep{Clarke-1990-Optimization}. For a differentiable function $f$, we let $\grad f(\cdot)$ denote its gradient. For a Lipschitz function $f(\cdot, \cdot)$ of two variables, we let $\subgx f(\cdot, \cdot)$ and $\subgy f(\cdot, \cdot)$ denote its partial Clarke subdifferentials with respect to the first and second variables. For a differentiable function $f(\cdot, \cdot)$, we let $\gradx f(\cdot, \cdot)$ and $\grady f(\cdot, \cdot)$ denote the partial gradients of $f$ with respect to the first and second variables. For a vector $\x$ and a matrix $X$, we let $\|\x\|$ and $\|X\|$ denote the $\ell_2$-norm of $\x$ and the spectral norm of $X$. For a set $\XCal$, we let $D = \max_{\x, \x' \in \XCal} \|\x - \x'\|$ denote its diameter and $\proj_\XCal$ denote its orthogonal projection. Lastly, we use $O(\cdot)$, $\Omega(\cdot)$ and $\Theta(\cdot)$ to hide absolute constants which do not depend on any problem parameter, and $\tilde{O}(\cdot)$, $\tilde{\Omega}(\cdot)$ and $\tilde{\Theta}(\cdot)$ to hide both absolute constants and log factors.

\section{Related Work}
Historically, an early concrete instantiation of minimax optimization involved computing a pair of probability vectors $\left(\x, \y\right)$ for solving the problem $\min_{\x \in \Delta^m} \max_{\y \in \Delta^n} \x^\top A \y$ for $A \in \br^{m \times n}$ and simplices $\Delta^m \subseteq \br^m$ and $\Delta^n \subseteq \br^n$. This bilinear minimax problem together with von Neumann's minimax theorem~\citep{Neumann-1928-Theorie} were a cornerstone in the development of game theory. A simple and generic algorithm was developed for solving this problem where the min and max players implemented a learning procedure in tandem~\citep{Robinson-1951-Iterative}. Subsequently,~\citet{Sion-1958-General} generalized von Neumann's result from bilinear games to general convex-concave games, i.e., $\min_\x \max_\y f(\x, \y) = \max_\y \min_\x f(\x, \y)$, and triggered a line of research on algorithmic design for convex-concave minimax optimization in continuous time~\citep{Kose-1956-Solutions, Cherukuri-2017-Saddle} and discrete time~\citep{Uzawa-1958-Iterative, Golshtein-1974-Generalized, Korpelevich-1976-Extragradient, Nemirovski-2004-Prox, Auslender-2009-Projected, Nedic-2009-Subgradient, Liang-2019-Interaction, Mokhtari-2020-Convergence, Mokhtari-2020-Unified, Azizian-2020-Tight}. This line of research has yielded two important conclusions: single-timescale GDA, with decreasing learning rates, can find an $\epsilon$-approximate saddle point within $O(\epsilon^{-2})$ iterations in the convex-concave setting~\citep{Auslender-2009-Projected}, and $O(\kappa^2 \log (1/\epsilon))$ iterations in the strongly-convex-strongly-concave setting~\citep{Liang-2019-Interaction}.

\paragraph{Nonconvex-concave setting.} Smooth and nonconvex-concave minimax optimization problems are a class of computationally tractable problems in the form of Eq.~\eqref{prob:main}. They have emerged as a focus in the optimization and machine learning communities~\citep{Sinha-2018-Certifiable, Rafique-2021-Weakly, Sanjabi-2018-Convergence, Grnarova-2018-An, Lu-2020-Hybrid, Nouiehed-2019-Solving, Thekumparampil-2019-Efficient, Ostrovskii-2021-Efficient, Kong-2021-Accelerated, Zhao-2023-Primal}. Among the existing works, we highlight the work of~\citet{Grnarova-2018-An}, who proposed a variant of GDA for this class of nonconvex minimax optimization problems and the works of~\citet{Sinha-2018-Certifiable} and~\citet{Sanjabi-2018-Convergence}, who studied inexact SGD algorithms (i.e., variants of SGDmax) and established a convergence guarantee in terms of iteration numbers. Subsequently,~\citet{Jin-2020-Local} proved the refined convergence results for GDmax and SGDmax in terms of (stochastic) gradient evaluations. 

There are some other algorithms developed for solving smooth and nonconvex-concave minimax optimization problems. For example,~\citet{Rafique-2021-Weakly} proposed two stochastic algorithms and prove that they efficiently find a stationary point of $\Phi(\cdot) = \max_{\y\in \YCal} f(\cdot, \y)$. However, these algorithms are nested-loop algorithms and the convergence guarantees hold only when $f(\x, \cdot)$ is linear and $\YCal = \br^n$ (see~\citet[Assumption 2]{Rafique-2021-Weakly}). \citet{Nouiehed-2019-Solving} have designed a multistep variant of GDA by incorporating momentum terms into the scheme.\footnote{\citet{Nouiehed-2019-Solving} considered a setting where (i) $\x$ is constrained to a convex set $\XCal$; (ii) $\YCal=\br^n$; and (iii) $f$ is smooth in $\x$ over $\XCal$ and satisfies the Polyak-\L{}ojasiewicz condition in $\y$.} In this context, the best-known gradient complexity is $\tilde{O}(\epsilon^{-3})$ and has been achieved by the ProxDIAG algorithm~\citep{Thekumparampil-2019-Efficient}. This complexity result was also achieved by an accelerated inexact proximal point algorithm. Indeed,~\citet{Kong-2021-Accelerated} showed that their algorithm finds one $\epsilon$-stationary point of $\Phi$ within $O(\epsilon^{-3})$ iterations.\footnote{\citet{Kong-2021-Accelerated} considered a more general setting than ours, where $\x$ is constrained to a convex set $\XCal$ and $f$ is smooth in $\x$ over $\XCal$. Their analysis and results are valid in our setting.} Solving the subproblem at each iteration is equivalent to minimizing a smooth and strongly convex function and can be approximately done using $\tilde{O}(1)$ gradient evaluations. Combining these two facts yields the desired claim. Subsequent to~\citet{Lin-2020-Gradient}, this complexity bound was also achieved by other algorithms~\citep{Lin-2020-Near, Yang-2020-Catalyst, Ostrovskii-2021-Efficient, Zhao-2023-Primal}. Despite the appealing theoretical convergence guarantee, the above algorithms are all nested-loop algorithms and thus can be complex to implement. One would like to understand whether the nested-loop structure is necessary or whether a single-loop algorithm provably converges in smooth and nonconvex-concave settings. An example of such an investigation is~\citet{Lu-2020-Hybrid}, where a single-loop algorithm was developed for solving smooth and nonconvex-concave minimax optimization problems. However, their analysis requires some restrictive conditions; e.g., that $f(\cdot, \cdot)$ is lower bounded. In contrast, we only require that $\max_{\y \in \YCal} f(\cdot, \y)$ is lower bounded. An example which meets our conditions rather than theirs is $\min_{\x \in \br}\max_{\y \in [-1, 1]} \x^\top\y$. Our less-restrictive conditions make the problem more challenging and our analysis accordingly differs significantly from theirs. 

To our knowledge, we are the first to analyze TTGDA and TTSGDA in the nonsmooth and nonconvex-concave setting where $f$ is Lipschitz and $f(\cdot, \y)$ is weakly convex. As already noted,~\citet{Rafique-2021-Weakly} considered the nonsmooth setting but focused on nested-loop algorithms.~\citet{Barazandeh-2020-Solving} and~\citet{Huang-2021-Efficient} proposed to study the composite nonconvex-concave setting with nonsmooth regularization terms which are however structured such that the associated proximal mappings can be computed efficiently. The closest work to ours is~\citet{Boct-2023-Alternating} who studied the alternating TTGDA and TTSGDA in the nonsmooth and nonconvex-concave setting. However, their setting is different from ours---they assume that $f(\cdot, \y)$ is Lipschitz and weakly convex but $f(\x, \cdot)$ is smooth. Their results demonstrate that the algorithms can achieve the same complexity bound in this setting as in the smooth setting. In contrast, as we will show in this paper, relaxing the smoothness of $f$ makes the problem substantially more difficult, especially for deterministic algorithms. 

\paragraph{Nonconvex-nonconcave setting.} The investigation of nonconvex-nonconcave minimax optimization problems has become a central topic in the machine learning community, inspired by the advent of generative adversarial networks (GANs)~\citep{Goodfellow-2014-Generative} and adversarial learning~\citep{Madry-2018-Towards, Sinha-2018-Certifiable}. In this context, most works focus on defining a notion of goodness or the development of new procedures for reducing oscillations~\citep{Daskalakis-2018-Limit, Adolphs-2019-Local, Jin-2020-Local} and speeding up the convergence of gradient dynamics~\citep{Heusel-2017-Gans, Balduzzi-2018-Mechanics, Mertikopoulos-2019-Optimistic}. For example,~\citet{Daskalakis-2018-Limit} showed that the stable fixed points of GDA are not necessarily saddle points, while~\citet{Adolphs-2019-Local} proposed Hessian-based algorithms where their stable fixed points are saddle points. In this context,~\citet{Balduzzi-2018-Mechanics} also developed a symplectic gradient adjustment algorithm for finding stable fixed points in potential games and Hamiltonian games. Notably,~\citet{Heusel-2017-Gans} have proposed TTGDA and TTSGDA formally for training GANs and showed that the saddle points are also stable fixed points of a continuous limit of these algorithms under strong conditions. All of these convergence results are either local or asymptotic and thus can not cover our nonasymptotic results for TTGDA and TTSGDA in the nonconvex-concave settings.~\citet{Mertikopoulos-2019-Optimistic} has provided nonasymptotic guarantees for a special class of nonconvex-nonconcave minimax optimization problems under variational stability and/or the so-called Minty condition. However, while these conditions must hold in convex-concave setting, they do not necessarily hold in the nonconvex-concave setting. Subsequent to our earlier work~\citep{Lin-2020-Gradient}, we are aware of significant effort that has been devoted to understanding the nonconvex-nonconcave minimax optimization problem~\citep{Yang-2020-Global, Diakonikolas-2021-Efficient, Liu-2021-First, Lee-2021-Semi, Ostrovskii-2021-Nonconvex, Fiez-2022-Minimax, Jiang-2023-Optimality, Grimmer-2023-Landscape, Daskalakis-2023-Stay, Zheng-2023-Universal, Cai-2024-Accelerated}. Some of their results are valid for the nonconvex-concave setting but can not cover our results in this paper. This is an active research area and we refer the interested reader to the aforementioned articles for more details.

\paragraph{Other settings.} Analysis in the online learning perspective aims at understanding whether an algorithm enjoys the no-regret property. For example, the optimistic GDA from~\citet{Daskalakis-2019-Last} is a no-regret learning algorithm, while the extragradient algorithm from~\citet{Mertikopoulos-2019-Optimistic} is not. Moreover, by carefully comparing limit behavior of algorithm dynamics,~\citet{Bailey-2018-Multiplicative} have shown that the multiplicative weights update behaves similarly to single-timescale GDA, thus justifying the need for introducing the optimistic update for achieving the last-iterate convergence guarantee. There are also several works on finite-sum minimax optimization and variance reduction~\citep{Luo-2020-Stochastic, Alacaoglu-2022-Stochastic, Huang-2022-Accelerated, Han-2024-Lower}. 

\section{Preliminaries and Technical Background}\label{sec:prelim}
We present standard definitions for Lipschitz functions and smooth functions.
\begin{definition}
$f$ is $L$-Lipschitz if for all $\x, \x'$, we have $\| f(\x) -  f (\x')\| \leq L\|\x-\x'\|$.
\end{definition} 
\begin{definition}
$f$ is $\ell$-smooth if for all $\x, \x'$, we have $\|\grad f(\x) - \grad f(\x')\| \leq \ell\|\x-\x'\|$.
\end{definition}
The minimax optimization problem in Eq.~\eqref{prob:main} can be rewritten as $\min_{\x \in \br^m} \Phi(\x)$ where $\Phi(\cdot) = \max_{\y \in \YCal} f(\cdot, \y)$. If $f(\x, \cdot)$ is concave for each $\x \in \br^m$, the problem of $\max_{\y \in \YCal} f(\x, \y)$ can be approximately solved efficiently, thus providing useful information about $\Phi$. However, it is impossible to find any reasonable approximation of the global minimum of $\Phi$ in general since $\Phi$ can be nonconvex and nonsmooth~\citep{Murty-1987-Some, Zhang-2020-Complexity}. This motivates us to pursue a local surrogate for the global minimum of $\Phi$. 

A common surrogate in nonconvex optimization is the notion of stationarity, which is appropriate if $\Phi$ is differentiable.
\begin{definition}\label{def:notion-smooth}
A point $\x$ is an $\epsilon$-stationary point of a differentiable function $\Phi$ for some tolerance $\epsilon \geq 0$ if $\|\grad\Phi(\x)\| \leq \epsilon$. If $\epsilon = 0$, we have that $\x$ is a stationary point. 
\end{definition}
Definition~\ref{def:notion-smooth} is suitable in the smooth and nonconvex-strongly-concave setting since $\Phi$ is differentiable. However, it is worth remarking that $\Phi$ is not necessarily differentiable in the other three settings: (i) smooth and nonconvex-concave; (ii) nonsmooth and nonconvex-strongly-concave; and (iii) nonsmooth and nonconvex-concave. Thus, we need to introduce a weaker condition as follows.
\begin{definition}
A function $\Phi$ is $\rho$-weakly convex if the function $\Phi(\cdot) + 0.5\rho\|\cdot\|^2$ is convex.
\end{definition}
For a $\rho$-weakly convex function $\Phi$, we know that its Clarke subdifferential $\partial\Phi$ is uniquely determined by the subdifferential of the convex function $\Phi(\cdot) + 0.5\rho\|\cdot\|^2$. This then implies that a naive measure of $\epsilon$-approximate stationarity can be defined as a point $\x$ satisfying $\min_{\xi \in \partial \Phi(\x)} \|\xi\| \leq \epsilon$ (i.e., at least one subgradient is small). However, this notion is restrictive for nonsmooth optimization. For example, when $\Phi(\cdot) = |\cdot|$ is a one-dimensional function, an $\epsilon$-stationary point is 0 for all $\epsilon \in [0, 1)$. Therefore, we see that finding an $\epsilon$-approximate stationary point of $\Phi$ under this notion is as hard as finding an exact stationary point even when $\epsilon > 0$ is not very small. 

In respond to this issue,~\citet{Davis-2019-Stochastic} employed an alternative notion of stationarity based on the Moreau envelope. Their notion can be suitable for our settings for two reasons: (i) the problem of $\max_{\y \in \YCal} f(\x, \y)$ can be approximately solved efficiently, thus providing useful information about $\Phi$ and its Moreau envelope; (ii) a stationary point based on the Moreau envelope has practical implications in real machine learning applications. For example, $\x$ and $\y$ stand for a classifier and an adversarial example arising from adversarial machine learning models. In general, practitioners are interested in finding a robust classifier $\x$ rather than recovering the adversarial example $\y$. 
\begin{definition}
A function $\Phi_\lambda$ is the Moreau envelope of $\Phi$ with the parameter $\lambda>0$ if $\Phi_\lambda(\cdot) = \min_\w \Phi(\w) + \frac{1}{2\lambda}\|\w - \cdot\|^2$. 
\end{definition}
In what follows, we prove two lemmas for the nonconvex-concave settings. For simplicity, we use $\Phi(\cdot) = \max_{\y \in \YCal} f(\cdot, \y)$ throughout in the subsequent statements. 
\begin{lemma}\label{Lemma:ME-smooth}
If $f(\x, \y)$ is $\ell$-smooth, concave in $\y$, and $\YCal$ is convex and bounded, we have that $\Phi$ is $\ell$-weakly convex and the following statements hold true, 
\begin{enumerate}
\item $\Phi(\prox_{\Phi/2\ell}(\x)) \leq \Phi(\x)$ where $\prox_{\Phi/2\ell}(\cdot) = \argmin_\w \Phi(\w) + \ell\|\w - \cdot\|^2$. 
\item $\Phi_{1/2\ell}$ is $\ell$-smooth with $\grad\Phi_{1/2\ell}(\x) = 2\ell(\x - \prox_{\Phi/2\ell}(\x))$. 
\end{enumerate} 
\end{lemma}
\begin{lemma}\label{Lemma:ME-nonsmooth}
If $f(\x, \y)$ is $\rho$-weakly convex in $\x$, concave in $\y$, and $\YCal$ is convex and bounded, we have that $\Phi$ is $\rho$-weakly convex and the following statements hold true, 
\begin{enumerate}
\item $\Phi(\prox_{\Phi/2\rho}(\x)) \leq \Phi(\x)$ where $\prox_{\Phi/2\rho}(\cdot) = \argmin_\w \Phi(\w) + \rho\|\w - \cdot\|^2$. 
\item $\Phi_{1/2\rho}$ is $\rho$-smooth with $\grad\Phi_{1/2\rho}(\x) = 2\rho(\x - \prox_{\Phi/2\rho}(\x))$. 
\end{enumerate} 
\end{lemma}
Based on Lemmas~\ref{Lemma:ME-smooth} and~\ref{Lemma:ME-nonsmooth}, an alternative measure of approximate stationarity of a $\rho$-weakly convex function $\Phi$ can be defined as a point $\x$ satisfying $\|\grad\Phi_{1/2\rho}(\x)\| \leq \epsilon$. 
\begin{definition}\label{def:notion-nonsmooth}
A point $\x$ is an $\epsilon$-stationary point of a $\rho$-weakly convex function $\Phi$ for some tolerance $\epsilon \geq 0$ if $\|\grad\Phi_{1/2\rho}(\x)\| \leq \epsilon$. If $\epsilon = 0$, we have that $\x$ is a stationary point.  
\end{definition}
Although Definition~\ref{def:notion-nonsmooth} uses the Moreau envelope $\Phi_{1/2\rho}$ to define stationarity, it connects to approximate stationarity of $\Phi$ in terms of small subgradients as follows. 
\begin{lemma}\label{Lemma:notion-nonsmooth}
If $\x$ is an $\epsilon$-stationary point of a $\rho$-weakly convex function $\Phi$, there exists $\hat{\x}$ such that $\min_{\xi \in \partial \Phi(\hat{\x})} \|\xi\| \leq \epsilon$ and $\|\x - \hat{\x}\| \leq \frac{\epsilon}{2\rho}$.  
\end{lemma}
Lemma~\ref{Lemma:notion-nonsmooth} shows that Definition~\ref{def:notion-nonsmooth} gives an intuitive surrogate for the naive measure of $\epsilon$-approximate stationarity. Indeed, if $\x$ is an $\epsilon$-stationary point of a $\rho$-weakly convex function $\Phi$, it must be $\frac{\epsilon}{2\rho}$-close to a point $\hat{\x}$ which has at least one small subgradient. 

There is another notion of stationarity based on $\grad f$ in smooth and nonconvex-concave setting~\citep{Nouiehed-2019-Solving, Lu-2020-Hybrid, Lin-2020-Near, Xu-2023-Unified}. A pair of points $(\x, \y) \in \br^m \times \YCal$ is an $\epsilon$-stationary point of $f$ if $\|\gradx f(\x, \y)\| \leq \epsilon$ and $\|\y^+ - \y\| \leq \frac{\epsilon}{\ell}$ where $\y^+$ is defined by 
\begin{equation*}
\y^+ = \proj_\YCal\left(\y+\tfrac{1}{\ell}\grady f(\x, \y)\right). 
\end{equation*}
Intuitively, the term $\ell\|\y^+ - \y\|$ stands for the distance between $\y \in \YCal$ and $\y^+$ obtained by performing a one-step projected partial gradient ascent at $(\x, \y)$ starting from $\y$. This is akin to the norm of the gradient mapping~\citep{Nesterov-2013-Gradient}. As highlighted by~\citet{Thekumparampil-2019-Efficient},  this notion is weaker than Definitions~\ref{def:notion-smooth} and~\ref{def:notion-nonsmooth}. For the sake of completeness, we provide a complete characterization of the relationship between these notions by showing that they can be translated in both directions with the extra computational cost. To pursue this goal, we introduce a new notion of stationarity based on $\grad f$ as follows.
\begin{definition}\label{def:stationary}
A pair of points $(\x, \y) \in \br^m \times \YCal$ is an $\epsilon$-stationary point of a function $f$ for some tolerance $\epsilon \geq 0$ if $\|\gradx f(\x, \y^+)\| \leq \epsilon$ and $\|\y^+ - \y\| \leq \tfrac{\epsilon}{\ell}$. 
\end{definition}
The stationarity of $f$ is in fact equivalent to Definition~\ref{def:stationary} up to a constant. Indeed, if $(\x, \y) \in \br^m \times \YCal$ is an $\epsilon$-stationary point as per Definition~\ref{def:stationary}, we have $\|\gradx f(\x, \y)\| \leq \|\gradx f(\x, \y^+)\| + \ell\|\y^+ - \y\| \leq 2\epsilon$. Conversely, if $(\x, \y) \in \br^m \times \YCal$ satisfies $\|\gradx f(\x, \y)\| \leq \epsilon$ and $\|\y^+ - \y\| \leq \frac{\epsilon}{\ell}$, we have $\|\gradx f(\x, \y^+)\| \leq \|\gradx f(\x, \y)\| + \ell\|\y^+ - \y\| \leq 2\epsilon$.

We focus on the notion in Definition~\ref{def:stationary} and relate it to the stationarity of $\Phi$ (cf. Definition~\ref{def:notion-smooth} and~\ref{def:notion-nonsmooth}). The following propositions summarize the details. 
\begin{proposition}\label{Prop:notion-smooth} 
Suppose that $f$ is $\ell$-smooth, $\mu$-strongly concave in $\y$, and $\YCal$ is convex and bounded. Given that $\hat{\x}$ is an $\epsilon$-stationary point as per Definition~\ref{def:notion-smooth}, we can obtain a $2\epsilon$-stationary point $(\x', \y')$ as per Definition~\ref{def:stationary} using $O(\kappa\log(1/\epsilon))$ gradient evaluations or $O(\epsilon^{-2})$ stochastic gradient evaluations. Given that $(\hat{\x}, \hat{\y})$ is a $\frac{\epsilon}{2\kappa}$-stationary point as per Definition~\ref{def:stationary}, we have that $\hat{\x}$ is an $2\epsilon$-stationary point as per Definition~\ref{def:notion-smooth}. 
\end{proposition}
\begin{proposition}\label{Prop:notion-nonsmooth} 
Suppose that $f$ is $\ell$-smooth, concave in $\y$, and $\YCal$ is convex and bounded. Given that $\hat{\x}$ is an $\epsilon$-stationary point as per Definition~\ref{def:notion-nonsmooth}, we can obtain a $4\epsilon$-stationary point $(\x', \y')$ as per Definition~\ref{def:stationary} using  $O(\epsilon^{-2})$ gradient evaluations or $O(\epsilon^{-4})$ stochastic gradient evaluations. Given that $(\hat{\x}, \hat{\y})$ is a $\frac{\epsilon^2}{\max\{1, \ell D\}}$-stationary point as per Definition~\ref{def:stationary}, we have that $\hat{\x}$ is a $(4\epsilon^2 + 4\epsilon)$-stationary point as per Definition~\ref{def:notion-nonsmooth}. 
\end{proposition}
Propositions~\ref{Prop:notion-smooth} and~\ref{Prop:notion-nonsmooth} show that one needs to pay an additional $O(\kappa\log(1/\epsilon))$ or $O(\epsilon^{-2})$ gradient evaluations for translating the stationarity of $f$ to that of $\Phi$ in nonconvex-strongly-concave and nonconvex-concave settings. In this sense, the notions of stationarity as per Definitions~\ref{def:notion-smooth} and~\ref{def:notion-nonsmooth} are stronger than that as per Definition~\ref{def:stationary}. 
\begin{remark}
\citet{Ostrovskii-2021-Efficient} identified an error in the proof of~\citet[Proposition~4.12]{Lin-2020-Gradient} and proposed a new notion of stationarity based on $\grad f$ which is stronger than the existing stationarity of $f$. After communicating with the authors of~\citet{Ostrovskii-2021-Efficient}, we agree that~\citet[Proposition~4.12]{Lin-2020-Gradient} (i.e., the precursor of Proposition~\ref{Prop:notion-nonsmooth}) is still valid and the error can be fixed by using the notion in Definition~\ref{def:stationary}. We also acknowledge that their notion is strictly stronger than ours in Definition~\ref{def:stationary} when $\YCal \neq \br^n$ while being tractable from a computational viewpoint. 
\end{remark}

\section{Smooth Minimax Optimization}\label{sec:smooth}
We prove the bounds on the complexity of solving smooth and nonconvex-concave minimax optimization problems. Our focus is TTGDA and TTSGDA (see Algorithms~\ref{Algorithm:TTGDA} and~\ref{Algorithm:TTSGDA}) with a proper choice of $(\eta_\x^t, \eta_\y^t)_{t \geq 0}$; recall that these methods were shown to converge locally and are useful in practice~\citep{Heusel-2017-Gans}. Intuitively, the reason for this phenomenon is that the choice of $\eta_\x^t \ll \eta_\y^t$ matches the asymmetric structure of nonconvex-concave minimax optimization problems. 

\subsection{Main results}
We make the following assumptions throughout this subsection. 
\begin{assumption}\label{Assumption:nsc-smooth} 
The function $f(\x, \y)$ is $\ell$-smooth and $\mu$-strongly concave in $\y$, and the constraint set $\YCal$ is convex and bounded with a diameter $D > 0$.
\end{assumption}
\begin{assumption}\label{Assumption:nc-smooth}
The function $f(\x, \y)$ is $\ell$-smooth, $L$-Lipschitz in $\x$ and concave in $\y$, and the constraint set $\YCal$ is convex and bounded with a diameter $D > 0$.
\end{assumption}  
In addition to smoothness condition, the function $f(\x, \y)$ is assumed to be concave in $\y$ but nonconvex in $\x$.  Imposing such one-side concavity is necessary since~\citet{Daskalakis-2021-Complexity} have shown that it is NP-hard to determine if an approximate min-max point exists in a general smooth and nonconvex-nonconcave setting. Imposing the boundedness condition for $\YCal$ is an assumption of convenience that reflects the fact that the ambiguity sets are generally bounded in real applications~\citep{Sinha-2018-Certifiable}. Relaxing this condition requires more complicated arguments and we leave it to future work.

For the ease of presentation, we define 
\begin{equation*}
\Phi(\x) = \max_{\y \in \YCal} f(\x, \y), \quad \YCal^\star(\x) = \argmax_{\y \in \YCal} f(\x, \y), \quad \textnormal{for each } \x \in \br^m. 
\end{equation*}
We present a key lemma on the structure of the function $\Phi(\cdot)$ and the set function $\YCal^\star(\cdot)$. 
\begin{lemma}\label{Lemma:smooth} 
The following statements hold true, 
\begin{enumerate}
\item Under Assumption~\ref{Assumption:nsc-smooth}, we have that $\YCal^\star(\x)$ is a singleton for each $\x$ such that $\YCal^\star(\x)=\{\y^\star(\x)\}$ where $\y^\star(\cdot)$ is $\kappa$-Lipschitz. Also, $\Phi(\cdot)$ is $(2\kappa\ell)$-smooth and $\grad\Phi(\cdot) = \gradx f(\cdot, \y^\star(\cdot))$. 
\item Under Assumption~\ref{Assumption:nc-smooth}, we have that $\YCal^\star(\x)$ contains many entries for each $\x$ such that $\gradx f(\x, \y) \in \partial\Phi(\x)$ where $\y \in \YCal^\star(\x)$. Also, $\Phi(\cdot)$ is $\ell$-weakly convex and $L$-Lipschitz. 
\end{enumerate}
\end{lemma}
Lemma~\ref{Lemma:smooth} shows that the function $\Phi(\cdot)$ is smooth under Assumption~\ref{Assumption:nsc-smooth}, implying that the stationarity notion in Definition~\ref{def:notion-smooth} is our target. Denoting $\Delta_\Phi = \Phi(\x_0) - \min_\x \Phi(\x)$, we summarize our results for Algorithm~\ref{Algorithm:TTGDA} and~\ref{Algorithm:TTSGDA} in the following theorem. 
\begin{theorem}\label{Thm:nsc-smooth}
Under Assumption~\ref{Assumption:nsc-smooth}, the following statements hold true, 
\begin{enumerate}
\item If we set 
\begin{equation}\label{rule:stepsize-smooth}
\eta_\x^t = \tfrac{1}{16(\kappa + 1)^2\ell}, \quad \eta_\y^t = \tfrac{1}{\ell}, 
\end{equation}
in the TTGDA scheme (see Algorithm~\ref{Algorithm:TTGDA}), the required number of gradient evaluations to return an $\epsilon$-stationary point is
\begin{equation*}
O\left(\tfrac{\kappa^2\ell\Delta_\Phi + \kappa\ell^2 D^2}{\epsilon^2}\right). 
\end{equation*}
\item If we set $\eta_\x^t$ and $\eta_\y^t$ as per Eq.~\eqref{rule:stepsize-smooth}, $M = \max\{1, 48\kappa\sigma^2\epsilon^{-2}\}$ and $G$ as per Eq.~\eqref{Assumption:SGDA} in the TTSGDA scheme (see Algorithm~\ref{Algorithm:TTSGDA}), the required number of stochastic gradient evaluations to return an $\epsilon$-stationary point is
\begin{equation*}
O\left(\tfrac{\kappa^2\ell\Delta_\Phi + \kappa\ell^2 D^2}{\epsilon^2}\max\left\{1, \tfrac{\kappa\sigma^2}{\epsilon^2}\right\}\right). 
\end{equation*}
\end{enumerate}
\end{theorem}
\begin{remark}
Theorem~\ref{Thm:nsc-smooth} provides the complexity results for TTGDA and TTSGDA when applied to solve the smooth and nonconvex-strongly-concave minimax optimization problems. The ratio $\frac{\eta_\y^t}{\eta_\x^t} = \Theta(\kappa^2)$ arises due to asymmetric structure and the order of $O(\kappa^2)$ reflects an algorithmic efficiency trade-off. Furthermore, our algorithms are only guaranteed to visit an $\epsilon$-stationary point within a certain number of iterations. Such the convergence guarantees are standard in nonconvex optimization and practitioners usually return an iterate at which the learning curve stops changing significantly. Finally, the batch size $M = \Theta(\kappa\sigma^2\epsilon^{-2})$ is necessary for proving Theorem~\ref{Thm:nsc-smooth}. While our proof technique can be extended to the case of $M=1$, the complexity bound becomes $O(\kappa^3\epsilon^{-5})$. This gap has been closed using other new single-loop algorithms~\citep{Yang-2022-Faster} but remains open for TTSGDA. 
\end{remark}
Lemma~\ref{Lemma:smooth} shows that the function $\Phi(\cdot)$ is weakly convex under Assumption~\ref{Assumption:nc-smooth}, implying that the stationarity notion in Definition~\ref{def:notion-nonsmooth} is our target. We summarize our results for Algorithm~\ref{Algorithm:TTGDA} and~\ref{Algorithm:TTSGDA} in the following theorem. 
\begin{theorem}\label{Thm:nc-smooth}
Under Assumption~\ref{Assumption:nc-smooth}, the following statements hold true, 
\begin{enumerate}
\item If we set 
\begin{equation*}
\eta_\x^t = \min\left\{\tfrac{\epsilon^2}{80\ell L^2}, \tfrac{\epsilon^4}{4096\ell^3 L^2 D^2}\right\}, \quad \eta_\y^t = \tfrac{1}{\ell}, 
\end{equation*}
in the TTGDA scheme (see Algorithm~\ref{Algorithm:TTGDA}), the required number of gradient evaluations to return an $\epsilon$-stationary point is
\begin{equation*}
O\left(\left(\tfrac{\ell L^2\Delta_\Phi}{\epsilon^4} + \tfrac{\ell\Delta_0}{\epsilon^2}\right)\max\left\{1, \tfrac{\ell^2D^2}{\epsilon^2}\right\}\right). 
\end{equation*}
\item If we set $M = 1$, $G$ as per Eq.~\eqref{Assumption:SGDA} and  
\begin{equation*}
\eta_\x^t = \min\left\{\tfrac{\epsilon^2}{80\ell(L^2 + \sigma^2)}, \tfrac{\epsilon^4}{8192\ell^3 (L^2+\sigma^2)D^2}, \tfrac{\epsilon^6}{131072\ell^3 (L^2+\sigma^2) D^2\sigma^2}\right\}, \quad \eta_\y^t = \min\left\{\tfrac{1}{2\ell}, \tfrac{\epsilon^2}{32\ell\sigma^2}\right\}, 
\end{equation*}
in the TTSGDA scheme (see Algorithm~\ref{Algorithm:TTSGDA}), the required number of stochastic gradient evaluations to return an $\epsilon$-stationary point is
\begin{equation*}
O\left(\left(\tfrac{\ell\left(L^2 + \sigma^2\right)\Delta_\Phi}{\epsilon^4} + \tfrac{\ell\Delta_0}{\epsilon^2}\right)\max\left\{1, \tfrac{\ell^2D^2}{\epsilon^2}, \tfrac{\ell^2 D^2\sigma^2}{\epsilon^4}\right\}\right). 
\end{equation*}
\end{enumerate}
\end{theorem}
\begin{remark}
Theorem~\ref{Thm:nc-smooth} is the analog of Theorem~\ref{Thm:nsc-smooth} in the nonconvex-concave setting but with two different results. First, the order of the ratio becomes $\eta_\y/\eta_\x = \Theta(\epsilon^4)$ which highlights the highly asymmetric structure here. Second, the theoretical results are valid for the case of $M=1$ and in that case they are different from that derived in Theorem~\ref{Thm:nsc-smooth}. 
\end{remark}

\subsection{Discussions} 
We remark that we have not shown that TTGDA and TTSGDA are optimal in any sense, nor that acceleration might be achieved by using momentum for updating the variable $\y$. 

In smooth and nonconvex-strongly-concave setting, the lower bound is $\Omega(\sqrt{\kappa}\epsilon^{-2})$ in terms of gradient evaluations~\citep{Zhang-2021-Complexity, Li-2021-Complexity} and was already achieved by several nested-loop algorithms~\citep{Kong-2021-Accelerated, Lin-2020-Near, Ostrovskii-2021-Efficient, Zhao-2023-Primal} up to logarithmic factors. In the context of single-loop algorithms,~\citet{Yang-2022-Faster} has extended the smoothed GDA~\citep{Zhang-2020-Single} to more general settings where the function $f(\x, \y)$ is nonconvex in $\x$ and satisfies the Polyak-\L{}ojasiewicz condition in $\y$ and proved a bound of $O(\kappa\epsilon^{-2})$. This result demonstrates that the proper adjustments of TTGDA leads to tighter bounds in smooth and nonconvex-strongly-concave setting. In addition,~\citet{Li-2022-Convergence} provided a local analysis for TTGDA and showed that the stepsize ratio of $\Theta(\kappa)$ is necessary and sufficient for local convergence to a stationary point. 

In smooth and nonconvex-concave setting, the lower bound remains open and the best-known upper bound is presented by recent works~\citep{Thekumparampil-2019-Efficient, Kong-2021-Accelerated, Lin-2020-Near, Ostrovskii-2021-Efficient, Zhao-2023-Primal}. All of the analyses require $\tilde{O}(\epsilon^{-3})$ gradient evaluations for returning an $\epsilon$-stationary point but the proposed algorithmic schemes are complex. The general question of constructing lower bounds and optimal algorithms in this setting is an important topic for future research.

Our complexity results do not contradict the classical results on the divergence of GDA with \textit{fixed stepsize} in the convex-concave setting. Indeed, there are a few key distinctions: (1) our results guarantee that the proposed algorithm visits an $\epsilon$-stationary point at some iterate, not the last iterate; (2) our results guarantee the stationarity of some iterate $\x_t$ rather than $(\x_t, \y_t)$. As such, our proof permits the possibility of significant changes in $\y_t$ even when $\x_t$ is very close to stationarity. This, together with our choice $\eta_\x^t \ll \eta_\y^t$, removes the contradiction. Based on the above fact, we highlight that our algorithms can be used to obtain an approximate saddle point in the smooth and convex-concave setting (i.e., optimality for both $\x$ and $\y$). Instead of averaging, we can run two passes of TTGDA for minimax and maximin optimization problems separately. We use $\eta_\x^t \ll \eta_\y^t$ in the first pass while we use $\eta_\x^t \gg \eta_\y^t$ in the second pass. Either pass returns an approximate stationary point for each variable $\x$ and/or $\y$, which is defined using the corresponding Moreau envelope and which jointly forms an approximate saddle point.

\subsection{Proof sketch}
We sketch the proofs for TTGDA by highlighting the ideas for proving Theorems~\ref{Thm:nsc-smooth} and~\ref{Thm:nc-smooth}. 

\paragraph{Nonconvex-strongly-concave setting.} We have that the function $\y^\star(\cdot)$ is $\kappa$-Lipschitz (cf. Lemma~\ref{Lemma:smooth}), which guarantees that $\{\y^\star(\x_t)\}_{t \geq 1}$ moves slowly if $\{\x_t\}_{t \geq 1}$ moves slowly. This implies that one-step gradient ascent over $\y$ at each iteration is sufficient for minimizing this slowly changing sequence of strongly concave functions $\{f(\x_t, \cdot)\}_{t > 1}$. Indeed, we define $\delta_t = \|\y^\star(\x_t) - \y_t\|^2$ and prove that $\sum_{t=0}^T \delta_t$ can be controlled using the relations as follows, 
\begin{equation*}
\delta_t \leq (1-\tfrac{1}{2\kappa}+4\kappa^3\ell^2(\eta_\x^t)^2)\delta_{t-1} + 4\kappa^3(\eta_\x^t)^2\|\grad\Phi(\x_{t-1})\|^2,
\end{equation*}
where $\eta_\x^t$ can be chosen such that $\sum_{t=0}^T \delta_t$ is controlled by $\sum_{t=0}^T \|\grad \Phi(\x_t)\|^2$. 

By viewing TTGDA as inexact gradient descent for minimizing the function $\Phi(\cdot)$, we derive the key descent inequality as follows, 
\begin{equation*}
\Phi(\x_{T+1}) - \Phi(\x_0) \leq -\tfrac{7\eta_\x^t}{16}\left(\sum_{t=0}^T \|\grad \Phi(\x_t)\|^2\right) + \tfrac{9(\eta_\x^t)\ell^2}{16}\left(\sum_{t=0}^T \delta_t\right).
\end{equation*}
Putting these pieces together yields the desired results in Theorem~\ref{Thm:nsc-smooth}.  

\paragraph{Nonconvex-concave setting.} We note that $\YCal^\star(\x)$ will not be a singleton for each $\x$ and $\YCal^\star(\x_1)$ can be dramatically different from $\YCal^\star(\x_2)$ even when $\x_1, \x_2$ are close to each other. This implies that $\min_{\y \in \YCal^\star(\x_t)} \|\y_t - \y\|$ is no longer a viable error to control.

However, we still have that one-step gradient ascent over $\y$ at each iteration is sufficient for minimizing this slowly changing sequence of concave functions $\{f(\x_t, \cdot)\}_{t > 1}$. Indeed, we have that the function $\Phi(\cdot)$ is $L$-Lipschitz (cf.\ Lemma~\ref{Lemma:nonsmooth}), which guarantees that $\{\Phi(\x_t)\}_{t \geq 1}$ moves slowly if $\{\x_t\}_{t \geq 1}$ moves slowly. We can define $\Delta_t = \Phi(\x_t) - f(\x_t, \y_t)$ and prove that $\frac{1}{T+1}(\sum_{t=0}^T \Delta_t)$ can be controlled using the relations as follows, 
\begin{equation*}
\tfrac{1}{T+1}\left(\sum_{t=0}^T \Delta_t\right) \leq \eta_\x^t (B+1) L^2 + \tfrac{\ell D^2}{2B} + \tfrac{\Delta_0}{T+1}, 
\end{equation*}
where $B = \left\lceil\frac{D}{2L}\sqrt{\frac{\ell}{\eta_\x^t}}\right\rceil$ and $\eta_\x^t$ can be chosen such that $\tfrac{1}{T+1}(\sum_{t=0}^T \Delta_t)$ is well controlled. 

By viewing TTGDA as an inexact subgradient descent for minimizing the function $\Phi(\cdot)$, we follow~\citet{Davis-2019-Stochastic} and derive the key descent inequality as follows, 
\begin{equation*}
\Phi_{1/2\ell}(\x_{T+1}) - \Phi_{1/2\ell}(\x_0) \leq 2\eta_\x\ell\left(\sum_{t=0}^T \Delta_t\right) + \eta_\x^2 \ell L^2 (T+1) - \tfrac{\eta_\x}{4} \left(\sum_{t=0}^T \|\grad \Phi_{1/2\ell}(\x_t)\|^2\right). 
\end{equation*}
Putting these pieces together yields the desired results in Theorem~\ref{Thm:nc-smooth}.  

\section{Nonsmooth Minimax Optimization}
We prove the bounds on the complexity of solving nonsmooth and nonconvex-concave minimax optimization problems. Indeed, we assume that $f(\x, \y)$ is $L$-Lipschitz and $\rho$-weakly convex in $\x$. Our focus is TTGDA and TTSGDA with a proper choice of $(\eta_\x^t, \eta_\y^t)_{t \geq 0}$. 

\subsection{Main results}
We make the following assumptions throughout this subsection. 
\begin{assumption}\label{Assumption:nsc-nonsmooth} 
The function $f(\x, \y)$ is $L$-Lipschitz, $\rho$-weakly convex in $\x$ and $\mu$-strongly concave in $\y$, and the constraint set $\YCal$ is convex and bounded with a diameter $D > 0$.  
\end{assumption}
\begin{assumption}\label{Assumption:nc-nonsmooth}
The function $f(\x, \y)$ is $L$-Lipschitz, $\rho$-weakly convex in $\x$ and concave in $\y$, and the constraint set $\YCal$ is convex and bounded with a diameter $D > 0$.  
\end{assumption} 
The setting defined by Assumptions~\ref{Assumption:nsc-nonsmooth} and~\ref{Assumption:nc-nonsmooth} generalizes the nonsmooth optimization and are intuitively more challenging than smooth minimax optimization studied in Section~\ref{sec:smooth}. The following lemma is an analog of Lemma~\ref{Lemma:smooth}. 
\begin{lemma}\label{Lemma:nonsmooth} 
The following statements hold true, 
\begin{enumerate}
\item Under Assumption~\ref{Assumption:nsc-nonsmooth}, we have that $\YCal^\star(\x)$ is a singleton for each $\x$ such that $\YCal^\star(\x)=\{\y^\star(\x)\}$ and $\subgx f(\x, \y^\star(\x)) \subseteq \partial\Phi(\x)$. Also, $\Phi(\cdot)$ is $L$-Lipschitz and $\rho$-weakly convex. 
\item Under Assumption~\ref{Assumption:nc-nonsmooth}, we have that $\YCal^\star(\x)$ contains many entries for each $\x$ such that $\subgx f(\x, \y) \subseteq \partial\Phi(\x)$ where $\y \in \YCal^\star(\x)$. Also, $\Phi(\cdot)$ is $L$-Lipschitz and $\rho$-weakly convex. 
\end{enumerate}
\end{lemma}
Lemma~\ref{Lemma:nonsmooth} shows that the function $\Phi(\cdot)$ is weakly convex in both of settings, implying that the stationarity notion in Definition~\ref{def:notion-nonsmooth} is our target. We summarize our results for Algorithm~\ref{Algorithm:TTGDA} and~\ref{Algorithm:TTSGDA} in the following theorem. 
\begin{theorem}\label{Thm:nsc-nonsmooth}
Under Assumption~\ref{Assumption:nsc-nonsmooth}, the following statements hold true, 
\begin{enumerate}
\item If we set 
\begin{equation*}
\eta_\x^t = \min\left\{\tfrac{\epsilon^2}{48\rho L^2}, \tfrac{\mu\epsilon^4}{4096\rho^2 L^4}, \tfrac{\mu\epsilon^4}{4096\rho^2 L^4 \log^2(1 + 4096\rho^2 L^4\mu^{-2}\epsilon^{-4})}\right\}, 
\end{equation*}
and
\begin{equation}\label{rule:stepsize-nonsmooth}
\eta_\y^t = \left\{
\begin{array}{cl}
\tfrac{1}{\mu t}, & \textnormal{if } 1 \leq t \leq B, \\
\tfrac{1}{\mu(t-B)}, & \textnormal{else if } B+1 \leq t \leq 2B, \\
\vdots & \vdots \\
\tfrac{1}{\mu(t-jB)}, & \textnormal{else if } jB+1 \leq t \leq (j+1)B, \\
\vdots & \vdots \\
\end{array}
\right.  \quad \textnormal{for } B = \left\lfloor\sqrt{\frac{1}{\mu \eta_\x}}\right\rfloor + 1,  
\end{equation}
in the TTGDA scheme (see Algorithm~\ref{Algorithm:TTGDA}), the required number of gradient evaluations to return an $\epsilon$-stationary point is
\begin{equation*}
O\left(\tfrac{\rho L^2\Delta_\Phi}{\epsilon^4}\max\left\{1, \tfrac{\rho L^2}{\mu\epsilon^2}, \tfrac{\rho L^2}{\mu\epsilon^2}\log^2\left(1 + \tfrac{\rho^2 L^4}{\mu^2 \epsilon^4}\right)\right\}\right). 
\end{equation*}
\item If we set $M = 1$, $G$ as per Eq.~\eqref{Assumption:SGDA}, $\eta_\y^t$ as per Eq.~\eqref{rule:stepsize-nonsmooth} and 
\begin{equation*}
\eta_\x^t = \min\left\{\tfrac{\epsilon^2}{48\rho(L^2 + \sigma^2)}, \tfrac{\mu\epsilon^4}{4096\rho^2 (L^2 + \sigma^2)^2}, \tfrac{\mu\epsilon^4}{4096\rho^2 (L^2 + \sigma^2)^2\log^2(1 + 4096\rho^2(L^2 + \sigma^2)^2 \mu^{-2}\epsilon^{-4})}\right\}, 
\end{equation*}
in the TTSGDA scheme (see Algorithm~\ref{Algorithm:TTSGDA}), the required number of stochastic gradient evaluations to return an $\epsilon$-stationary point is
\begin{equation*}
O\left(\tfrac{\rho(L^2 + \sigma^2)\Delta_\Phi}{\epsilon^4}\max\left\{1, \tfrac{\rho^2(L^2 + \sigma^2)}{\mu\epsilon^2}, \tfrac{\rho^2(L^2 + \sigma^2)}{\mu\epsilon^2}\log^2\left(1 + \tfrac{\rho^2(L^2 + \sigma^2)^2}{\mu^2\epsilon^4}\right)\right\}\right). 
\end{equation*}
\end{enumerate}
\end{theorem}
\begin{remark}
Theorem~\ref{Thm:nsc-nonsmooth} shows that the adaptive choice of $\eta_\y^t$ better exploits the nonconvex-strongly-concave structure in the nonsmooth setting; indeed, $\eta_\y^t = 1/(\mu t)$ for $t \in \{1, 2, \ldots, B\}$ which corresponds to the first epoch. Then, for each subsequent epoch, we set the same sequence of stepsizes. Such a finding is consistent with that for (stochastic) subgradient-type methods for minimizing a strongly convex function~\citep{Nemirovski-1983-Problem, Rakhlin-2012-Making, Shamir-2013-Stochastic}. We remark that the complexity bounds for TTGDA and TTSGDA are the same in terms of $1/\epsilon$, which is also consistent with the classical results for nonsmooth optimization~\citep{Nemirovski-1983-Problem, Davis-2019-Stochastic}.  
\end{remark}
\begin{theorem}\label{Thm:nc-nonsmooth}
Under Assumption~\ref{Assumption:nc-nonsmooth}, the following statements hold true, 
\begin{enumerate}
\item If we set 
\begin{equation*}
\eta_\x^t = \min\left\{\tfrac{\epsilon^2}{48\rho L^2}, \tfrac{\epsilon^6}{65536\rho^3 L^4 D^2}\right\}, \quad \eta_\y^t = \tfrac{\epsilon^2}{16\rho L^2}, 
\end{equation*}
in the TTGDA scheme (see Algorithm~\ref{Algorithm:TTGDA}), the required number of gradient evaluations to return an $\epsilon$-stationary point is
\begin{equation*}
O\left(\tfrac{\rho L^2\Delta_\Phi}{\epsilon^4}\max\left\{1, \tfrac{\rho^2 L^2D^2}{\epsilon^4}\right\}\right). 
\end{equation*}
\item If we set $M = 1$, $G$ as per Eq.~\eqref{Assumption:SGDA} and 
\begin{equation*}
\eta_\x^t = \min\left\{\tfrac{\epsilon^2}{48\rho(L^2 + \sigma^2)}, \tfrac{\epsilon^6}{131072\rho^3 (L^2+\sigma^2)^2 D^2}\right\}, \quad \eta_\y^t = \tfrac{\epsilon^2}{32\rho(L^2 + \sigma^2)}, 
\end{equation*}
in the TTSGDA scheme (see Algorithm~\ref{Algorithm:TTSGDA}), the required number of stochastic gradient evaluations to return an $\epsilon$-stationary point is
\begin{equation*}
O\left(\tfrac{\rho\left(L^2 + \sigma^2\right)\Delta_\Phi}{\epsilon^4}\max\left\{1, \tfrac{\rho^2(L^2 + \sigma^2)D^2}{\epsilon^4}\right\}\right). 
\end{equation*}
\end{enumerate}
\end{theorem}
\begin{remark}
Theorems~\ref{Thm:nsc-nonsmooth} and~\ref{Thm:nc-nonsmooth} provide the gradient complexity results for TTGDA and TTSGDA when applied to solve the nonsmooth and nonconvex-(strongly)-concave minimax optimization problems. The orders of $O(\epsilon^{-4})$ and $\Theta(\epsilon^{-4})$ reflect the effect of nonsmoothness when compared to that presented in Theorem~\ref{Thm:nsc-smooth} and~\ref{Thm:nc-smooth}. Furthermore, we guarantee that an $\epsilon$-stationary point is visited within a certain number of iterations rather than guaranteeing last-iterate convergence. Finally, Theorem~\ref{Thm:nsc-nonsmooth} is valid for the case of $M = 1$ and such a result is different from the theoretical results in Theorem~\ref{Thm:nsc-smooth}. 
\end{remark}

\section{Applications and Empirical Results}\label{sec:exp}
We discuss our evaluation of TTGDA and TTSGDA in two domains---robust regression and Wasserstein generative adversarial networks (WGANs). All of algorithms were implemented on a MacBook Pro with an Intel Core i9 2.4GHz and 16GB memory.

\subsection{Robust logistic regression}\label{subsec:regression}
We will consider the problem of robust logistic regression with nonconvex penalty functions. Given a set of samples $\{(\sa_i, b_i)\}_{i=1}^N$ where $\sa_i \in \br^d$ is the feature of the $i^\textnormal{th}$ sample and $b_i \in \{1, -1\}$ is the label, the problem can be written in the form of Eq.~\eqref{prob:main} with 
\begin{equation*}
f(\x, \y) = \tfrac{1}{N}\left(\sum_{i=1}^N y_i \log(1 + \exp(-b_i\sa_i^\top\x))\right) - \tfrac{1}{2}\lambda_1\|N\y - \textbf{1}_n\|^2 + \lambda_2\left(\sum_{i=1}^d \tfrac{\alpha x_i^2}{1 + \alpha x_i^2}\right), 
\end{equation*}
and $\YCal = \{\y \in \br_+^n: \y^\top \textbf{1}_n = 1\}$. Here, we compare TTGDA and TTSGDA with GDmax~\citep{Jin-2020-Local} on 6 \textsf{LIBSVM} datasets\footnote{https://www.csie.ntu.edu.tw/$\sim$cjlin/libsvm/} and follow the setup of~\citet{Kohler-2017-Cubic} to set $\lambda_1 = \frac{1}{n^2}$, $\lambda_2 = 10^{-2}$ and $\alpha = 10$ for our experiments. 
\begin{figure*}[!t]
\centering
\includegraphics[width=0.32\textwidth]{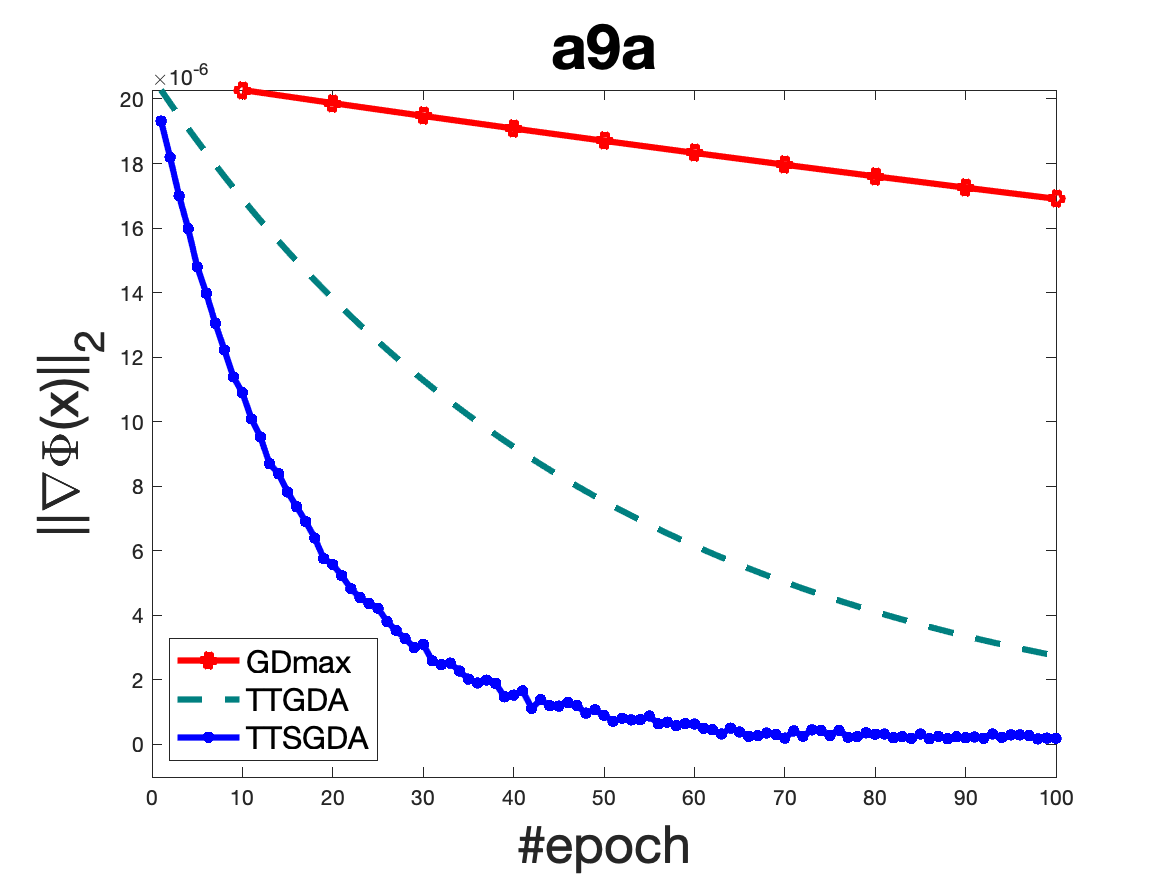}
\includegraphics[width=0.32\textwidth]{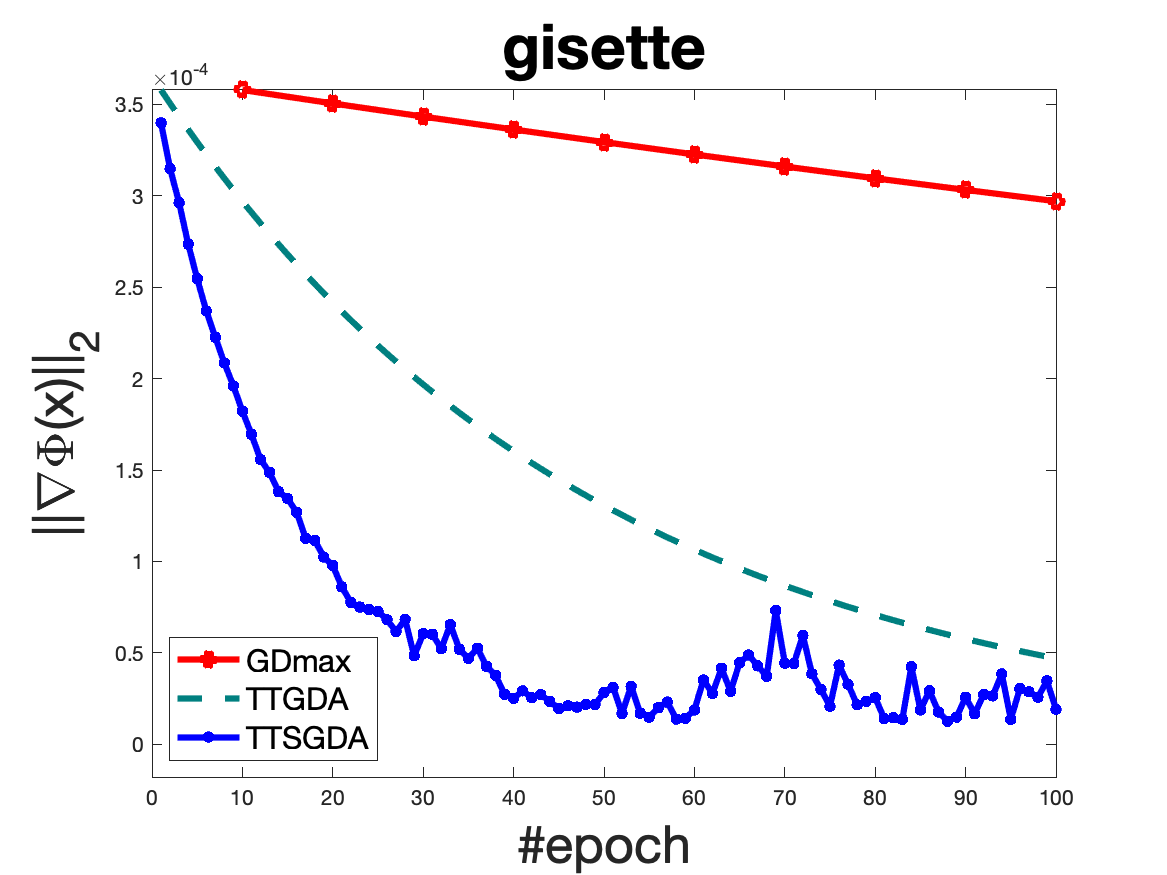}
\includegraphics[width=0.32\textwidth]{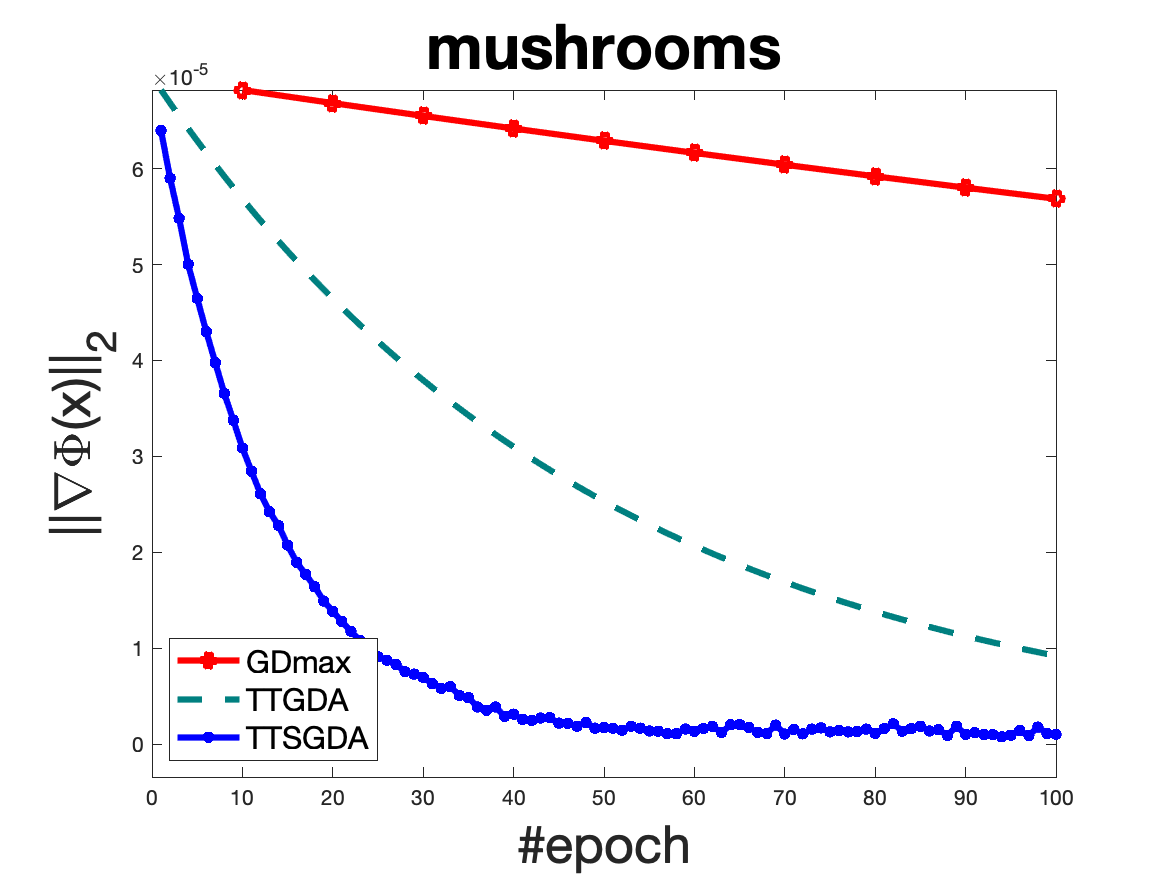} \\ \vspace{5pt}

\includegraphics[width=0.32\textwidth]{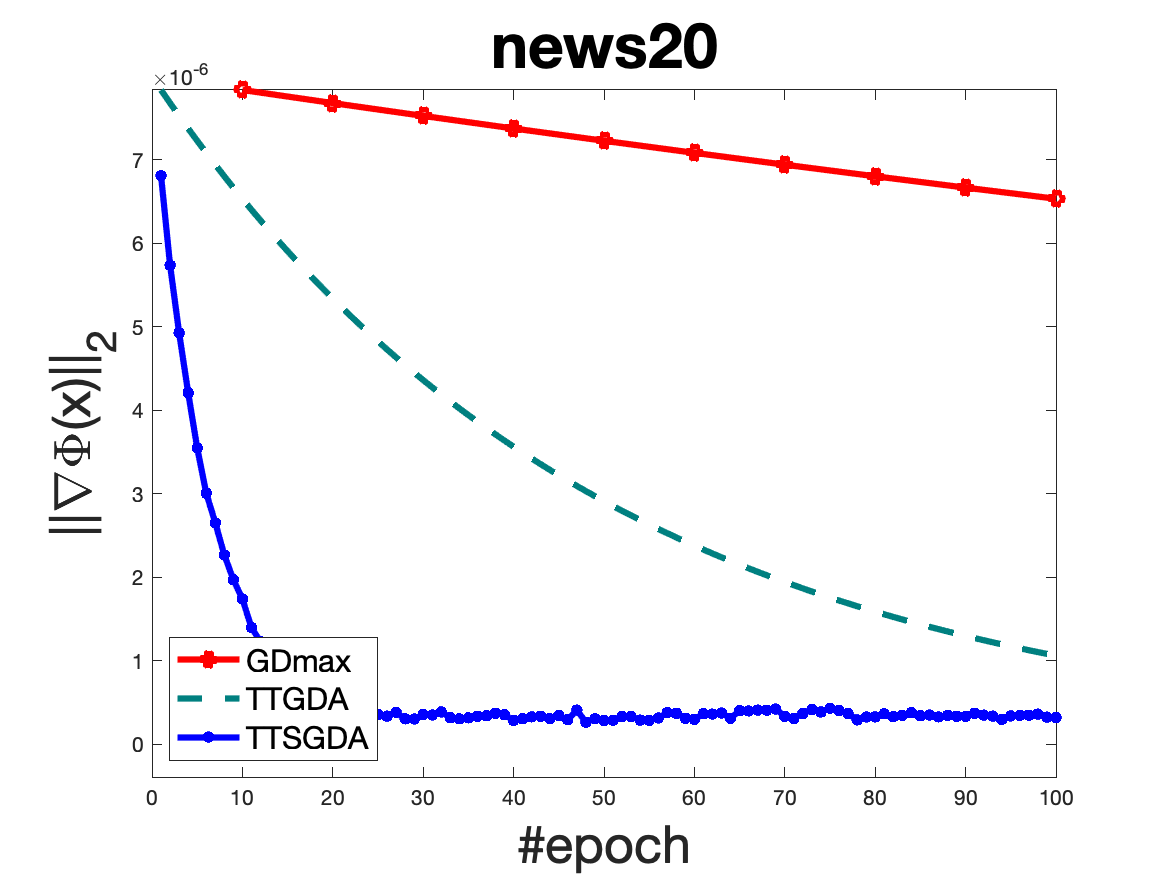}
\includegraphics[width=0.32\textwidth]{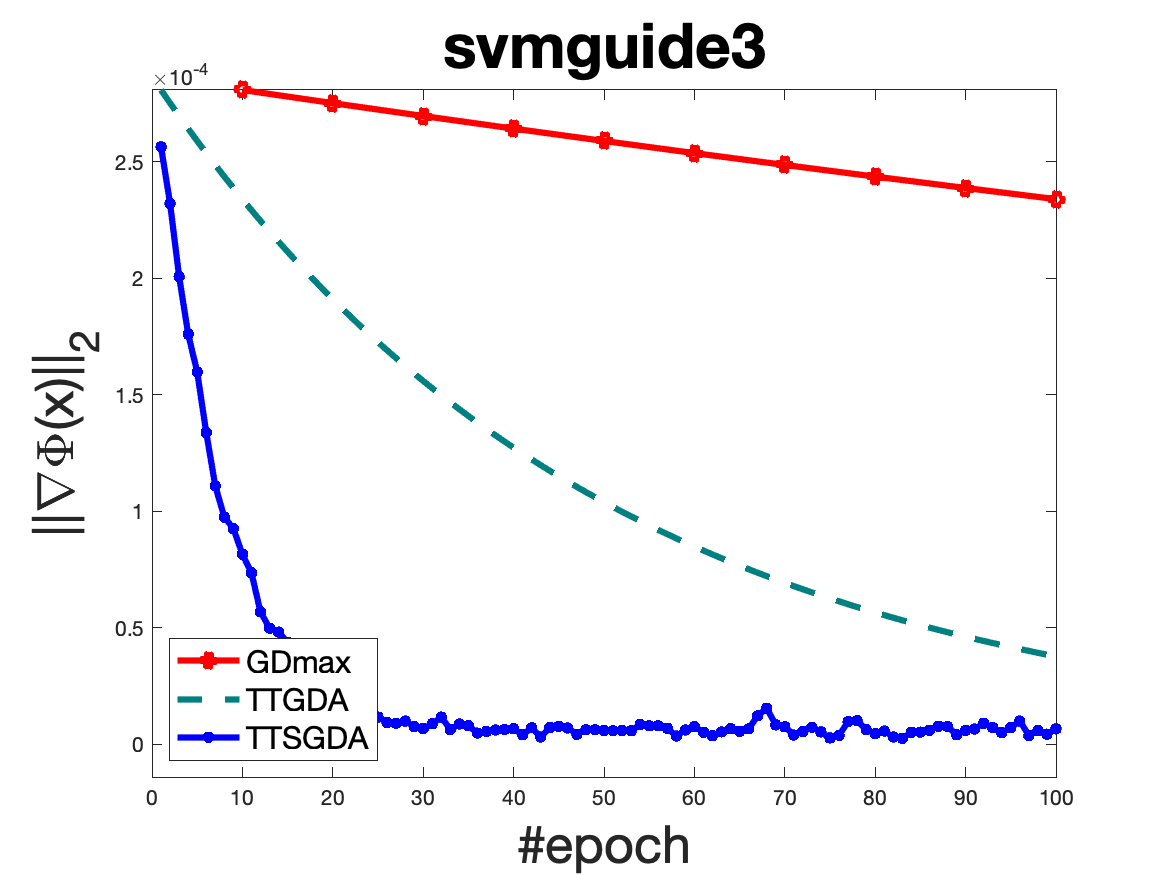}
\includegraphics[width=0.32\textwidth]{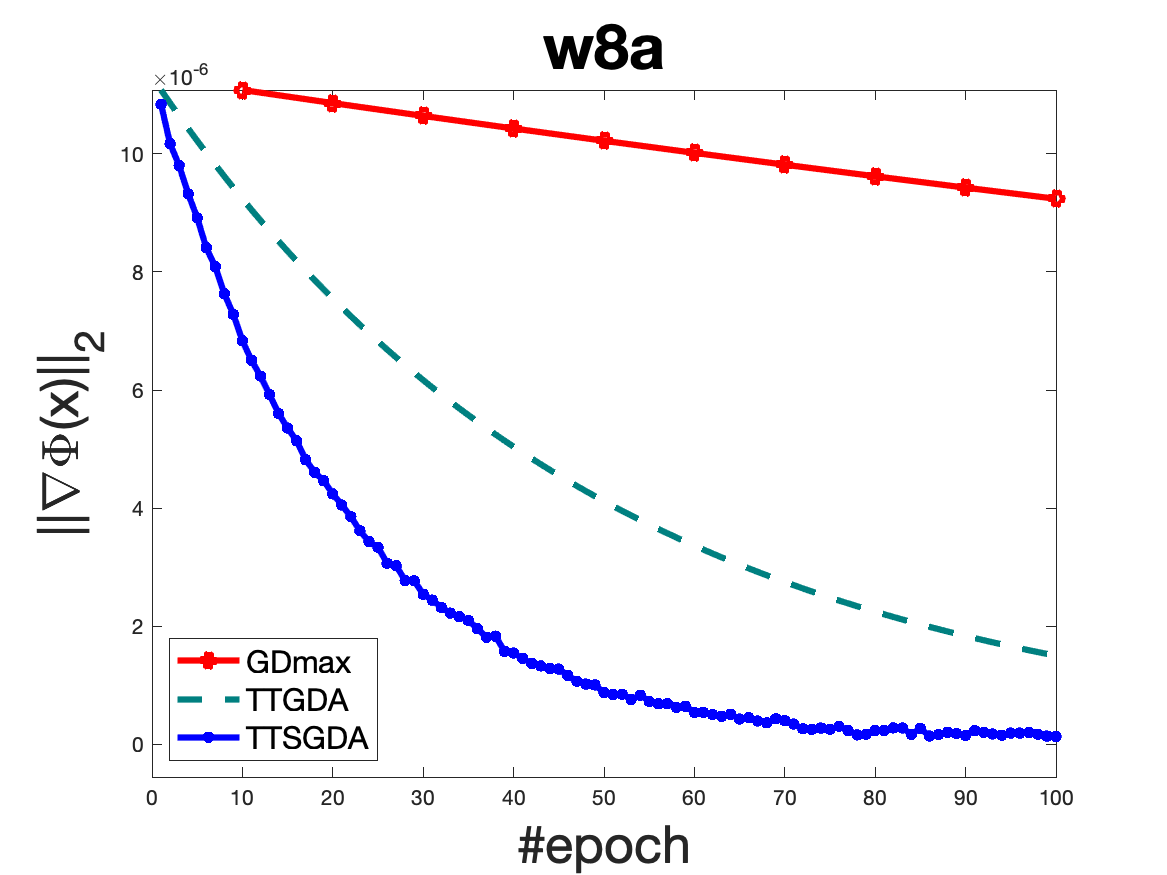} \\ \vspace{5pt}
\caption{Performance of all the algorithms with 6 \textsf{LIBSVM} datasets. The numerical results are presented in terms of epoch count where the evaluation metric is the gradient norm of the function $\Phi(\cdot) = \max_{\y \in \YCal} f(\cdot, \y)$.}\label{fig:regression}\vspace*{-1em}
\end{figure*}

The algorithmic parameters are tuned as follows: we consider different pairs of $(\eta_\x, \eta_\y)$ where $\eta_\y \in \{10^{-1}, 1\}$ and the ratio $\frac{\eta_\y}{\eta_\x} \in \{10, 10^2, 10^3\}$, and different sizes $M \in \{10, 100, 200\}$ for TTSGDA. Figure~\ref{fig:regression} presents the performance of each algorithm with its corresponding fine-tuning parameters, demonstrating that TTGDA and TTSGDA consistently converge faster than GDmax on 6 \textsf{LIBSVM} datasets that we consider in this paper. 

\subsection{WGANs with linear generators}
We will consider the same setting as~\citet{Loizou-2020-Stochastic} which uses the WGAN~\citep{Arjovsky-2017-Wasserstein} to approximate a one-dimensional Gaussian distribution. 

The problem setup is as follows: We have the real data $a^\textnormal{real}$ generated from a normal distribution with $\hat{\mu} = 0$ and $\hat{\sigma} = 0.1$, and the latent variable $z$ drawn from a standard normal distribution. We define the generator as $G_\x(z) = x_1 + x_2 z$ and the discriminator as $D_\y(a) = y_1 a + y_2 a^2$, where $a$ is either real data or fake data generated from the generator. The problem can be written in the form of Eq.~\eqref{prob:main} with 
\begin{equation*}
f(\x, \y) = \EE_{(a^\textnormal{real}, z)}\left[D_\y(a^\textnormal{real}) - D_\y(G_\x(z)) - \lambda\|\y\|^2\right],
\end{equation*}
where $\lambda = 10^{-3}$ is chosen to make the function $f(\x, \y)$ concave in $\y$. 

We fix the batch size $M = 100$ in TTSGDA and tune the parameters $(\eta_\x, \eta_\y)$ as we have done in Section~\ref{subsec:regression}. Figure~\ref{fig:wgan_linear} shows that TTSGDA outperforms other adaptive algorithms, including \textsf{Adam}~\citep{Kingma-2015-Adam} and \textsf{RMSprop}~\citep{Tieleman-2012-Lecture}. Here, the possible reason is that the linear generators result in a relatively simple structure such that the two-timescale update with tuning parameters is sufficient. 
\begin{figure*}[!t]
\centering
\includegraphics[width=0.9\textwidth]{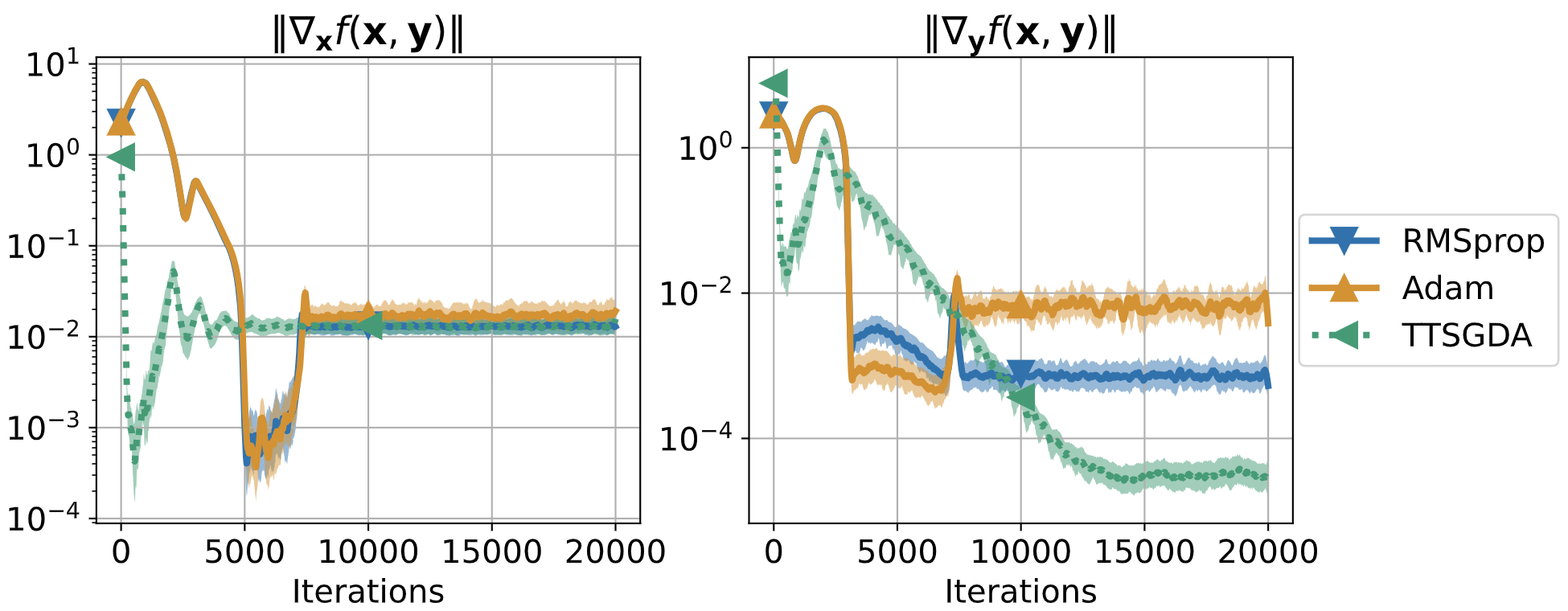}
\caption{Performance of all the algorithms for training WGANs with linear generators. The numerical results are presented in terms of iteration count where the evaluation metric is the gradient norm of the function $f(\cdot, \cdot)$.}\label{fig:wgan_linear}\vspace*{-1em}
\end{figure*}
\begin{figure*}[!t]
\centering
\includegraphics[width=0.95\textwidth]{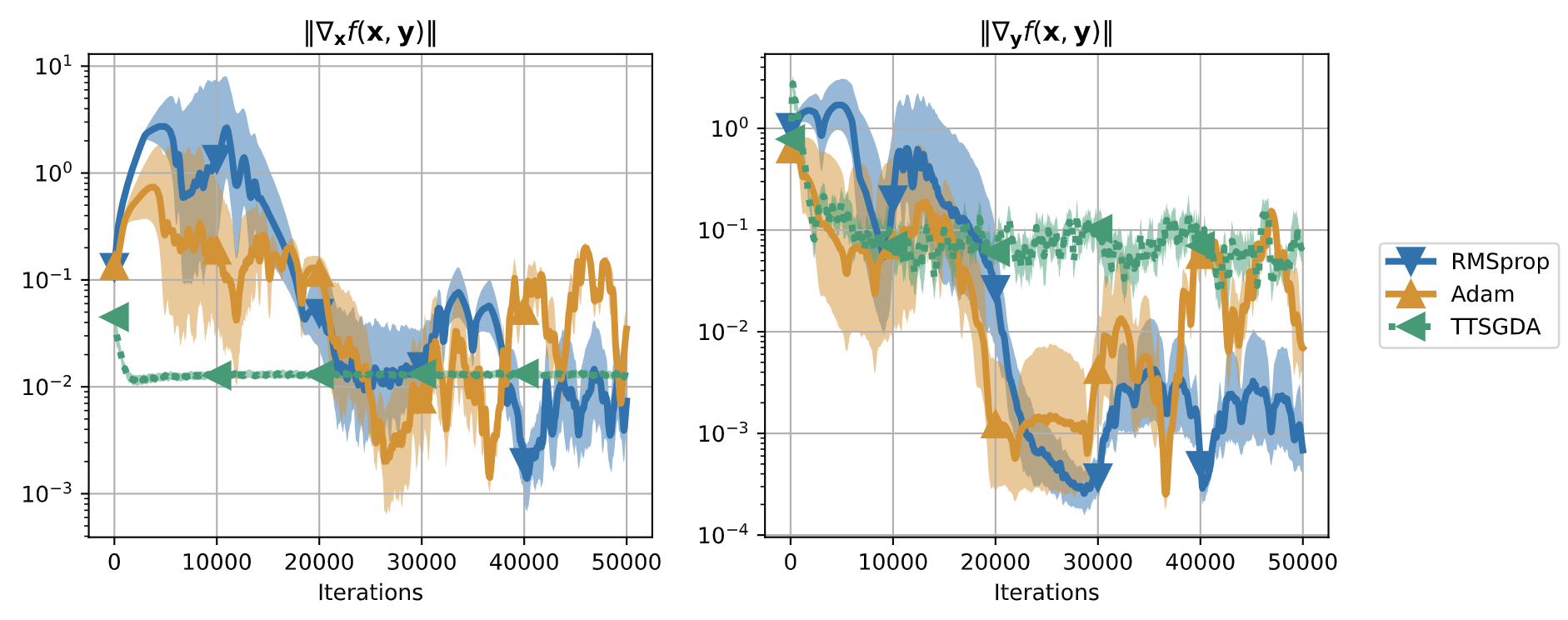}
\caption{Performance of all the algorithms for training WGANs with nonlinear generators. The numerical results are presented in terms of iteration count where the evaluation metric is the gradient norm of the function $f(\cdot, \cdot)$.}\label{fig:wgan_nonlinear}\vspace*{-1em}
\end{figure*}
\subsection{WGANs with nonlinear generators}
Following the setup of~\citet{Lei-2020-SGD} and~\citet{Yang-2022-Faster}, we  consider the WGAN with a ReLU neural network as generator. For ease of comparison, we retain all the problem settings as in the previous paragraph and view $\x$ as the parameters of a small neural network (one hidden layer with five neurons and ReLU activations). 
 
Figure~\ref{fig:wgan_nonlinear} shows that TTSGDA is worse with \textsf{Adam} and \textsf{RMSprop} in this complex setting. This makes sense since the nonlinear generators result in a richer structure and the adaptive algorithms are highly optimized for handling such structure. The vanilla implementation of TTSGDA does not achieve a satisfactory level of performance and can be enhanced by exploiting the advanced tools of regularization and adaptiveness~\citep{Yang-2022-Nest, Yang-2022-Faster, Li-2023-TiAda}. For more details on the implementation heuristics and empirical performance for TTSGDA in training GANs and its variants, we refer to~\citet{Heusel-2017-Gans}. 

\section{Conclusion}\label{sec:conclu}
We proposed two-timescale GDA and SGDA algorithms (TTGDA and TTSGDA for short) for solving a class of nonconvex minimax optimization problems. We established complexity bounds for TTGDA and TTSGDA in terms of the number of (stochastic) gradient evaluations, demonstrating the efficiency of these two algorithms in theory. Our theoretical analysis is of interest to both optimization and machine learning communities as it not only provides a new combination of a two-timescale update rule and GDA/SGDA but supplies a novel proof technique that applies to concave optimization problems with slowly changing objective functions. Future directions of research include the investigation of lower bounds for gradient-based algorithms in the nonconvex-concave settings and the extension of upper bounds for TTGDA and TTSGDA in the structured nonconvex-nonconcave settings. 

\section*{Acknowledgments}
This work was supported in part by the Mathematical Data Science program of the Office of Naval Research under grant number N00014-18-1-2764 and by the Vannevar Bush Faculty Fellowship program under grant number N00014-21-1-2941.

\bibliographystyle{plainnat}
\bibliography{ref}

\appendix
\section{Proofs for Technical Lemmas}
We provide the detailed proofs for Lemmas~\ref{Lemma:ME-smooth},~\ref{Lemma:ME-nonsmooth},~\ref{Lemma:notion-nonsmooth},~\ref{Lemma:smooth},~\ref{Lemma:nonsmooth}, Propositions~\ref{Prop:notion-smooth},~\ref{Prop:notion-nonsmooth}, and auxiliary results on stochastic gradients.  

\subsection{Proof of Lemma~\ref{Lemma:ME-smooth}}
We first prove that the Moreau envelope $\Phi_{1/2\ell}(\cdot)$ is well defined. Fixing $\x$ and using that $f$ is $\ell$-smooth, the function $f(\cdot, \y) + 0.5\ell\|\cdot - \x\|^2$ is convex for any $\y \in \YCal$. By definition, we have $\Phi(\cdot) + 0.5\ell\|\cdot - \x\|^2 = \max_{\y \in \YCal} \{f(\cdot, \y) + 0.5\ell\|\cdot - \x\|^2\}$. Since the supremum of many convex functions is convex, the function $\Phi(\cdot) + 0.5\ell\|\cdot - \x\|^2$ is convex. As such, the function $\Phi(\cdot) + \ell\|\cdot - \x\|^2$ is $0.5\ell$-strongly convex and the Moreau envelope $\Phi_{1/2\ell}(\cdot)$ is well defined. 

As a byproduct of the above argument, $\prox_{\Phi/2\ell}(\cdot) = \argmin_\w \Phi(\w) + \ell\|\w - \cdot\|^2$ is well defined. By definition, we have
\begin{equation*}
\Phi(\prox_{\Phi/2\ell}(\x)) \leq \Phi(\prox_{\Phi/2\ell}(\x)) + \ell\|\prox_{\Phi/2\ell}(\x) - \x\|^2 \leq \Phi(\x), \quad \textnormal{for all } \x \in \br^m.  
\end{equation*}
Finally,~\citet[Lemma~2.2]{Davis-2019-Stochastic} implies that the Moreau envelope $\Phi_{1/2\ell}$ is $\ell$-smooth with $\grad\Phi_{1/2\ell}(\x) = 2\ell(\x - \prox_{\Phi/2\ell}(\x))$. 

\subsection{Proof of Lemma~\ref{Lemma:ME-nonsmooth}}
Fixing $\x$ and using that $f(\x, \y)$ is $\rho$-weakly convex in $\x$, the function $f(\cdot, \y) + 0.5\rho\|\cdot - \x\|^2$ is convex for any $\y \in \YCal$. By applying the same argument used for proving Lemma~\ref{Lemma:ME-smooth}, we have that $\Phi_{1/2\rho}$ and $\prox_{\Phi/2\rho}(\cdot)$ are well defined, $\Phi(\prox_{\Phi/2\rho}(\x)) \leq \Phi(\x)$ for all $\x \in \br^m$, and $\Phi_{1/2\rho}$ is $\rho$-smooth with $\grad\Phi_{1/2\rho}(\x) = 2\rho(\x - \prox_{\Phi/2\rho}(\x))$. 

\subsection{Proof of Lemma~\ref{Lemma:notion-nonsmooth}}
It suffices to show that $\hat{\x} = \prox_{\Phi/2\ell}(\x)$ satisfies $\min_{\xi \in \partial \Phi(\hat{\x})} \|\xi\| \leq \epsilon$ and $\|\x - \hat{\x}\| \leq \epsilon/2\rho$. Indeed,~\citet[Lemma~2.2]{Davis-2019-Stochastic} guarantees that $\grad\Phi_{1/2\rho}(\x) = 2\rho(\x - \hat{\x})$ where $\hat{\x}=\prox_{\Phi/2\rho}(\x)$ and $\|\hat{\x} - \x\| = \|\grad\Phi_{1/2\rho}(\x)\|/2\rho$. By the definition of $\prox_{\Phi/2\rho}(\cdot)$, we have $2\rho(\x - \hat{\x}) \in \partial \Phi(\hat{\x})$. Putting these pieces together yields $\min_{\xi \in \subg\Phi(\hat{\x})} \|\xi\| \leq \|\grad\Phi_{1/2\rho}(\x)\|$. Since $\|\grad\Phi_{1/2\rho}(\x)\| \leq \epsilon$, we achieve the desired results. 

\subsection{Proof of Proposition~\ref{Prop:notion-smooth}} 
Given that $\hat{\x}$ satisfies $\|\grad\Phi(\hat{\x})\| \leq \epsilon$, the function $f(\hat{\x}, \y)$ is $\ell$-smooth and $\mu$-strongly concave in $\y$, and $\y^\star(\hat{\x}) = \argmax_{\y \in \YCal} f(\hat{\x}, \y)$ is uniquely defined. By applying the standard gradient ascent algorithm for the optimization problem of $\max_{\y \in \YCal} f(\hat{\x}, \cdot)$, we obtain a point $\y' \in \YCal$ satisfying that, for $\y^+ = \proj_\YCal(\y'+\tfrac{1}{\ell}\grady f(\hat{\x}, \y'))$, the following inequalities hold true, 
\begin{equation}\label{cond:notion-smooth}
\|\y^+ - \y'\| \leq \tfrac{\epsilon}{\ell}, \quad \|\y^+ - \y^\star(\hat{\x})\| \leq \tfrac{\epsilon}{\ell}. 
\end{equation}
Since $\|\grad\Phi(\hat{\x})\| \leq \epsilon$ and $\grad \Phi(\hat{\x}) = \gradx f(\hat{\x}, \y^\star(\hat{\x}))$, we have 
\begin{equation*}
\|\gradx f(\hat{\x}, \y^+)\| \leq \|\gradx f(\hat{\x}, \y^+) - \grad \Phi(\hat{\x})\| + \|\grad \Phi(\hat{\x})\| = \|\gradx f(\hat{\x}, \y^+) - \gradx f(\hat{\x}, \y^\star(\hat{\x}))\| + \epsilon. 
\end{equation*}
Since $f$ is $\ell$-smooth, we have $\|\gradx f(\hat{\x}, \y^+)\| \leq \ell\|\y^+ - \y^\star(\hat{\x})\| + \epsilon \leq 2\epsilon$ and the convergence guarantee of standard gradient descent method implies that the required number of gradient evaluations to output a point $\y' \in \YCal$ satisfying Eq.~\eqref{cond:notion-smooth} is $\OCal(\kappa\log(1/\epsilon))$~\citep{Nesterov-2018-Lectures}. The similar argument holds for applying stochastic gradient descent with proper stepsize and the required number of stochastic gradient evaluations is $\OCal(\epsilon^{-2})$. 

Conversely, given that $(\hat{\x}, \hat{\y}) \in \br^m \times \YCal$ is a $\frac{\epsilon}{2\kappa}$-stationary point satisfying
\begin{equation*}
\|\gradx f(\hat{\x}, \hat{\y}^+)\| \leq \tfrac{\epsilon}{2\kappa}, \quad \|\hat{\y}^+ - \hat{\y}\| \leq \tfrac{\epsilon}{2\kappa\ell}, 
\end{equation*}
where $\hat{\y}^+ = \proj_\YCal(\hat{\y} + (1/\ell) \grady f(\hat{\x}, \hat{\y}))$, we have 
\begin{equation*}
\|\nabla \Phi(\hat{\x})\| \leq \|\nabla \Phi(\hat{\x}) - \nabla_\x f(\hat{\x}, \hat{\y}^+)\| + \|\nabla_\x f(\hat{\x}, \hat{\y}^+)\| \leq \ell\|\hat{\y}^+ - \y^\star(\hat{\x})\| + \tfrac{\epsilon}{2\kappa}. 
\end{equation*}
By the definition of $\hat{\y}^+$ and $\y^\star(\hat{\x})$, we have
\begin{equation*}
(\y^\star(\hat{\x}) - \hat{\y}^+)^\top\left(\hat{\y}^+ - \hat{\y} - \tfrac{1}{\ell} \grady f(\hat{\x}, \hat{\y})\right) \geq 0, 
\end{equation*}
which implies 
\begin{equation*}
\ell(\y^\star(\hat{\x}) - \hat{\y}^+)^\top(\hat{\y}^+ - \hat{\y}) \geq (\y^\star(\hat{\x}) - \hat{\y}^+)^\top \grady f(\hat{\x}, \hat{\y}). 
\end{equation*}
Since the function $f(\hat{\x}, \cdot)$ is $\ell$-smooth and $\mu$-strongly concave, we have
\begin{equation*}
\ell(\y^\star(\hat{\x}) - \hat{\y}^+)^\top(\hat{\y}^+ - \hat{\y}) \geq \tfrac{\mu}{2}\|\y^\star(\hat{\x}) - \hat{\y}\|^2 - \tfrac{\ell}{2}\|\hat{\y}^+ - \hat{\y}\|^2. 
\end{equation*}
In addition, we have
\begin{equation*}
(\y^\star(\hat{\x}) - \hat{\y}^+)^\top(\hat{\y}^+ - \hat{\y}) \leq \|\y^\star(\hat{\x}) - \hat{\y}\|\|\hat{\y}^+ - \hat{\y}\| - \|\hat{\y}^+ - \hat{\y}\|^2. 
\end{equation*}
Putting these pieces together yields $\|\hat{\y}^+ - \y^\star(\hat{\x})\| \leq 2\kappa\|\hat{\y} - \hat{\y}^+\| \leq \frac{\epsilon}{\ell}$. This together with $\|\hat{\y}^+ - \y^\star(\hat{\x})\| \leq \|\hat{\y} - \y^\star(\hat{\x})\|$ yields $\|\nabla \Phi(\hat{\x})\| \leq \epsilon + \frac{\epsilon}{2\kappa} \leq 2\epsilon$ which implies the desired result. 

\subsection{Proof of Proposition~\ref{Prop:notion-nonsmooth}} 
Given that $\hat{\x}$ satisfies $\|\grad\Phi_{1/2\ell}(\hat{\x})\| \leq\epsilon$, the function $f(\x, \y) + \ell\|\x - \hat{\x}\|^2$ is $0.5\ell$-strongly convex in $\x$ and concave in $\y$, and $\x^\star(\hat{\x}) = \argmin_\w \Phi(\w) + \ell\|\w - \hat{\x}\|^2$ is uniquely defined. By applying the standard extragradient algorithm for solving the optimization problem of $\min_\x \max_{\y \in \YCal} f(\x, \y) + \ell\|\x - \hat{\x}\|^2$, we obtain a pair of points $(\x', \y') \in \br^m \times \YCal$ satisfying that, for $\y^+ = \proj_\YCal(\y'+\tfrac{1}{\ell}\grady f(\x', \y'))$, the following inequalities hold true, 
\begin{equation}\label{cond:notion-nonsmooth-first}
\|\nabla_\x f(\x', \y^+) + 2\ell(\x' - \hat{\x})\| \leq \epsilon, \quad \|\x' - \x^\star(\hat{\x})\| \leq \tfrac{\epsilon}{2\ell}, \quad \|\y^+ - \y'\| \leq \tfrac{\epsilon}{\ell}. 
\end{equation}
Since $\|\x^\star(\hat{\x}) - \hat{\x}\| = \frac{1}{2\ell}\|\nabla \Phi_{1/2\ell}(\hat{\x})\| \leq \frac{\epsilon}{2\ell}$, we have 
\begin{equation*}
\|\nabla_\x f(\x', \y^+)\| \leq \|\nabla_\x f(\x', \y^+) + 2\ell(\x' - \hat{\x})\| + 2\ell\|\x' - \x^\star(\hat{\x})\| + 2\ell\|\x^\star(\hat{\x}) - \hat{\x}\| \leq 4\epsilon. 
\end{equation*}
To facilitate the reader's understanding, we give a brief description of the extragradient algorithm. It was introduced by~\citet{Korpelevich-1976-Extragradient} and has been recognized as one of the most useful methods for solving minimax optimization problems. In the convex-concave setting, its global convergence rate guarantees were established for both averaged iterates~\citep{Nemirovski-2004-Prox, Mokhtari-2020-Convergence, Mokhtari-2020-Unified} and last iterates~\citep{Gorbunov-2022-Extragradient, Cai-2022-Finite}. When applied to solve the optimization problem of $\min_\x \max_{\y \in \YCal} f(\x, \y) + \ell\|\x - \hat{\x}\|^2$, one iteration of the extragradient algorithm is summarized as follows,  
\begin{equation*}
\begin{array}{ll}
\tilde{\x}_t \leftarrow \x_{t-1} - \eta(\gradx f(\x_{t-1}, \y_{t-1}) + 2\ell(\x_{t-1} - \hat{\x})), & \tilde{\y}_t \leftarrow \y_{t-1} + \eta\grady f(\x_{t-1}, \y_{t-1}), \\
\x_t \leftarrow \x_{t-1} - \eta(\gradx f(\tilde{\x}_t, \tilde{\y}_t) + 2\ell(\tilde{\x}_t - \hat{\x})), & \y_t \leftarrow \y_{t-1} + \eta\grady f(\tilde{\x}_t, \tilde{\y}_t). 
\end{array}
\end{equation*}
The convergence guarantee of extragradient algorithm implies that the required number of gradient evaluations to output a pair of points $(\x', \y') \in \br^m \times \YCal$ satisfying Eq.~\eqref{cond:notion-nonsmooth-first} is $O(\epsilon^{-2})$. The similar argument holds for applying stochastic mirror-prox algorithm and the required number of stochastic gradient evaluations is $O(\epsilon^{-4})$~\citep{Juditsky-2011-Solving}. 

Conversely, given that $(\hat{\x}, \hat{\y})$ is a $\frac{\epsilon^2}{\max\{1, \ell D\}}$-stationary point satisfying 
\begin{equation}\label{cond:notion-nonsmooth-second}
\|\gradx f(\hat{\x}, \hat{\y}^+)\| \leq \tfrac{\epsilon^2}{\max\{1, \ell D\}}, \quad \|\hat{\y}^+ - \hat{\y}\| \leq \tfrac{\epsilon^2}{\ell\max\{1, \ell D\}}, 
\end{equation}
where $\hat{\y}^+ = \proj_\YCal(\hat{\y} + (1/\ell) \grady f(\hat{\x}, \hat{\y}))$, we have
\begin{equation*}
\|\grad \Phi_{1/2\ell}(\hat{\x})\|^2 = 4\ell^2\|\hat{\x} - \x^\star(\hat{\x})\|^2. 
\end{equation*}
Since $\Phi(\cdot) + \ell\|\cdot - \hat{\x}\|^2$ is $0.5\ell$-strongly-convex, we have
\begin{eqnarray}\label{inequality:notion-nonsmooth-first}
\lefteqn{\max_{\y \in \YCal} f(\hat{\x}, \y) - \max_{\y \in \YCal} f(\x^\star(\hat{\x}), \y) - \ell\|\x^\star(\hat{\x}) - \hat{\x}\|^2} \\ 
& = & \Phi(\hat{\x}) - \Phi(\x^\star(\hat{\x})) - \ell\|\x^\star(\hat{\x}) - \hat{\x}\|^2 \geq \tfrac{\ell}{4}\|\hat{\x} - \x^\star(\hat{\x})\|^2 = \tfrac{1}{16\ell}\|\grad \Phi_{1/2\ell}(\hat{\x})\|^2. \nonumber
\end{eqnarray}
Further, we have
\begin{eqnarray*}
\lefteqn{\max_{\y \in \YCal} f(\hat{\x}, \y) - \max_{\y \in \YCal} f(\x^\star(\hat{\x}), \y) - \ell\|\x^\star(\hat{\x}) - \hat{\x}\|^2} \\
& = & \max_{\y \in \YCal} f(\hat{\x}, \y) - f(\hat{\x}, \hat{\y}^+) + f(\hat{\x}, \hat{\y}^+) - \max_{\y \in \YCal} f(\x^\star(\hat{\x}), \y) - \ell\|\x^\star(\hat{\x}) - \hat{\x}\|^2 \\
& \leq & \max_{\y \in \YCal} f(\hat{\x}, \y) - f(\hat{\x}, \hat{\y}^+) + (f(\hat{\x}, \hat{\y}^+) - f(\x^\star(\hat{\x}), \hat{\y}^+) - \ell\|\x^\star(\hat{\x}) - \hat{\x}\|^2) \\ 
& \leq & \max_{\y \in \YCal} f(\hat{\x}, \y) - f(\hat{\x}, \hat{\y}^+) + (\|\hat{\x} - \x^\star(\hat{\x})\|\|\gradx f(\hat{\x}, \hat{\y}^+)\| - \tfrac{\ell}{2} \|\hat{\x} - \x^\star(\hat{\x})\|^2) \\
& \leq & \max_{\y \in \YCal} f(\hat{\x}, \y) - f(\hat{\x}, \hat{\y}^+) + \tfrac{1}{2\ell}\|\gradx f(\hat{\x}, \hat{\y}^+)\|^2. 
\end{eqnarray*}
By the definition of $\hat{\y}^+$, we have $(\y - \hat{\y}^+)^\top\left(\hat{\y}^+ - \hat{\y} - (1/\ell)\grady f(\hat{\x}, \hat{\y})\right) \geq 0$ for all $\y \in \YCal$. Since the function $f(\hat{\x}, \cdot)$ is $\ell$-smooth and the diameter of $\YCal$ is $D$, we have
\begin{equation*}
f(\hat{\x}, \y) - f(\hat{\x}, \hat{\y}^+) \leq \ell D\|\hat{\y}^+ - \hat{\y}\| \textnormal{ for all } \y \in \YCal. 
\end{equation*}
Putting these pieces together yields
\begin{equation*}
\max_{\y \in \YCal} f(\hat{\x}, \y) - \max_{\y \in \YCal} f(\x^\star(\hat{\x}), \y) - \ell\|\x^\star(\hat{\x}) - \hat{\x}\|^2 \leq \ell D\|\hat{\y}^+ - \hat{\y}\| + \tfrac{1}{2\ell}\|\gradx f(\hat{\x}, \hat{\y}^+)\|^2. 
\end{equation*}
Plugging Eq.~\eqref{cond:notion-nonsmooth-second} into the above inequality yields
\begin{equation}\label{inequality:notion-nonsmooth-second}
\max_{\y \in \YCal} f(\hat{\x}, \y) - \max_{\y \in \YCal} f(\x^\star(\hat{\x}), \y) - \ell\|\x^\star(\hat{\x}) - \hat{\x}\|^2 \leq \tfrac{\epsilon^2}{\ell} + \tfrac{\epsilon^4}{2\ell}. 
\end{equation}
Combining Eq.~\eqref{inequality:notion-nonsmooth-first} with Eq.~\eqref{inequality:notion-nonsmooth-second} yields $\|\grad \Phi_{1/2\ell}(\hat{\x})\| \leq 4\epsilon^2 + 4\epsilon$ and the desired result. 

\subsection{Proof of Lemma~\ref{Lemma:smooth}}
We prove the results under Assumption~\ref{Assumption:nsc-smooth}. Indeed, we have that $\YCal^\star(\x) = \{\y^\star(\x)\}$ since the function $f(\x, \cdot)$ is strongly concave for each $\x$. Then, let $\x_1, \x_2 \in \br^m$, the definition of $\y^\star(\x_1)$ and the definition of $\y^\star(\x_2)$ imply
\begin{align}
(\y-\y^\star(\x_1))^\top\grady f(\x_1, \y^\star(\x_1)) & \leq 0, \quad \textnormal{for all } \y \in \YCal, \label{inequality:smooth-first} \\
(\y-\y^\star(\x_2))^\top\grady f(\x_2, \y^\star(\x_2)) & \leq 0, \quad \textnormal{for all } \y \in \YCal. \label{inequality:smooth-second}
\end{align}
Letting $\y=\y^\star(\x_2)$ in Eq.~\eqref{inequality:smooth-first} and $\y=\y^\star(\x_1)$ in Eq.~\eqref{inequality:smooth-second} and adding them yields
\begin{equation}\label{inequality:smooth-third}
(\y^\star(\x_2)-\y^\star(\x_1))^\top(\grady f(\x_1, \y^\star(\x_1)) - \grady f(\x_2, \y^\star(\x_2))) \leq 0. 
\end{equation} 
Since the function $f(\x_1, \cdot)$ is $\mu$-strongly concave, we have
\begin{equation}\label{inequality:smooth-fourth}
(\y^\star(\x_2) - \y^\star(\x_1))^\top(\grady f(\x_1, \y^\star(\x_2)) - \grady f(\x_1, \y^\star(\x_1))) \leq - \mu \|\y^\star(\x_2) - \y^\star(\x_1)\|^2. 
\end{equation}
Combining Eq.~\eqref{inequality:smooth-third} and Eq.~\eqref{inequality:smooth-fourth} with the fact that $f$ is $\ell$-smooth yields
\begin{equation*}
\mu\|\y^\star(\x_2) - \y^\star(\x_1)\|^2 \leq \ell\|\y^\star(\x_2) - \y^\star(\x_1)\|\|\x_2 - \x_1\|. 
\end{equation*}
which implies that the function $\y^\star(\cdot)$ is $\kappa$-Lipschitz. 

Since $\y^\star(\cdot)$ is a well-defined function and $\YCal$ is a convex and bounded set, the celebrated Danskin theorem (see~\citet{Danskin-1967-Theory} and~\citet[Proposition~A.22]{Bertsekas-1971-Control}) implies that $\grad \Phi(\cdot) = \gradx f(\cdot, \y^\star(\cdot))$. Since $\y^\star(\cdot)$ is $\kappa$-Lipschitz and $\kappa \geq 1$, we have
\begin{equation*}
\|\grad \Phi(\x) - \grad \Phi(\x')\| = \|\gradx f(\x, \y^\star(\x)) - \gradx f(\x', \y^\star(\x'))\| \leq \ell(\|\x - \x'\| + \kappa\|\x - \x'\|) \leq 2\kappa\ell\|\x - \x'\|. 
\end{equation*}

We then prove the results under Assumption~\ref{Assumption:nc-smooth}. Indeed, Lemma~\ref{Lemma:ME-smooth} implies that the function $\Phi(\cdot)$ is $\ell$-weakly convex. Since $\YCal$ is a convex and bounded set, there exist $\y_1, \y_2 \in \YCal$ such that 
\begin{equation*}
f(\x_1, \y_1) = \max_{\y \in \YCal} f(\x_1, \y), \quad f(\x_2, \y_2) = \max_{\y \in \YCal} f(\x_2, \y), \quad \textnormal{for any given } \x_1, \x_2 \in \br^m. 
\end{equation*}
Thus, we have
\begin{eqnarray*}
\Phi(\x_1) - \Phi(\x_2) & = & f(\x_1, \y_1) - f(\x_2, \y_2) \ \leq \ f(\x_1, \y_1) - f(\x_2, \y_1), \\ 
\Phi(\x_2) - \Phi(\x_1) & = & f(\x_2, \y_2) - f(\x_1, \y_1) \ \leq \ f(\x_2, \y_2) - f(\x_1, \y_2). 
\end{eqnarray*}
Since the function $f(\cdot, \y)$ is $L$-Lipschitz for each $\y \in \YCal$, we have
\begin{equation*}
|\Phi(\x_1) - \Phi(\x_2)| = \max\left\{\Phi(\x_1) - \Phi(\x_2), \Phi(\x_2) - \Phi(\x_1)\right\} \leq L\|\x_1 - \x_2\|. 
\end{equation*}
Further, we have $\partial\Phi(\x) = \partial\Psi(\x) - \ell\x$ for each $\x$ where $\Psi(\x) = \max_{\y \in \YCal} f(\x, \y) + 0.5\ell\|\x\|^2$. Since $f(\cdot, \y) + 0.5\ell\|\cdot\|^2$ is convex for each $\y \in \YCal$ where $\YCal$ is convex and bounded, we have that $\gradx f(\x, \y) + \ell\x \in \partial\Psi(\x)$ where $\y \in \YCal^\star(\x)$. This yields the desired result. 

\subsection{Proof of Lemma~\ref{Lemma:nonsmooth}}
We prove the results under Assumption~\ref{Assumption:nc-nonsmooth}. Indeed, Lemma~\ref{Lemma:ME-nonsmooth} implies that the function $\Phi(\cdot)$ is $\rho$-weakly convex. Given any $\x_1, \x_2$, we have that there exist $\y_1, \y_2 \in \YCal$ such that 
\begin{eqnarray*}
\Phi(\x_1) - \Phi(\x_2) & = & f(\x_1, \y_1) - f(\x_2, \y_2) \ \leq \ f(\x_1, \y_1) - f(\x_2, \y_1), \\ 
\Phi(\x_2) - \Phi(\x_1) & = & f(\x_2, \y_2) - f(\x_1, \y_1) \ \leq \ f(\x_2, \y_2) - f(\x_1, \y_2). 
\end{eqnarray*}
Since $f$ is $L$-Lipschitz, we have
\begin{equation*}
|\Phi(\x_1) - \Phi(\x_2)| \leq L\|\x_1 - \x_2\|. 
\end{equation*}
Applying the same argument used for proving Lemma~\ref{Lemma:smooth}, we have that $\subgx f(\x, \y) \subseteq \partial\Phi(\x)$ for each $\x$ where $\y \in \YCal^\star(\x)$. 

We then prove the results under Assumption~\ref{Assumption:nsc-nonsmooth}. Since the function $f(\x, \cdot)$ is $\mu$-strongly concave for each $\x$, we have that $\YCal^\star(\x) = \{\y^\star(\x)\}$. Since Assumption~\ref{Assumption:nsc-nonsmooth} is stronger than Assumption~\ref{Assumption:nc-nonsmooth}, all other desired results follow from the previous analysis.  

\subsection{Auxiliary Results on Stochastic Gradient}
We establish the properties of stochastic gradients sampled at each iteration. 
\begin{lemma}\label{Lemma:SG-unbiased}
If Eq.~\eqref{Assumption:SGDA} holds true, the estimators $(\stocgx^t, \stocgy^t) = \textsf{SG}(G, \{\xi_i^t\}_{i=1}^M, \x_t, \y_t)$ satisfy
\begin{equation*}
\begin{array}{ll}
\EE[\stocgx^t \mid \x_t, \y_t] \in \subgx f(\x_t, \y_t), & \EE[\|\stocgx^t - \EE[\stocgx^t \mid \x_t, \y_t]\|^2 \mid \x_t, \y_t] \leq \frac{\sigma^2}{M}, \\
\EE[\stocgy^t \mid \x_t, \y_t] \in \subgy f(\x_t, \y_t), & \EE[\|\stocgy^t - \EE[\stocgy^t \mid \x_t, \y_t]\|^2 \mid \x_t, \y_t] \leq \frac{\sigma^2}{M}. 
\end{array}
\end{equation*}
\end{lemma}
\begin{proof}
Since $G = (G_\x, G_\y)$ is unbiased and 
\begin{equation*}
\stocgx^t = \frac{1}{M}\left(\sum_{i=1}^M G_\x(\x_t, \y_t, \xi_i^t)\right), \quad \stocgy^t = \frac{1}{M}\left(\sum_{i=1}^M G_\y(\x_t, \y_t, \xi_i^t)\right), 
\end{equation*}
we have 
\begin{equation*}
\EE[\stocgx^t \mid \x_t, \y_t] \in \subgx f(\x_t, \y_t), \quad \EE[\stocgy^t \mid \x_t, \y_t] \in \subgy f(\x_t, \y_t). 
\end{equation*}
Furthermore, we have
\begin{equation*}
\EE[\|\stocgx^t - \EE[\stocgx^t \mid \x_t, \y_t]\|^2 \mid \x_t, \y_t] \leq \tfrac{\sigma^2}{M}, \quad \EE[\|\stocgy^t - \EE[\stocgy^t \mid \x_t, \y_t]\|^2 \mid \x_t, \y_t] \leq \tfrac{\sigma^2}{M}. 
\end{equation*}
This completes the proof. 
\end{proof}

\section{Smooth and Nonconvex-Strongly-Concave Setting}
We provide the proof for Theorem~\ref{Thm:nsc-smooth}. The trivial case is $\kappa = 1$ since it is as easy as smooth nonconvex optimization, which has been investigated extensively. For simplicity, we assume $\kappa > 1$ in the subsequent analysis. Throughout this subsection, we set 
\begin{equation*}
\eta_\x^t \equiv \eta_\x = \tfrac{1}{16(\kappa + 1)^2\ell}, \quad \eta_\y^t \equiv \eta_\y = \tfrac{1}{\ell}, \quad \textnormal{for both Algorithms~\ref{Algorithm:TTGDA} and~\ref{Algorithm:TTSGDA}}.
\end{equation*}
\begin{lemma}\label{Lemma:nsc-smooth-descent}
The iterates $\{\x_t\}_{t \geq 1}$ generated by Algorithm~\ref{Algorithm:TTGDA} satisfy
\begin{equation*}
\Phi(\x_t) \leq \Phi(\x_{t-1}) - (\tfrac{\eta_\x}{2} - 2\eta_\x^2\kappa\ell) \|\grad \Phi(\x_{t-1})\|^2 + (\tfrac{\eta_\x}{2} + 2\eta_\x^2\kappa\ell)\|\grad\Phi(\x_{t-1}) - \gradx f(\x_{t-1}, \y_{t-1})\|^2. 
\end{equation*}
The iterates $\{\x_t\}_{t \geq 1}$ generated by Algorithm~\ref{Algorithm:TTSGDA} satisfy
\begin{eqnarray*}
\lefteqn{\EE[\Phi(\x_t)] \leq \EE[\Phi(\x_{t-1})] - (\tfrac{\eta_\x}{2} - 2\eta_\x^2\kappa\ell)\EE[\|\grad \Phi(\x_{t-1})\|^2]} \\
& & + (\tfrac{\eta_\x}{2} + 2\eta_\x^2\kappa\ell)\EE[\|\grad\Phi(\x_{t-1}) - \gradx f(\x_{t-1}, \y_{t-1})\|^2] + \tfrac{\eta_\x^2\kappa\ell\sigma^2}{M}.
\end{eqnarray*}
\end{lemma}
\begin{proof}
We first prove the inequality for Algorithm~\ref{Algorithm:TTGDA}. Since $\Phi(\cdot)$ is $(2\kappa\ell)$-smooth, we have
\begin{equation*}
\Phi(\x_t) - \Phi(\x_{t-1}) - (\x_t - \x_{t-1})^\top\nabla\Phi(\x_{t-1}) \leq \kappa\ell\|\x_t - \x_{t-1}\|^2. 
\end{equation*}
Since $\x_t - \x_{t-1} = - \eta_\x \gradx f(\x_{t-1}, \y_{t-1})$, we have
\begin{equation*}
\Phi(\x_t) \leq \Phi(\x_{t-1})  - \eta_\x (\gradx f(\x_{t-1}, \y_{t-1}))^\top\nabla\Phi(\x_{t-1}) + \eta_\x^2\kappa\ell\|\gradx f(\x_{t-1}, \y_{t-1})\|^2
\end{equation*}
Using Young's inequality yields
\begin{equation*}
\Phi(\x_t) \leq \Phi(\x_{t-1}) + \eta_\x^2\kappa\ell\|\gradx f(\x_{t-1}, \y_{t-1})\|^2 + \tfrac{\eta_\x}{2}(\|\nabla\Phi(\x_{t-1}) - \gradx f(\x_{t-1}, \y_{t-1})\|^2 - \|\nabla\Phi(\x_{t-1})\|^2). 
\end{equation*}
By the Cauchy-Schwartz inequality, we have
\begin{equation*}
\|\gradx f(\x_{t-1}, \y_{t-1})\|^2 \leq 2(\|\nabla\Phi(\x_{t-1}) - \gradx f(\x_{t-1}, \y_{t-1})\|^2 + \|\nabla\Phi(\x_{t-1})\|^2).
\end{equation*}
Putting these pieces yields the inequality for Algorithm~\ref{Algorithm:TTGDA}. 

We then prove the inequality for Algorithm~\ref{Algorithm:TTSGDA}. By using the same argument, we have
\begin{equation*}
\Phi(\x_t) \leq \Phi(\x_{t-1})  - \eta_\x\|\nabla\Phi(\x_{t-1})\|^2 + \eta_\x^2\kappa\ell\|\stocgx^{t-1}\|^2 + \eta_\x(\nabla\Phi(\x_{t-1}) - \stocgx^{t-1})^\top\nabla\Phi(\x_{t-1}).
\end{equation*}
Conditioned on the iterate $(\x_{t-1}, \y_{t-1})$, we take the expectation of both sides of the above inequality (where the expectation is taken over the randomness in the selection of samples $\{\xi_i^{t-1}\}_{i=1}^M$) and use Lemma~\ref{Lemma:SG-unbiased} to yield
\begin{eqnarray*}
\lefteqn{\EE[\Phi(\x_t) \mid \x_{t-1}, \y_{t-1}] \leq \Phi(\x_{t-1})  - \eta_\x\|\nabla\Phi(\x_{t-1})\|^2 + \tfrac{\eta_\x^2\kappa\ell\sigma^2}{M}} \\ 
& & + \eta_\x^2\kappa\ell\|\gradx f(\x_{t-1}, \y_{t-1})\|^2 + \eta_\x(\nabla\Phi(\x_{t-1}) - \gradx f(\x_{t-1}, \y_{t-1}))^\top \nabla\Phi(\x_{t-1}). 
\end{eqnarray*}
By using the same argument, we have
\begin{eqnarray*}
\lefteqn{\EE[\Phi(\x_t) \mid \x_{t-1}, \y_{t-1}] \leq \Phi(\x_{t-1}) - (\tfrac{\eta_\x}{2} - 2\eta_\x^2\kappa\ell) \|\grad \Phi(\x_{t-1})\|^2} \\ 
& & + (\tfrac{\eta_\x}{2} + 2\eta_\x^2\kappa\ell)\|\grad\Phi(\x_{t-1}) - \gradx f(\x_{t-1}, \y_{t-1})\|^2 + \tfrac{\eta_\x^2\kappa\ell\sigma^2}{M}. 
\end{eqnarray*}
We take the expectation of both sides of the above inequality (where the expectation is taken over the randomness in the selection of all previous samples). This yields the inequality for Algorithm~\ref{Algorithm:TTSGDA}.
\end{proof}
In the subsequent analysis, we define $\delta_t = \|\y^\star(\x_t) - \y_t\|^2$ and $\delta_t = \EE[\|\y^\star(\x_t) - \y_t\|^2]$ for the iterates generated by Algorithms~\ref{Algorithm:TTGDA} and~\ref{Algorithm:TTSGDA} respectively. 
\begin{lemma}\label{Lemma:nsc-smooth-neighbor}
The iterates $\{\x_t\}_{t \geq 1}$ generated by Algorithm~\ref{Algorithm:TTGDA} satisfy
\begin{equation*}
\delta_t \leq (1-\tfrac{1}{2\kappa}+4\kappa^3\ell^2\eta_\x^2)\delta_{t-1} + 4\kappa^3\eta_\x^2\|\grad\Phi(\x_{t-1})\|^2.  
\end{equation*}
The iterates $\{\x_t\}_{t \geq 1}$ generated by Algorithm~\ref{Algorithm:TTSGDA} satisfy
\begin{equation*}
\delta_t \leq (1-\tfrac{1}{2\kappa}+4\kappa^3\ell^2\eta_\x^2)\delta_{t-1} + 4\kappa^3\eta_\x^2\EE[\|\grad\Phi(\x_{t-1})\|^2] + \tfrac{2\eta_\x^2\sigma^2\kappa^3}{M} + \tfrac{\sigma^2}{\ell^2 M}.   
\end{equation*}
\end{lemma}
\begin{proof}
We first prove the inequality for Algorithm~\ref{Algorithm:TTGDA}. Since $\eta_\y = \frac{1}{\ell}$, we have
\begin{equation*}
(\y - \y_t)^\top\left(\y_t - \y_{t-1} - \tfrac{1}{\ell} \grady f(\x_{t-1}, \y_{t-1})\right) \geq 0, \quad \textnormal{for all } \y \in \YCal. 
\end{equation*}
Rearranging the above inequality with $\y = \y^\star(\x_{t-1})$ yields 
\begin{equation*}
\ell(\y^\star(\x_{t-1}) - \y_t)^\top(\y_t - \y_{t-1}) \geq (\y^\star(\x_{t-1}) - \y_t)^\top \grady f(\x_{t-1}, \y_{t-1}). 
\end{equation*}
Since the function $f(\x_{t-1}, \cdot)$ is $\ell$-smooth and $\mu$-strongly concave, we have
\begin{equation*}
(\y^\star(\x_{t-1}) - \y_t)^\top \grady f(\x_{t-1}, \y_{t-1}) \geq \tfrac{\mu}{2}\|\\y^\star(\x_{t-1}) - \y_{t-1}\|^2 - \tfrac{\ell}{2}\|\y_t - \y_{t-1}\|^2. 
\end{equation*}
In addition, we have
\begin{equation*}
(\y^\star(\x_{t-1}) - \y_t)^\top(\y_t - \y_{t-1}) = \tfrac{1}{2}(\|\y^\star(\x_{t-1}) - \y_{t-1}\|^2 - \|\y^\star(\x_{t-1}) - \y_t\|^2 - \|\y_t - \y_{t-1}\|^2).  
\end{equation*}
Putting these pieces together yields
\begin{equation}\label{inequality:nsc-smooth-neighbor-first}
\|\y^\star(\x_{t-1}) - \y_t\|^2 \leq (1 - \tfrac{1}{\kappa})\delta_{t-1}. 
\end{equation}
Using Young's inequality and the fact that the function $\y^\star(\cdot)$ is $\kappa$-Lipschitz, we have
\begin{eqnarray*}
\delta_t & \leq & (1 + \tfrac{1}{2\kappa - 2})\|\y^\star(\x_{t-1}) - \y_t\|^2 + (2\kappa - 1)\|\y^\star(\x_t) - \y^\star(\x_{t-1})\|^2 \\
& \leq & \tfrac{2\kappa-1}{2\kappa-2}\|\y^\star(\x_{t-1}) - \y_t\|^2 + 2\kappa\|\y^\star(\x_t) - \y^\star(\x_{t-1})\|^2 \\ 
& \overset{\textnormal{Eq.~\eqref{inequality:nsc-smooth-neighbor-first}}}{\leq} & (1-\tfrac{1}{2\kappa})\delta_{t-1} + 2\kappa\|\y^\star(\x_t) - \y^\star(\x_{t-1})\|^2 \\
& \leq & (1-\tfrac{1}{2\kappa})\delta_{t-1} + 2\kappa^3\|\x_t - \x_{t-1}\|^2. 
\end{eqnarray*}
Furthermore, we have
\begin{equation*}
\|\x_t - \x_{t-1}\|^2 = \eta_\x^2\|\gradx f(\x_{t-1}, \y_{t-1})\|^2 \leq 2\eta_\x^2\left(\ell^2 \delta_{t-1} + \|\grad \Phi(\x_{t-1})\|^2\right). 
\end{equation*}
Putting these pieces together yields the inequality for Algorithm~\ref{Algorithm:TTGDA}. 

We now prove the results for Algorithm~\ref{Algorithm:TTSGDA}. Applying the same argument used for proving Eq.~\eqref{inequality:nsc-smooth-neighbor-first}, we have
\begin{equation*}
\EE[\|\y^\star(\x_{t-1}) - \y_t\|^2] \leq (1 - \tfrac{1}{\kappa})\delta_{t-1} + \tfrac{\sigma^2}{\ell^2 M}, 
\end{equation*}
which further implies
\begin{equation*}
\delta_t \leq (1 - \tfrac{1}{2\kappa})\delta_{t-1} + 2\kappa^3\EE[\|\x_t - \x_{t-1}\|^2] + \tfrac{\sigma^2}{\ell^2 M}.
\end{equation*}
Furthermore, we have
\begin{equation*}
\EE[\|\x_t - \x_{t-1}\|^2] = \eta_\x^2 \EE[\|\stocgx^{t-1}\|^2] \leq 2\eta_\x^2\left(\ell^2 \delta_{t-1} + \EE[\|\grad \Phi(\x_{t-1})\|^2]\right) + \tfrac{\eta_\x^2\sigma^2}{M}.  
\end{equation*}
Putting these pieces together yields the inequality for Algorithm~\ref{Algorithm:TTSGDA}. 
\end{proof}
\begin{lemma}\label{Lemma:nsc-smooth-obj}
The iterates $\{\x_t\}_{t \geq 1}$ generated by Algorithm~\ref{Algorithm:TTGDA} satisfy
\begin{equation*}
\Phi(\x_t) \leq \Phi(\x_{t-1}) - \tfrac{7\eta_\x}{16}\|\grad \Phi(\x_{t-1})\|^2 + \tfrac{9\eta_{\x}\ell^2\delta_{t-1}}{16}. 
\end{equation*}
The iterates $\{\x_t\}_{t \geq 1}$ generated by Algorithm~\ref{Algorithm:TTSGDA} satisfy
\begin{equation*}
\EE[\Phi(\x_t)] \leq \EE[\Phi(\x_{t-1})] - \tfrac{7\eta_\x}{16} \EE[\|\grad \Phi(\x_{t-1})\|^2] + \tfrac{9\eta_\x\ell^2\delta_{t-1}}{16} + \tfrac{\eta_\x^2\kappa\ell\sigma^2}{M}. 
\end{equation*}
\end{lemma}
\begin{proof}
Since $\eta_\x = \frac{1}{16(\kappa+1)^2\ell}$ in Algorithms~\ref{Algorithm:TTGDA} and~\ref{Algorithm:TTSGDA}, we have
\begin{equation}\label{inequality:nsc-smooth-stepsize}
\tfrac{7\eta_\x}{16} \leq \tfrac{\eta_\x}{2} - 2\eta_\x^2\kappa\ell \leq \tfrac{\eta_\x}{2} + 2\eta_\x^2\kappa\ell \leq \tfrac{9\eta_\x}{16}.  
\end{equation}
We first prove the results for Algorithm~\ref{Algorithm:TTGDA}. By combining Eq.~\eqref{inequality:nsc-smooth-stepsize} with the first inequality in Lemma~\ref{Lemma:nsc-smooth-descent}, we have
\begin{equation*}
\Phi(\x_t) \leq \Phi(\x_{t-1}) - \tfrac{7\eta_\x}{16}\|\grad \Phi(\x_{t-1})\|^2 + \tfrac{9\eta_\x}{16}\|\grad\Phi(\x_{t-1}) - \gradx f(\x_{t-1}, \y_{t-1})\|^2. 
\end{equation*}
Since $\grad \Phi(\x_{t-1}) = \gradx f(\x_{t-1}, \y^\star(\x_{t-1}))$, we have
\begin{equation*}
\|\grad\Phi(\x_{t-1}) - \gradx f(\x_{t-1}, \y_{t-1})\|^2 \leq \ell^2\|\y^\star(\x_{t-1}) - \y_{t-1}\|^2 = \ell^2\delta_{t-1}. 
\end{equation*}
Putting these pieces together yields the inequality for Algorithm~\ref{Algorithm:TTGDA}. 

We then prove the results for Algorithm~\ref{Algorithm:TTSGDA}. By combining Eq.~\eqref{inequality:nsc-smooth-stepsize} with the second inequality in Lemma~\ref{Lemma:nsc-smooth-descent}, we have
\begin{equation*}
\EE[\Phi(\x_t)] \leq \EE[\Phi(\x_{t-1})] - \tfrac{7\eta_\x}{16}\EE[\|\grad \Phi(\x_{t-1})\|^2] + \tfrac{9\eta_\x}{16}\EE[\|\grad\Phi(\x_{t-1}) - \gradx f(\x_{t-1}, \y_{t-1})\|^2] + \tfrac{\eta_\x^2\kappa\ell\sigma^2}{M}. 
\end{equation*}
Since $\grad \Phi(\x_{t-1}) = \gradx f(\x_{t-1}, \y^\star(\x_{t-1}))$, we have
\begin{equation*}
\EE[\|\grad\Phi(\x_{t-1}) - \gradx f(\x_{t-1}, \y_{t-1})\|^2] \leq \ell^2\EE[\|\y^\star(\x_{t-1}) - \y_{t-1}\|^2] = \ell^2\delta_{t-1}. 
\end{equation*}
Putting these pieces together yields the inequality for Algorithm~\ref{Algorithm:TTSGDA}. 
\end{proof}

\paragraph{Proof of Theorem~\ref{Thm:nsc-smooth}.} For simplicity, we define $\gamma = 1 - \frac{1}{2\kappa} + 4\kappa^3\ell^2\eta_\x^2$. 

We first prove the results for Algorithm~\ref{Algorithm:TTGDA}. Since $\delta_t = \|\y^\star(\x_t) - \y_t\|^2$, we have $\delta_t \leq D^2$. Performing the first inequality in Lemma~\ref{Lemma:nsc-smooth-neighbor} recursively and using $\delta_0 \leq D^2$ yields
\begin{equation}\label{inequality:nsc-smooth-first}
\delta_t \leq \gamma^t D^2 + 4\kappa^3\eta_\x^2\left(\sum_{j=0}^{t-1} \gamma^{t-1-j} \|\grad \Phi(\x_j)\|^2 \right). 
\end{equation} 
Combining Eq.~\eqref{inequality:nsc-smooth-first} with the first inequality in Lemma~\ref{Lemma:nsc-smooth-obj} yields
\begin{equation}\label{inequality:nsc-smooth-second}
\Phi(\x_t) \leq \Phi(\x_{t-1}) - \tfrac{7\eta_\x}{16} \|\grad \Phi(\x_{t-1})\|^2 + \tfrac{9\eta_\x\ell^2\gamma^{t-1} D^2}{16} + \tfrac{9\eta_\x^3\ell^2\kappa^3}{4}\left(\sum_{j=0}^{t-2} \gamma^{t-2-j} \|\grad \Phi(\x_j)\|^2 \right).
\end{equation}
Summing Eq.~\eqref{inequality:nsc-smooth-second} over $t=1, 2, \ldots, T+1$ and rearranging the terms yields
\begin{equation*}
\Phi(\x_{T+1}) \leq \Phi(\x_0) - \tfrac{7\eta_{\x}}{16}\left(\sum_{t=0}^T \|\grad \Phi(\x_t)\|^2\right) + \tfrac{9\eta_{\x}\ell^2 D^2}{16}\left(\sum_{t=0}^T \gamma^t \right) + \tfrac{9\eta_\x^3\ell^2\kappa^3}{4}\left(\sum_{t=1}^{T+1}\sum_{j=0}^{t-2} \gamma^{t-2-j}\|\grad \Phi(\x_j)\|^2\right).
\end{equation*}
Since $\eta_{\x} = \frac{1}{16(\kappa+1)^2\ell}$, we have $\gamma \leq 1 - \frac{1}{4\kappa}$ and $\frac{9\eta_\x^3\ell^2\kappa^3}{4} \leq \frac{9\eta_\x}{1024\kappa}$. This yields $\sum_{t=0}^T \gamma^t \leq 4\kappa$ and
\begin{equation*}
\sum_{t=1}^{T+1}\sum_{j=0}^{t-2} \gamma^{t-2-j} \|\grad \Phi(\x_j)\|^2 \leq 4\kappa\left(\sum_{t=0}^T \|\grad \Phi(\x_t)\|^2\right) 
\end{equation*}
Putting these pieces together yields
\begin{equation*}
\Phi(\x_{T+1}) \leq \Phi(\x_0) - \tfrac{103\eta_\x}{256}\left(\sum_{t=0}^T \|\grad \Phi(\x_t)\|^2\right) + \tfrac{9\eta_{\x}\kappa\ell^2 D^2}{4}. 
\end{equation*}
By the definition of $\Delta_\Phi$, we have
\begin{equation*}
\tfrac{1}{T+1}\left(\sum_{t=0}^T \|\grad \Phi(\x_t)\|^2\right) \leq \tfrac{256(\Phi(\x_0) - \Phi(\x_{T+1}))}{103\eta_\x(T+1)} + \tfrac{576\kappa\ell^2D^2}{103(T+1)} \leq \tfrac{128\kappa^2 \ell \Delta_\Phi + 5\kappa\ell^2 D^2}{T+1}. 
\end{equation*}
This implies that the number of gradient evaluations required by Algorithm \ref{Algorithm:TTGDA} to return an $\epsilon$-stationary point is
\begin{equation*}
O\left(\tfrac{\kappa^2\ell\Delta_\Phi + \kappa\ell^2 D^2}{\epsilon^2}\right). 
\end{equation*}
We then prove the results for Algorithm~\ref{Algorithm:TTSGDA}. Applying the same argument used for analyzing Algorithm~\ref{Algorithm:TTGDA} but with the second inequalities in Lemmas~\ref{Lemma:nsc-smooth-neighbor} and~\ref{Lemma:nsc-smooth-obj}, we have 
\begin{eqnarray*}
\tfrac{1}{T+1}\left(\sum_{t=0}^T \EE[\|\grad \Phi(\x_t)\|^2]\right) & \leq & \tfrac{256(\Phi(\x_0) - \EE\left[\Phi(\x_{T+1})\right])}{103\eta_\x(T+1)} + \tfrac{576\kappa\ell^2D^2}{103(T+1)} + \tfrac{2304\kappa\sigma^2}{103M}\\
& \leq & \tfrac{128\kappa^2 \ell\Delta_\Phi + 5\kappa\ell^2 D^2}{T+1} + \tfrac{24\kappa\sigma^2}{M}. 
\end{eqnarray*}
This implies that the number of stochastic gradient evaluations required by Algorithm \ref{Algorithm:TTSGDA} to return an $\epsilon$-stationary point is
\begin{equation*}
O\left(\tfrac{\kappa^2\ell\Delta_\Phi + \kappa\ell^2 D^2}{\epsilon^2}\max\left\{1, \tfrac{\kappa\sigma^2}{\epsilon^2}\right\}\right). 
\end{equation*}
This completes the proof. 

\section{Smooth and Nonconvex-Concave Setting}
We provide the proof for Theorem~\ref{Thm:nc-smooth}. Throughout this subsection, we set
\begin{equation*}
\eta_\x^t \equiv \eta_\x = \min\{\tfrac{\epsilon^2}{80\ell L^2}, \tfrac{\epsilon^4}{4096\ell^3 L^2 D^2}\}, \quad \eta_\y^t \equiv \eta_\y = \tfrac{1}{\ell}, 
\end{equation*}
for Algorithm~\ref{Algorithm:TTGDA}, and 
\begin{equation*}
\eta_\x^t \equiv \eta_\x = \min\left\{\tfrac{\epsilon^2}{80\ell(L^2 + \sigma^2)}, \tfrac{\epsilon^4}{8192\ell^3 (L^2+\sigma^2)D^2}, \tfrac{\epsilon^6}{131072\ell^3 (L^2+\sigma^2) D^2\sigma^2}\right\}, \quad \eta_\y^t \equiv \eta_\y = \min\left\{\tfrac{1}{2\ell}, \tfrac{\epsilon^2}{32\ell\sigma^2}\right\}, 
\end{equation*}
for Algorithm~\ref{Algorithm:TTSGDA}. We define $\Delta_t = \Phi(\x_t) - f(\x_t, \y_t)$ and $\Delta_t = \EE[\Phi(\x_t) - f(\x_t, \y_t)]$ for the iterates generated by Algorithms~\ref{Algorithm:TTGDA} and~\ref{Algorithm:TTSGDA} respectively. 
\begin{lemma}\label{Lemma:nc-smooth-descent}
The iterates $\{\x_t\}_{t \geq 1}$ generated by Algorithm~\ref{Algorithm:TTGDA} satisfy
\begin{equation*}
\Phi_{1/2\ell}(\x_t) \leq \Phi_{1/2\ell}(\x_{t-1}) + 2\eta_\x\ell\Delta_{t-1} - \tfrac{1}{4}\eta_\x\|\grad \Phi_{1/2\ell}(\x_{t-1})\|^2 + \eta_\x^2 \ell L^2. 
\end{equation*}
The iterates $\{\x_t\}_{t \geq 1}$ generated by Algorithm~\ref{Algorithm:TTSGDA} satisfy
\begin{equation*}
\EE[\Phi_{1/2\ell}(\x_t)] \leq \EE[\Phi_{1/2\ell}(\x_{t-1})] + 2\eta_\x\ell\Delta_{t-1} - \tfrac{1}{4}\eta_\x\EE[\|\grad \Phi_{1/2\ell}(\x_{t-1})\|^2] + \eta_\x^2 \ell(L^2 + \sigma^2).
\end{equation*}
\end{lemma}
\begin{proof}
We first prove the results for Algorithm~\ref{Algorithm:TTGDA}. Letting $\hat{\x}_{t-1} = \prox_{\Phi/2\ell}(\x_{t-1})$, we have
\begin{equation*}
\Phi_{1/2\ell}(\x_t) \leq \Phi(\hat{\x}_{t-1}) + \ell\|\hat{\x}_{t-1} - \x_t\|^2. 
\end{equation*}
Since the function $f(\cdot, \y)$ is $L$-Lipschitz for each $\y \in \YCal$, we have $\|\gradx f(\x, \y)\| \leq L$ for each pair of $\x$ and $\y \in \YCal$. This yields 
\begin{equation*}
\|\hat{\x}_{t-1} - \x_t\|^2 \leq \|\hat{\x}_{t-1} - \x_{t-1}\|^2 + 2\eta_\x \langle\hat{\x}_{t-1} - \x_{t-1}, \gradx f(\x_{t-1}, \y_{t-1})\rangle + \eta_\x^2 L^2. 
\end{equation*}
Combining the above two inequalities yields 
\begin{equation}\label{inequality:nc-smooth-descent-first}
\Phi_{1/2\ell}(\x_t) \leq \Phi_{1/2\ell}(\x_{t-1}) + 2\eta_\x \ell\langle\hat{\x}_{t-1} - \x_{t-1}, \gradx f(\x_{t-1}, \y_{t-1})\rangle + \eta_\x^2 \ell L^2.
\end{equation}
Since the function $f(\cdot, \y_{t-1})$ is $\ell$-smooth, we have
\begin{equation}\label{inequality:nc-smooth-descent-second}
\langle\hat{\x}_{t-1} - \x_{t-1}, \gradx f(\x_{t-1}, \y_{t-1})\rangle \leq f(\hat{\x}_{t-1}, \y_{t-1}) - f(\x_{t-1}, \y_{t-1}) + \tfrac{\ell}{2}\|\hat{\x}_{t-1} - \x_{t-1}\|^2. 
\end{equation}
Since $\Phi(\x) = \max_{\y \in \YCal} f(\x, \y)$ and $\hat{\x}_{t-1} = \prox_{\Phi/2\ell}(\x_{t-1})$, we have
\begin{equation*}
f(\hat{\x}_{t-1}, \y_{t-1}) + \ell\|\hat{\x}_{t-1} - \x_{t-1}\|^2 \leq \Phi(\hat{\x}_{t-1}) + \ell\|\hat{\x}_{t-1} - \x_{t-1}\|^2 \leq \Phi(\x_{t-1}). 
\end{equation*}
Since $\|\grad\Phi_{1/2\ell}(\x_{t-1})\| = 2\ell\|\hat{\x}_{t-1} - \x_{t-1}\|$, we have
\begin{equation}\label{inequality:nc-smooth-descent-third}
f(\hat{\x}_{t-1}, \y_{t-1}) - f(\x_{t-1}, \y_{t-1}) + \tfrac{\ell}{2}\|\hat{\x}_{t-1} - \x_{t-1}\|^2 \leq \Phi(\x_{t-1}) - f(\x_{t-1}, \y_{t-1}) - \tfrac{1}{8\ell}\|\grad\Phi_{1/2\ell}(\x_{t-1})\|^2.
\end{equation}
Plugging Eq.~\eqref{inequality:nc-smooth-descent-second} and Eq.~\eqref{inequality:nc-smooth-descent-third} into Eq.~\eqref{inequality:nc-smooth-descent-first} yields the inequality for Algorithm~\ref{Algorithm:TTGDA}.  

We then prove the results for Algorithm~\ref{Algorithm:TTSGDA}. By using the same argument, we have
\begin{eqnarray*}
\|\hat{\x}_{t-1} - \x_t\|^2 & = & \|\hat{\x}_{t-1} - \x_{t-1}\|^2 + 2\eta_\x \langle \hat{\x}_{t-1} - \x_{t-1}, \stocgx^{t-1}\rangle + \eta_\x^2\|\stocgx^{t-1} - \gradx f(\x_{t-1}, \y_{t-1})\|^2 \\
& & + 2\eta_\x^2\langle \stocgx^{t-1} - \gradx f(\x_{t-1}, \y_{t-1}), \gradx f(\x_{t-1}, \y_{t-1})\rangle +  \eta_\x^2L^2. 
\end{eqnarray*}
Conditioned on the iterate $(\x_{t-1}, \y_{t-1})$, we take the expectation of both sides of the above inequality (where the expectation is taken over the randomness in the selection of samples $\{\xi_i^{t-1}\}_{i=1}^M$) and use Lemma~\ref{Lemma:SG-unbiased} with $M=1$ to yield
\begin{equation*}
\EE[\|\hat{\x}_{t-1} - \x_t\|^2 \mid \x_{t-1}, \y_{t-1}] \leq \|\hat{\x}_{t-1} - \x_{t-1}\|^2 + 2\eta_\x\langle\hat{\x}_{t-1} - \x_{t-1}, \gradx f(\x_{t-1}, \y_{t-1})\rangle + \eta_\x^2 (L^2+\sigma^2). 
\end{equation*}
We take the expectation of both sides of the above inequality (where the expectation is taken over the randomness in the selection of all previous samples). This yields
\begin{equation*}
\EE[\|\hat{\x}_{t-1} - \x_t\|^2] \leq \EE[\|\hat{\x}_{t-1} - \x_{t-1}\|^2] + 2\eta_\x \EE[\langle \hat{\x}_{t-1} - \x_{t-1}, \gradx f(\x_{t-1}, \y_{t-1})\rangle] + \eta_\x^2(L^2 + \sigma^2). 
\end{equation*}
Applying the same argument used for proving Eq.~\eqref{inequality:nc-smooth-descent-first}, we have
\begin{equation}\label{inequality:nc-smooth-descent-fourth}
\EE[\Phi_{1/2\ell}(\x_t)] \leq \EE[\Phi_{1/2\ell}(\x_{t-1})] + 2\eta_\x\ell\EE[\langle \hat{\x}_{t-1} - \x_{t-1}, \gradx f(\x_{t-1}, \y_{t-1})\rangle] + \eta_\x^2\ell(L^2 + \sigma^2). 
\end{equation}
Plugging Eq.~\eqref{inequality:nc-smooth-descent-second} and Eq.~\eqref{inequality:nc-smooth-descent-third} into Eq.~\eqref{inequality:nc-smooth-descent-fourth} yields the inequality for Algorithm~\ref{Algorithm:TTSGDA}.  
\end{proof}
\begin{lemma}\label{Lemma:nc-smooth-neighbor}
The iterates $\{\x_t\}_{t \geq 1}$ generated by Algorithm~\ref{Algorithm:TTGDA} satisfy for $\forall s \leq t-1$ that
\begin{equation*}
\Delta_{t-1} \leq \eta_\x L^2(2t-2s-1) + \tfrac{\ell}{2}(\|\y_{t-1} - \y^\star(\x_s)\|^2 - \|\y_t - \y^\star(\x_s)\|^2) + (f(\x_t, \y_t) - f(\x_{t-1}, \y_{t-1})). 
\end{equation*}
The iterates $\{\x_t\}_{t \geq 1}$ generated by Algorithm~\ref{Algorithm:TTSGDA} satisfy for $\forall s \leq t-1$ that
\begin{eqnarray*}
\Delta_{t-1} & \leq & \eta_\x (L^2+\sigma^2)(2t-2s-1) + \tfrac{1}{2\eta_\y}(\EE[\|\y_{t-1} - \y^\star(\x_s)\|^2] - \EE[\|\y_t - \y^\star(\x_s)\|^2]) \\
& & + \EE[f(\x_t, \y_t) - f(\x_{t-1}, \y_{t-1})] + \eta_\y\sigma^2.    
\end{eqnarray*}
\end{lemma}
\begin{proof}
We first prove the results for Algorithm~\ref{Algorithm:TTGDA}. By the definition of $\Delta_{t-1}$, we have
\begin{equation*}
\Delta_{t-1} \leq (f(\x_t, \y_t) - f(\x_{t-1}, \y_{t-1})) + \textbf{A} + \textbf{B} + \textbf{C}, 
\end{equation*}
where 
\begin{eqnarray*}
\textbf{A} & = & f(\x_{t-1}, \y_t) - f(\x_t, \y_t), \\ 
\textbf{B} & = & f(\x_{t-1}, \y^\star(\x_s)) - f(\x_{t-1}, \y_t), \\ 
\textbf{C} & = & f(\x_{t-1}, \y^\star(\x_{t-1})) - f(\x_{t-1}, \y^\star(\x_s)). 
\end{eqnarray*}
Since $f(\cdot, \y)$ is $L$-Lipschitz for each $\y \in \YCal$, we have $\|\gradx f(\x, \y)\| \leq L$ for each pair of $\x$ and $\y \in \YCal$. This yields  
\begin{equation*}
\textbf{A} \leq L\|\x_{t-1} - \x_t\| = \eta_\x L\|\gradx f(\x_{t-1}, \y_{t-1})\| \leq \eta_\x L^2. 
\end{equation*}
By using the same argument and that $f(\x_s, \y^\star(\x_s)) \geq f(\x_s, \y)$, we have
\begin{eqnarray*}
\textbf{C} & = & f(\x_{t-1}, \y^\star(\x_{t-1})) - f(\x_s, \y^\star(\x_{t-1})) + f(\x_s, \y^\star(\x_{t-1})) - f(\x_{t-1}, \y^\star(\x_s)) \\
& \leq & f(\x_{t-1}, \y^\star(\x_{t-1})) - f(\x_s, \y^\star(\x_{t-1})) + f(\x_s, \y^\star(\x_s)) - f(\x_{t-1}, \y^\star(\x_s)) \\
& \leq & 2L\|\x_{t-1} - \x_s\| \leq 2\eta_\x L \left(\sum_{k=s}^{t-1} \|\gradx f(\x_k, \y_k)\|\right) \leq 2\eta_\x L^2(t-s-1). 
\end{eqnarray*}
It suffices to bound the term $\textbf{B}$. Since $\eta_\y = \frac{1}{\ell}$, we have
\begin{equation*}
(\y - \y_t)^\top(\y_t - \y_{t-1} - \tfrac{1}{\ell}\grady f(\x_{t-1}, \y_{t-1})) \geq 0, \quad \textnormal{for all } \y \in \YCal. 
\end{equation*}
Rearranging the above inequality yields
\begin{equation*}
(\y - \y_t)^\top\grady f(\x_{t-1}, \y_{t-1}) \leq \tfrac{\ell}{2}\left(\|\y - \y_{t-1}\|^2 - \|\y - \y_t\|^2 - \|\y_t - \y_{t-1}\|^2\right). 
\end{equation*}
Since the function $f(\x_{t-1}, \cdot)$ is concave and $\ell$-smooth, we have
\begin{eqnarray*}
\lefteqn{(\y - \y_t)^\top\grady f(\x_{t-1}, \y_{t-1}) = (\y - \y_{t-1})^\top\grady f(\x_{t-1}, \y_{t-1}) + (\y_{t-1} - \y_t)^\top\grady f(\x_{t-1}, \y_{t-1})} \\
& & \geq  f(\x_{t-1}, \y) - f(\x_{t-1}, \y_{t-1}) + f(\x_{t-1}, \y_{t-1}) - f(\x_{t-1}, \y_t) - \tfrac{\ell}{2}\|\y_t - \y_{t-1}\|^2. 
\end{eqnarray*}
Putting these pieces together yields 
\begin{equation}\label{inequality:nc-smooth-neighbor-first}
f(\x_{t-1}, \y) - f(\x_{t-1}, \y_t) \leq \tfrac{\ell}{2}(\|\y - \y_{t-1}\|^2 - \|\y - \y_t\|^2). 
\end{equation}
Since $\y^\star(\x_s) \in \YCal$ for $s \leq t-1$, Eq.~\eqref{inequality:nc-smooth-neighbor-first} implies
\begin{equation*}
\textbf{B} \leq \tfrac{\ell}{2}(\|\y^\star(\x_s) - \y_{t-1}\|^2 - \|\y^\star(\x_s) - \y_t\|^2). 
\end{equation*}
Putting these pieces together yields the inequality for Algorithm~\ref{Algorithm:TTGDA}. 

We then prove the results for Algorithm~\ref{Algorithm:TTSGDA}. By the definition of $\Delta_{t-1}$, we have
\begin{equation*}
\Delta_{t-1} \leq \EE[f(\x_t, \y_t) - f(\x_{t-1}, \y_{t-1})] + \textbf{A} + \textbf{B} + \textbf{C}, 
\end{equation*}
where 
\begin{eqnarray*}
\textbf{A} & = & \EE[f(\x_{t-1}, \y_t) - f(\x_t, \y_t)], \\ 
\textbf{B} & = & \EE[f(\x_{t-1}, \y^\star(\x_s)) - f(\x_{t-1}, \y_t)], \\ 
\textbf{C} & = & \EE[f(\x_{t-1}, \y^\star(\x_{t-1})) - f(\x_{t-1}, \y^\star(\x_s))]. 
\end{eqnarray*}
Applying the same argument used for analyzing Algorithm~\ref{Algorithm:TTGDA}, we have
\begin{equation*}
\textbf{A} + \textbf{C} \leq \eta_\x L\left(\EE[\|\stocgx^{t-1}\|] + 2\left(\sum_{k=s}^{t-1} \EE[\|\stocgx^k\|]\right)\right). 
\end{equation*}
For $\forall k \in \{s, s+1, \ldots, t-1\}$, we have
\begin{equation*}
\|\stocgx^k\|^2 = \|\gradx f(\x_k, \y_k)\|^2 + 2\langle \stocgx^k - \gradx f(\x_k, \y_k), \gradx f(\x_k, \y_k)\rangle + \|\stocgx^k - \gradx f(\x_k, \y_k)\|^2.   
\end{equation*}
Conditioned on the iterate $(\x_k, \y_k)$, we take the expectation of both sides of this equation (where the expectation is taken over the randomness in the selection of samples $\{\xi_i^k\}_{i=1}^M$) and use Lemma~\ref{Lemma:SG-unbiased} with $M=1$ to yield
\begin{equation*}
\EE[\|\stocgx^k\|^2 \mid \x_k, \y_k] \leq \|\gradx f(\x_k, \y_k)\|^2 + \sigma^2 \leq L^2 + \sigma^2. 
\end{equation*}
We take the expectation of both sides of the above inequality (where the expectation is taken over the randomness in the selection of all previous samples). This yields
\begin{equation*}
\EE[\|\stocgx^k\|^2] \leq L^2 + \sigma^2. 
\end{equation*}
Putting these pieces together yields 
\begin{equation*}
\textbf{A} + \textbf{C} \leq \eta_\x (L^2+\sigma^2)(2t-2s-1). 
\end{equation*}
It suffices to bound the term $\textbf{B}$. By using the same argument, we have
\begin{equation*}
\eta_\y(\y - \y_t)^\top \stocgy^{t-1} \leq \tfrac{1}{2}\left(\|\y - \y_{t-1}\|^2 - \|\y - \y_t\|^2 - \|\y_t - \y_{t-1}\|^2\right). 
\end{equation*}
Using the Young's inequality, we have
\begin{eqnarray*}
\lefteqn{\eta_\y(\y - \y_t)^\top \stocgy^{t-1} \geq \eta_\y(\y - \y_{t-1})^\top \stocgy^{t-1}} \\
& & + \eta_\y(\y_{t-1} - \y_t)^\top \grady f(\x_{t-1}, \y_{t-1}) - \tfrac{1}{4}\|\y_t - \y_{t-1}\|^2 - \eta_\y^2\|\stocgy^{t-1} - \grady f(\x_{t-1}, \y_{t-1})\|^2.
\end{eqnarray*}
Conditioned on the iterate $(\x_{t-1}, \y_{t-1})$, we take the expectation of both sides of the above inequality (where the expectation is taken over the randomness in the selection of samples $\{\xi_i^k\}_{i=1}^M$) and use Lemma~\ref{Lemma:SG-unbiased} with $M=1$ to yield
\begin{eqnarray*}
\lefteqn{(\y - \y_{t-1})^\top \grady f(\x_{t-1}, \y_{t-1}) + \EE[(\y_{t-1} - \y_t)^\top \grady f(\x_{t-1}, \y_{t-1}) \mid \x_{t-1}, \y_{t-1}]} \\
& \leq &  \tfrac{1}{2\eta_\y}\left(\|\y - \y_{t-1}\|^2 - \EE[\|\y - \y_t\|^2 \mid \x_{t-1}, \y_{t-1}]\right) - \tfrac{1}{4\eta_\y}\EE[\|\y_t - \y_{t-1}\|^2 \mid \x_{t-1}, \y_{t-1}] + \eta_\y\sigma^2. 
\end{eqnarray*}
We take the expectation of both sides of the above inequality (where the expectation is taken over the randomness in the selection of all previous samples). This yields
\begin{equation*}
\EE[(\y - \y_t)^\top \grady f(\x_{t-1}, \y_{t-1})] \leq \tfrac{1}{2\eta_\y}\left(\EE[\|\y - \y_{t-1}\|^2] - \EE[\|\y - \y_t\|^2]\right) - \tfrac{1}{4\eta_\y}\EE[\|\y_t - \y_{t-1}\|^2] + \eta_\y\sigma^2. 
\end{equation*}
Since the function $f(\x_{t-1}, \cdot)$ is concave and $\ell$-smooth, we have
\begin{equation*}
(\y - \y_t)^\top\grady f(\x_{t-1}, \y_{t-1}) \geq  f(\x_{t-1}, \y) - f(\x_{t-1}, \y_t) - \tfrac{\ell}{2}\|\y_t - \y_{t-1}\|^2. 
\end{equation*}
Since $\eta_\y \leq \frac{1}{2\ell}$, we have
\begin{equation}\label{inequality:nc-smooth-neighbor-second}
\EE[f(\x_{t-1}, \y) - f(\x_{t-1}, \y_t)] \leq \tfrac{1}{2\eta_\y}(\EE[\|\y - \y_{t-1}\|^2] - \EE[\|\y - \y_t\|^2]) + \eta_\y\sigma^2. 
\end{equation}
Since $\y^\star(\x_s) \in \YCal$ for $s \leq t-1$, Eq.~\eqref{inequality:nc-smooth-neighbor-second} implies
\begin{equation*}
\textbf{B} \leq \tfrac{1}{2\eta_\y}(\EE[\|\y^\star(\x_s) - \y_{t-1}\|^2] - \EE[\|\y^\star(\x_s) - \y_t\|^2]) + \eta_\y\sigma^2. 
\end{equation*}
Putting these pieces together yields the inequality for Algorithm~\ref{Algorithm:TTSGDA}. 
\end{proof}
Without loss of generality, we let $B \leq T+1$ satisfy that $\frac{T+1}{B}$ is an integer. The following lemma bounds $\frac{1}{T+1}(\sum_{t=0}^T \Delta_t)$ for Algorithm~\ref{Algorithm:TTGDA} and~\ref{Algorithm:TTSGDA} using a novel localization technique. 
\begin{lemma}\label{Lemma:nc-smooth-obj}
The iterates $\{\x_t\}_{t \geq 1}$ generated by Algorithm~\ref{Algorithm:TTGDA} satisfy
\begin{equation*}
\tfrac{1}{T+1}\left(\sum_{t=0}^T \Delta_t\right) \leq \eta_\x (B+1) L^2 + \tfrac{\ell D^2}{2B} + \tfrac{\Delta_0}{T+1}. 
\end{equation*}
The iterates $\{\x_t\}_{t \geq 1}$ generated by Algorithm~\ref{Algorithm:TTSGDA} satisfy
\begin{equation*}
\tfrac{1}{T+1}\left(\sum_{t=0}^T \Delta_t\right) \leq \eta_\x (B+1) (L^2+\sigma^2) + \tfrac{D^2}{2B\eta_\y} + \eta_\y\sigma^2 + \tfrac{\Delta_0}{T+1}. 
\end{equation*}
\end{lemma}
\begin{proof}
We first prove the results for Algorithm~\ref{Algorithm:TTGDA}. Indeed, we can divide $\{\Delta_t\}_{t=0}^T$ into several blocks in which each block contains at most $B$ terms, given by
\begin{equation*}
\{\Delta_t\}_{t=0}^{B-1}, \{\Delta_t\}_{t=B}^{2B-1}, \ldots, \{\Delta_t\}_{T-B+1}^{T}. 
\end{equation*}
Then, we have
\begin{equation}\label{inequality:nc-smooth-obj-first}
\tfrac{1}{T+1}\left(\sum_{t=0}^T \Delta_t\right) \leq \tfrac{B}{T+1}\left(\sum_{j=0}^{\frac{T+1}{B}-1} \left(\tfrac{1}{B}\sum_{t= jB}^{(j+1)B-1} \Delta_t\right)\right). 
\end{equation}
Further, letting $s = 0$ in the first inequality in Lemma~\ref{Lemma:nc-smooth-neighbor} yields
\begin{equation*}
\sum_{t = 0}^{B-1} \Delta_t \leq \eta_\x L^2 B^2 + \tfrac{\ell D^2}{2} + (f(\x_B, \y_B) - f(\x_0, \y_0)).
\end{equation*}
Similarly, letting $s = jB$ yields, for $1 \leq j \leq \frac{T+1}{B} - 1$, 
\begin{equation*}
\sum_{t= jB}^{(j+1)B-1} \Delta_t \leq \eta_\x L^2 B^2 + \tfrac{\ell D^2}{2} + (f(\x_{jB+B}, \y_{jB+B}) - f(\x_{jB}, \y_{jB})). 
\end{equation*}
Plugging the above two inequalities into Eq.~\eqref{inequality:nc-smooth-obj-first} yields 
\begin{equation}\label{inequality:nc-smooth-obj-second}
\tfrac{1}{T+1}\left(\sum_{t=0}^T \Delta_t\right) \leq \eta_\x L^2B + \tfrac{\ell D^2}{2B} + \tfrac{f(\x_{T+1}, \y_{T+1}) - f(\x_0, \y_0)}{T+1}. 
\end{equation}
Since $f(\cdot, \y)$ is $L$-Lipschitz for each $\y \in \YCal$, we have $\|\gradx f(\x, \y)\| \leq L$ for each pair of $\x$ and $\y \in \YCal$. This yields  
\begin{eqnarray*}
f(\x_{T+1}, \y_{T+1}) - f(\x_0, \y_0) & = & f(\x_{T+1}, \y_{T+1}) - f(\x_0, \y_{T+1}) + f(\x_0, \y_{T+1}) - f(\x_0, \y_0) \\
& \leq & \eta_\x L^2(T+1) + \Delta_0. 
\end{eqnarray*}
Plugging this inequality into Eq.~\eqref{inequality:nc-smooth-obj-second} yields the inequality for Algorithm~\ref{Algorithm:TTGDA}. 

We then prove the results for Algorithm~\ref{Algorithm:TTSGDA}. Applying the same argument used for analyzing Algorithm~\ref{Algorithm:TTGDA} but with the second inequality in Lemma~\ref{Lemma:nc-smooth-neighbor}, we can obtain the inequality for Algorithm~\ref{Algorithm:TTSGDA}. 
\end{proof}

\paragraph{Proof of Theorem~\ref{Thm:nc-smooth}.} We first prove the results for Algorithm~\ref{Algorithm:TTGDA}. Summing the first inequality in Lemma~\ref{Lemma:nc-smooth-descent} over $t = 1, \ldots, T+1$ yields
\begin{equation*}
\Phi_{1/2\ell}(\x_{T+1}) \leq \Phi_{1/2\ell}(\x_0) + 2\eta_\x\ell\left(\sum_{t=0}^T \Delta_t\right) - \tfrac{\eta_\x}{4} \left(\sum_{t=0}^T \|\grad \Phi_{1/2\ell}(\x_t)\|^2\right) + \eta_\x^2 \ell L^2 (T+1). 
\end{equation*}
Combining the above inequality with the first inequality in Lemma~\ref{Lemma:nc-smooth-obj} yields
\begin{eqnarray*}
\Phi_{1/2\ell}(\x_{T+1}) & \leq & \Phi_{1/2\ell}(\x_0) + 2\eta_\x\ell(T+1)(\eta_\x L^2(B+1) + \tfrac{\ell D^2}{2B}) + 2\eta_\x\ell\Delta_0 \\
& & - \tfrac{\eta_\x}{4}\left(\sum_{t=0}^T \|\grad \Phi_{1/2\ell}(\x_t)\|^2\right) + \eta_\x^2 \ell L^2 (T+1). 
\end{eqnarray*}
By the definition of the function $\Phi_{1/2\ell}(\cdot)$, we have
\begin{equation*}
\begin{array}{rclcl}
\Phi_{1/2\ell}(\x_0) & = & \min_\w \Phi(\w) + \ell\|\w - \x_0\|^2 & \leq & \Phi(\x_0), \\
\Phi_{1/2\ell}(\x_{T+1}) & = & \min_\w \Phi(\w) + \ell\|\w - \x_{T+1}\|^2 & \geq & \min_\x \Phi(\x), 
\end{array}
\end{equation*}
which implies $\Phi_{1/2\ell}(\x_0) - \Phi_{1/2\ell}(\x_{T+1}) \leq \Delta_\Phi$. Putting these pieces together yields
\begin{equation*}
\tfrac{1}{T+1}\left(\sum_{t=0}^T \|\grad \Phi_{1/2\ell}(\x_t)\|^2\right) \leq \tfrac{4\Delta_\Phi}{\eta_\x (T+1)} + 8\ell(\eta_\x(B+1)L^2 + \tfrac{\ell D^2}{2B}) + \tfrac{8\ell\Delta_0}{T+1} + 4\eta_\x\ell L^2. 
\end{equation*}
Letting $B = \left\lfloor\frac{D}{2L}\sqrt{\frac{\ell}{\eta_\x}}\right\rfloor + 1$, we have
\begin{equation*}
\tfrac{1}{T+1}\left(\sum_{t=0}^T \|\grad \Phi_{1/2\ell}(\x_t)\|^2\right) \leq \tfrac{4\Delta_\Phi}{\eta_\x (T+1)} + \tfrac{8\ell\Delta_0}{T+1} + 16\ell LD\sqrt{\ell\eta_\x} + 20\eta_\x\ell L^2. 
\end{equation*}
By the definition of $\eta_\x$, we have
\begin{equation*}
\tfrac{1}{T+1}\left(\sum_{t=0}^T \|\grad \Phi_{1/2\ell}(\x_t)\|^2\right) \leq \tfrac{4\Delta_\Phi}{\eta_\x(T+1)} + \tfrac{8\ell\Delta_0}{T+1} + \tfrac{\epsilon^2}{2}. 
\end{equation*}
This implies that the number of gradient evaluations required by Algorithm \ref{Algorithm:TTGDA} to return an $\epsilon$-stationary point is
\begin{equation*}
O\left(\left(\tfrac{\ell L^2\Delta_\Phi}{\epsilon^4} + \tfrac{\ell\Delta_0}{\epsilon^2}\right)\max\left\{1, \tfrac{\ell^2D^2}{\epsilon^2}\right\}\right).  
\end{equation*}
We then prove the results for Algorithm~\ref{Algorithm:TTSGDA}. Applying the same argument used for analyzing Algorithm~\ref{Algorithm:TTGDA} but with the second inequalities in Lemmas~\ref{Lemma:nc-smooth-descent} and~\ref{Lemma:nc-smooth-obj}, we have
\begin{eqnarray*}
\tfrac{1}{T+1}\left(\sum_{t=0}^T \EE[\|\grad \Phi_{1/2\ell}(\x_t)\|^2]\right) & \leq & \tfrac{4\Delta_\Phi}{\eta_\x (T+1)} + 8\ell(\eta_\x (L^2+\sigma^2) (B+1) + \tfrac{D^2}{2B\eta_\y} + \eta_\y \sigma^2) \\
& & + \tfrac{8\ell\Delta_0}{T+1} + 4\eta_\x\ell(L^2 + \sigma^2). 
\end{eqnarray*}
Letting $B = \left\lfloor\frac{D}{2}\sqrt{\frac{1}{\eta_\x\eta_\y (L^2+\sigma^2)}}\right\rfloor + 1$, we have
\begin{equation*}
\tfrac{1}{T+1}\left(\sum_{t=0}^T \EE[\|\grad \Phi_{1/2\ell}(\x_t)\|^2]\right) \leq \tfrac{4\Delta_\Phi}{\eta_\x (T+1)} + \tfrac{8\ell\Delta_0}{T+1} + 16\ell D\sqrt{\tfrac{\eta_\x (L^2+\sigma^2)}{\eta_\y}} + 8\eta_\y\ell\sigma^2 + 20\eta_\x\ell(L^2 + \sigma^2). 
\end{equation*}
By the definition of $\eta_\x$ and $\eta_\y$, we have
\begin{equation*}
\tfrac{1}{T+1}\left(\sum_{t=0}^T \EE[\|\grad \Phi_{1/2\ell}(\x_t)\|^2] \right) \leq \tfrac{4\Delta_\Phi}{\eta_\x (T+1)} + \tfrac{8\ell\Delta_0}{T+1} + \tfrac{3\epsilon^2}{4}. 
\end{equation*}
This implies that the number of stochastic gradient evaluations required by Algorithm \ref{Algorithm:TTSGDA} to return an $\epsilon$-stationary point is
\begin{equation*}
O\left(\left(\tfrac{\ell\left(L^2 + \sigma^2\right)\widehat{\Delta}_\Phi}{\epsilon^4} + \tfrac{\ell\Delta_0}{\epsilon^2}\right)\max\left\{1, \tfrac{\ell^2D^2}{\epsilon^2}, \tfrac{\ell^2 D^2\sigma^2}{\epsilon^4}\right\}\right). 
\end{equation*}
This completes the proof. 

\section{Nonsmooth and Nonconvex-Strongly-Concave Setting}
We provide the proof for Theorem~\ref{Thm:nsc-nonsmooth}. Throughout this subsection, we set
\begin{equation*}
\eta_\x^t \equiv \eta_\x = \min\left\{\tfrac{\epsilon^2}{48\rho L^2}, \tfrac{\mu\epsilon^4}{4096\rho^2 L^4}, \tfrac{\mu\epsilon^4}{4096\rho^2 L^4 \log^2(1 + 4096\rho^2 L^4\mu^{-2}\epsilon^{-4})}\right\}, 
\end{equation*}
and
\begin{equation}\label{rule:stepsize-nonsmooth-app}
\eta_\y^t = \left\{
\begin{array}{cl}
\tfrac{1}{\mu t}, & \textnormal{if } 1 \leq t \leq B, \\
\tfrac{1}{\mu(t-B)}, & \textnormal{else if } B+1 \leq t \leq 2B, \\
\vdots & \vdots \\
\tfrac{1}{\mu(t-jB)}, & \textnormal{else if } jB+1 \leq t \leq (j+1)B, \\
\vdots & \vdots \\
\end{array}
\right.  \quad \textnormal{for } B = \left\lfloor\sqrt{\frac{1}{\mu \eta_\x}}\right\rfloor + 1,  
\end{equation}
for Algorithm~\ref{Algorithm:TTGDA}, and 
\begin{equation*}
\eta_\x^t \equiv \eta_\x = \min\left\{\tfrac{\epsilon^2}{48\rho(L^2 + \sigma^2)}, \tfrac{\mu\epsilon^4}{4096\rho^2 (L^2 + \sigma^2)^2}, \tfrac{\mu\epsilon^4}{4096\rho^2 (L^2 + \sigma^2)^2\log^2(1 + 4096\rho^2(L^2 + \sigma^2)^2 \mu^{-2}\epsilon^{-4})}\right\}, 
\end{equation*}
and $\eta_\y^t$ in Eq.~\eqref{rule:stepsize-nonsmooth-app} for Algorithm~\ref{Algorithm:TTSGDA}. We define $\Delta_t = \Phi(\x_t) - f(\x_t, \y_t)$ and $\Delta_t = \EE[\Phi(\x_t) - f(\x_t, \y_t)]$ for the iterates generated by Algorithms~\ref{Algorithm:TTGDA} and~\ref{Algorithm:TTSGDA} respectively.
\begin{lemma}\label{Lemma:nonsmooth-descent}
The iterates $\{\x_t\}_{t \geq 1}$ generated by Algorithm~\ref{Algorithm:TTGDA} satisfy
\begin{equation*}
\Phi_{1/2\rho}(\x_t) \leq \Phi_{1/2\rho}(\x_{t-1}) + 2\eta_\x\rho\Delta_{t-1} - \tfrac{1}{4}\eta_\x\|\grad \Phi_{1/2\rho}(\x_{t-1})\|^2 + \eta_\x^2 \rho L^2. 
\end{equation*}
The iterates $\{\x_t\}_{t \geq 1}$ generated by Algorithm~\ref{Algorithm:TTSGDA} satisfy
\begin{equation*}
\EE[\Phi_{1/2\rho}(\x_t)] \leq \EE[\Phi_{1/2\rho}(\x_{t-1})] + 2\eta_\x\rho\Delta_{t-1} - \tfrac{1}{4}\eta_\x\EE[\|\grad \Phi_{1/2\rho}(\x_{t-1})\|^2] + \eta_\x^2 \rho(L^2 + \sigma^2).
\end{equation*}
\end{lemma}
\begin{proof}
We first prove the results for Algorithm~\ref{Algorithm:TTGDA}. Applying the same argument used for proving Lemma~\ref{Lemma:nc-smooth-descent}, we have
\begin{equation}\label{inequality:nonsmooth-descent-first}
\Phi_{1/2\rho}(\x_t) \leq \Phi_{1/2\rho}(\x_{t-1}) + 2\eta_\x \rho\langle\hat{\x}_{t-1} - \x_{t-1}, \gx^{t-1}\rangle + \eta_\x^2\rho L^2.
\end{equation}
Since the function $f(\cdot, \y_{t-1})$ is $\rho$-weakly convex and $\gx^{t-1} \in \subgx f(\x_{t-1}, \y_{t-1})$, we have
\begin{equation}\label{inequality:nonsmooth-descent-second}
\langle\hat{\x}_{t-1} - \x_{t-1}, \gx^{t-1}\rangle \leq f(\hat{\x}_{t-1}, \y_{t-1}) - f(\x_{t-1}, \y_{t-1}) + \tfrac{\rho}{2}\|\hat{\x}_{t-1} - \x_{t-1}\|^2. 
\end{equation}
Applying the same argument used for proving Lemma~\ref{Lemma:nc-smooth-descent}, we have
\begin{equation}\label{inequality:nonsmooth-descent-third}
f(\hat{\x}_{t-1}, \y_{t-1}) - f(\x_{t-1}, \y_{t-1}) + \tfrac{\rho}{2}\|\hat{\x}_{t-1} - \x_{t-1}\|^2 \leq \Phi(\x_{t-1}) - f(\x_{t-1}, \y_{t-1}) - \tfrac{1}{8\rho}\|\grad\Phi_{1/2\rho}(\x_{t-1})\|^2.
\end{equation}
Plugging Eq.~\eqref{inequality:nonsmooth-descent-second} and Eq.~\eqref{inequality:nonsmooth-descent-third} into Eq.~\eqref{inequality:nonsmooth-descent-first} yields the inequality for Algorithm~\ref{Algorithm:TTGDA}. 

We then prove the results for Algorithm~\ref{Algorithm:TTSGDA}. Applying the same argument used for proving Lemma~\ref{Lemma:nc-smooth-descent}, we have
\begin{equation*}
\EE[\Phi_{1/2\rho}(\x_t)] \leq \EE[\Phi_{1/2\rho}(\x_{t-1})] + 2\eta_\x\rho\EE[\langle \hat{\x}_{t-1} - \x_{t-1}, \EE[\stocgx^{t-1} \mid \x_{t-1}, \y_{t-1}]\rangle] + \eta_\x^2\rho(L^2 + \sigma^2). 
\end{equation*}
Since $\EE[\stocgx^{t-1} \mid \x_{t-1}, \y_{t-1}] \in \subgx f(\x_{t-1}, \y_{t-1})$, we have 
\begin{equation*}
\EE[\langle \hat{\x}_{t-1} - \x_{t-1}, \EE[\stocgx^{t-1} \mid \x_{t-1}, \y_{t-1}]\rangle] \leq \EE[f(\hat{\x}_{t-1}, \y_{t-1})] - \EE[f(\x_{t-1}, \y_{t-1})] + \tfrac{\rho}{2}\EE[\|\hat{\x}_{t-1} - \x_{t-1}\|^2]. 
\end{equation*}
Combining two above inequalities with Eq.~\eqref{inequality:nonsmooth-descent-third} yields the inequality for Algorithm~\ref{Algorithm:TTSGDA}. 
\end{proof}
\begin{lemma}\label{Lemma:nsc-nonsmooth-neighbor}
The iterates $\{\x_t\}_{t \geq 1}$ generated by Algorithm~\ref{Algorithm:TTGDA} satisfy for $\forall s \leq t-1$ that
\begin{equation*}
\Delta_{t-1} \leq 2\eta_\x L^2(t-s-1) + (\tfrac{1}{2\eta_\y^{t-1}} - \tfrac{\mu}{2})\|\y_{t-1} - \y^\star(\x_s)\|^2 - \tfrac{1}{2\eta_\y^{t-1}}\|\y_t - \y^\star(\x_s)\|^2 + \tfrac{1}{2}\eta_\y^{t-1}L^2. 
\end{equation*}
The iterates $\{\x_t\}_{t \geq 1}$ generated by Algorithm~\ref{Algorithm:TTSGDA} satisfy for $\forall s \leq t-1$ that
\begin{equation*}
\Delta_{t-1} \leq 2\eta_\x (L^2+\sigma^2)(t-s-1) + (\tfrac{1}{2\eta_\y^{t-1}} - \tfrac{\mu}{2})\EE[\|\y_{t-1} - \y^\star(\x_s)\|^2] - \tfrac{1}{2\eta_\y^{t-1}}\EE[\|\y_t - \y^\star(\x_s)\|^2] + \eta_\y^{t-1}(L^2 + \sigma^2).     
\end{equation*}
\end{lemma}
\begin{proof}
We first prove the results for Algorithm~\ref{Algorithm:TTGDA}. By the definition of $\Delta_{t-1}$, we have
\begin{equation*}
\Delta_{t-1} \leq \textbf{A} + \textbf{B}, 
\end{equation*}
where 
\begin{eqnarray*}
\textbf{A} & = & f(\x_{t-1}, \y^\star(\x_{t-1})) - f(\x_{t-1}, \y^\star(\x_s)), \\
\textbf{B} & = & f(\x_{t-1}, \y^\star(\x_s)) - f(\x_{t-1}, \y_{t-1}). 
\end{eqnarray*}
Since $f(\cdot, \y)$ is $L$-Lipschitz for each $\y \in \YCal$, we have $\|\g\| \leq L$ for all $\g \in \subgx f(\x, \y)$ and each pair of $\x$ and $\y \in \YCal$. Applying the same argument used for proving Lemma~\ref{Lemma:nc-smooth-neighbor} yields
\begin{equation*}
\textbf{A} \leq 2\eta_\x L^2(t-s-1). 
\end{equation*}
It suffices to bound the term $\textbf{B}$. Indeed, we have
\begin{equation*}
(\y - \y_t)^\top\gy^{t-1} \leq \tfrac{1}{2\eta_\y^{t-1}}\left(\|\y - \y_{t-1}\|^2 - \|\y - \y_t\|^2 - \|\y_t - \y_{t-1}\|^2\right), \quad \textnormal{for all } \y \in \YCal. 
\end{equation*}
Since the function $f(\x_{t-1}, \cdot)$ is $\mu$-strongly concave and $\gy^{t-1} \in \subgy f(\x_{t-1}, \y_{t-1})$, we have
\begin{equation*}
(\y - \y_{t-1})^\top\gy^{t-1} \geq f(\x_{t-1}, \y) - f(\x_{t-1}, \y_{t-1}) + \tfrac{\mu}{2}\|\y - \y_{t-1}\|^2. 
\end{equation*}
Since the function $f(\x_{t-1}, \cdot)$ is $L$-Lipschitz, we have $\|\gy^{t-1}\| \leq L$. Combining this fact with Young's inequality yields
\begin{equation*}
(\y_{t-1} - \y_t)^\top\gy^{t-1} \geq -\tfrac{1}{2\eta_\y^{t-1}}\|\y_t - \y_{t-1}\|^2 - \tfrac{1}{2}\eta_\y^{t-1}L^2. 
\end{equation*}
Putting these pieces together yields 
\begin{equation}\label{inequality:nsc-nonsmooth-neighbor-first}
f(\x_{t-1}, \y) - f(\x_{t-1}, \y_{t-1}) \leq \tfrac{1}{2\eta_\y^{t-1}}(\|\y - \y_{t-1}\|^2 - \|\y - \y_t\|^2) - \tfrac{\mu}{2}\|\y - \y_{t-1}\|^2 + \tfrac{1}{2}\eta_\y^{t-1}L^2. 
\end{equation}
Since $\y^\star(\x_s) \in \YCal$ for $s \leq t-1$, Eq.~\eqref{inequality:nsc-nonsmooth-neighbor-first} implies
\begin{equation*}
\textbf{B} \leq \tfrac{1}{2\eta_\y^{t-1}}(\|\y - \y_{t-1}\|^2 - \|\y - \y_t\|^2) - \tfrac{\mu}{2}\|\y - \y_{t-1}\|^2 + \tfrac{1}{2}\eta_\y^{t-1}L^2. 
\end{equation*}
Putting these pieces together yields the inequality for Algorithm~\ref{Algorithm:TTGDA}. 

We then prove the results for Algorithm~\ref{Algorithm:TTSGDA}. By the definition of $\Delta_{t-1}$, we have
\begin{equation*}
\Delta_{t-1} \leq \textbf{A} + \textbf{B}, 
\end{equation*}
where 
\begin{eqnarray*}
\textbf{A} & = & \EE[f(\x_{t-1}, \y^\star(\x_{t-1})) - f(\x_{t-1}, \y^\star(\x_s))], \\ 
\textbf{B} & = & \EE[f(\x_{t-1}, \y^\star(\x_s)) - f(\x_{t-1}, \y_{t-1})]. 
\end{eqnarray*}
Applying the same argument used for proving Lemma~\ref{Lemma:nc-smooth-neighbor} yields
\begin{equation*}
\textbf{A} \leq 2\eta_\x (L^2 + \sigma^2)(t-s-1). 
\end{equation*}
It suffices to bound the term $\textbf{B}$. Indeed, we have
\begin{equation*}
(\y - \y_t)^\top \stocgy^{t-1} \leq \tfrac{1}{2\eta_\y^{t-1}}\left(\|\y - \y_{t-1}\|^2 - \|\y - \y_t\|^2 - \|\y_t - \y_{t-1}\|^2\right), \quad \textnormal{for all } \y \in \YCal. 
\end{equation*}
Using Young's inequality, we have
\begin{eqnarray*}
\lefteqn{(\y - \y_t)^\top \stocgy^{t-1} \geq (\y - \y_{t-1})^\top \stocgy^{t-1}} \\
& & - \tfrac{1}{2\eta_\y^{t-1}}\|\y_t - \y_{t-1}\|^2 - \eta_\y^{t-1}\|\EE[\stocgy^{t-1} \mid \x_{t-1}, \y_{t-1}]\|^2 - \eta_\y^{t-1}\|\stocgy^{t-1} - \EE[\stocgy^{t-1} \mid \x_{t-1}, \y_{t-1}]\|^2.
\end{eqnarray*}
Conditioned on the iterate $(\x_{t-1}, \y_{t-1})$, we take the expectation of both sides of the above inequality (where the expectation is taken over the randomness in the selection of samples $\{\xi_i^k\}_{i=1}^M$) and use Lemma~\ref{Lemma:SG-unbiased} with $M=1$ to yield
\begin{eqnarray*}
\lefteqn{(\y - \y_{t-1})^\top \EE[\stocgy^{t-1} \mid \x_{t-1}, \y_{t-1}]} \\
& \leq &  \tfrac{1}{2\eta_\y^{t-1}}\left(\|\y - \y_{t-1}\|^2 - \EE[\|\y - \y_t\|^2 \mid \x_{t-1}, \y_{t-1}]\right) + \eta_\y^{t-1}\|\EE[\stocgy^{t-1} \mid \x_{t-1}, \y_{t-1}]\|^2 + \eta_\y^{t-1}\sigma^2. 
\end{eqnarray*}
Note that $\EE[\stocgy^{t-1} \mid \x_{t-1}, \y_{t-1}] \in \subgy f(\x_{t-1}, \y_{t-1})$ (cf. Lemma~\ref{Lemma:SG-unbiased}). Thus, by using the same argument used for analyzing Algorithm~\ref{Algorithm:TTGDA}, we have
\begin{eqnarray*}
f(\x_{t-1}, \y) - f(\x_{t-1}, \y_{t-1}) & \leq & \tfrac{1}{2\eta_\y^{t-1}}\left(\|\y - \y_{t-1}\|^2 - \EE[\|\y - \y_t\|^2 \mid \x_{t-1}, \y_{t-1}]\right) \\ 
& & - \tfrac{\mu}{2}\|\y - \y_{t-1}\|^2 + \eta_\y^{t-1}(L^2+\sigma^2). 
\end{eqnarray*}
We take the expectation of both sides of the above inequality (where the expectation is taken over the randomness in the selection of all previous samples). This yields
\begin{equation}\label{inequality:nsc-nonsmooth-neighbor-second}
\EE[f(\x_{t-1}, \y) - f(\x_{t-1}, \y_{t-1})] \leq (\tfrac{1}{2\eta_\y^{t-1}}-\tfrac{\mu}{2})\EE[\|\y - \y_{t-1}\|^2] - \tfrac{1}{2\eta_\y^{t-1}}\EE[\|\y - \y_t\|^2] + \eta_\y^{t-1}(L^2 + \sigma^2). 
\end{equation}
Since $\y^\star(\x_s) \in \YCal$ for $s \leq t-1$, Eq.~\eqref{inequality:nsc-nonsmooth-neighbor-second} implies
\begin{equation*}
\textbf{B} \leq (\tfrac{1}{2\eta_\y^{t-1}}-\tfrac{\mu}{2})\EE[\|\y^\star(\x_s) - \y_{t-1}\|^2] - \tfrac{1}{2\eta_\y^{t-1}}\EE[\|\y^\star(\x_s) - \y_t\|^2] + \eta_\y^{t-1}(L^2 + \sigma^2). 
\end{equation*}
Putting these pieces together yields the inequality for Algorithm~\ref{Algorithm:TTSGDA}. 
\end{proof}
Without loss of generality, we let $B \leq T+1$ satisfy that $\frac{T+1}{B}$ is an integer. The following lemma bounds $\frac{1}{T+1}(\sum_{t=0}^T \Delta_t)$ for Algorithm~\ref{Algorithm:TTGDA} and~\ref{Algorithm:TTSGDA} using the same localization technique.  
\begin{lemma}\label{Lemma:nsc-nonsmooth-obj}
The iterates $\{\x_t\}_{t \geq 1}$ generated by Algorithm~\ref{Algorithm:TTGDA} satisfy
\begin{equation*}
\tfrac{1}{T+1}\left(\sum_{t=0}^T \Delta_t\right) \leq \eta_\x BL^2 + \tfrac{L^2(1+\log(B))}{2\mu B}. 
\end{equation*}
The iterates $\{\x_t\}_{t \geq 1}$ generated by Algorithm~\ref{Algorithm:TTSGDA} satisfy
\begin{equation*}
\tfrac{1}{T+1}\left(\sum_{t=0}^T \Delta_t\right) \leq \eta_\x B(L^2+\sigma^2) + \tfrac{(L^2 + \sigma^2)(1+\log(B))}{2\mu B}. 
\end{equation*}
\end{lemma}
\begin{proof}
Applying the same argument used for proving Lemma~\ref{Lemma:nc-smooth-obj} but with the two inequalities in Lemma~\ref{Lemma:nsc-nonsmooth-neighbor}, we can obtain the inequalities for Algorithm~\ref{Algorithm:TTGDA} and~\ref{Algorithm:TTSGDA}. 
\end{proof}

\paragraph{Proof of Theorem~\ref{Thm:nsc-nonsmooth}.} We first prove the results for Algorithm~\ref{Algorithm:TTGDA}. Applying the same argument used for proving Theorem~\ref{Thm:nc-smooth} but with the first inequalities in Lemmas~\ref{Lemma:nonsmooth-descent} and~\ref{Lemma:nsc-nonsmooth-obj}, we have
\begin{equation*}
\tfrac{1}{T+1}\left(\sum_{t=0}^T \|\grad \Phi_{1/2\rho}(\x_t)\|^2\right) \leq \tfrac{4\Delta_\Phi}{\eta_\x (T+1)} + 8\rho L^2(\eta_\x B + \tfrac{1 + \log(B)}{2\mu B}) + 4\eta_\x\rho L^2. 
\end{equation*}
Letting $B = \left\lfloor\sqrt{\frac{1}{\mu \eta_\x}}\right\rfloor + 1$, we have
\begin{equation*}
\tfrac{1}{T+1}\left(\sum_{t=0}^T \|\grad \Phi_{1/2\rho}(\x_t)\|^2\right) \leq \tfrac{4\Delta_\Phi}{\eta_\x (T+1)} + 16\rho L^2 \sqrt{\tfrac{\eta_\x}{\mu}} + 4\rho L^2 \sqrt{\tfrac{\eta_\x}{\mu}}\log\left(1 + \tfrac{1}{\mu \eta_\x}\right) + 12\eta_\x\rho L^2. 
\end{equation*}
By the definition of $\eta_\x$, we have
\begin{equation*}
\tfrac{1}{T+1}\left(\sum_{t=0}^T \|\grad \Phi_{1/2\rho}(\x_t)\|^2 \right) \leq \tfrac{4\Delta_\Phi}{\eta_\x (T+1)} + \tfrac{3\epsilon^2}{4}. 
\end{equation*}
This implies that the number of gradient evaluations required by Algorithm \ref{Algorithm:TTGDA} to return an $\epsilon$-stationary point is
\begin{equation*}
O\left(\tfrac{\rho L^2\Delta_\Phi}{\epsilon^4}\max\left\{1, \tfrac{\rho L^2}{\mu\epsilon^2}, \tfrac{\rho L^2}{\mu\epsilon^2}\log^2\left(1 + \tfrac{\rho^2 L^4}{\mu^2 \epsilon^4}\right)\right\}\right). 
\end{equation*}
We then prove the results for Algorithm~\ref{Algorithm:TTSGDA}. Applying the same argument used for analyzing Algorithm~\ref{Algorithm:TTGDA} but with the second inequalities in Lemmas~\ref{Lemma:nonsmooth-descent} and~\ref{Lemma:nsc-nonsmooth-obj} and using $B = \lfloor\sqrt{\frac{1}{\mu \eta_\x}}\rfloor + 1$ and the definition of $\eta_\x$, we have
\begin{equation*}
\tfrac{1}{T+1}\left(\sum_{t=0}^T \EE[\|\grad \Phi_{1/2\rho}(\x_t)\|^2]\right) \leq \tfrac{4\Delta_\Phi}{\eta_\x(T+1)} + \tfrac{3\epsilon^2}{4}. 
\end{equation*}
This implies that the number of stochastic gradient evaluations required by Algorithm \ref{Algorithm:TTSGDA} to return an $\epsilon$-stationary point is
\begin{equation*}
O\left(\tfrac{\rho(L^2 + \sigma^2)\Delta_\Phi}{\epsilon^4}\max\left\{1, \tfrac{\rho(L^2 + \sigma^2)}{\mu\epsilon^2}, \tfrac{\rho(L^2 + \sigma^2)}{\mu\epsilon^2}\log^2\left(1 + \tfrac{\rho^2(L^2 + \sigma^2)^2}{\mu^2 \epsilon^4}\right)\right\}\right). 
\end{equation*}
This completes the proof. 

\section{Nonsmooth and Nonconvex-Concave Setting}
We provide the proof for Theorem~\ref{Thm:nc-nonsmooth}. Throughout this subsection, we set
\begin{equation*}
\eta_\x^t \equiv \eta_\x = \min\left\{\tfrac{\epsilon^2}{48\rho L^2}, \tfrac{\epsilon^6}{65536\rho^3 L^4 D^2}\right\}, \quad \eta_\y^t \equiv \eta_\y = \tfrac{\epsilon^2}{16\rho L^2}, 
\end{equation*}
for Algorithm~\ref{Algorithm:TTGDA}, and 
\begin{equation*}
\eta_\x^t \equiv \eta_\x = \min\left\{\tfrac{\epsilon^2}{48\rho(L^2 + \sigma^2)}, \tfrac{\epsilon^6}{131072\rho^3 (L^2+\sigma^2)^2 D^2}\right\}, \quad \eta_\y^t \equiv \eta_\y = \tfrac{\epsilon^2}{32\rho(L^2 + \sigma^2)}, 
\end{equation*}
for Algorithm~\ref{Algorithm:TTSGDA}. We define $\Delta_t = \Phi(\x_t) - f(\x_t, \y_t)$ and $\Delta_t = \EE[\Phi(\x_t) - f(\x_t, \y_t)]$ for the iterates generated by Algorithms~\ref{Algorithm:TTGDA} and~\ref{Algorithm:TTSGDA} respectively.
\begin{lemma}\label{Lemma:nc-nonsmooth-neighbor}
The iterates $\{\x_t\}_{t \geq 1}$ generated by Algorithm~\ref{Algorithm:TTGDA} satisfy for $\forall s \leq t-1$ that
\begin{equation*}
\Delta_{t-1} \leq 2\eta_\x L^2(t-s-1) + \tfrac{1}{2\eta_\y}(\|\y_{t-1} - \y^\star(\x_s)\|^2 - \|\y_t - \y^\star(\x_s)\|^2) + \tfrac{1}{2}\eta_\y L^2. 
\end{equation*}
The iterates $\{\x_t\}_{t \geq 1}$ generated by Algorithm~\ref{Algorithm:TTSGDA} satisfy for $\forall s \leq t-1$ that
\begin{equation*}
\Delta_{t-1} \leq 2\eta_\x (L^2+\sigma^2)(t-s-1) + \tfrac{1}{2\eta_\y}(\EE[\|\y_{t-1} - \y^\star(\x_s)\|^2] - \EE[\|\y_t - \y^\star(\x_s)\|^2]) + \eta_\y(L^2 + \sigma^2).     
\end{equation*}
\end{lemma}
\begin{proof}
The proof is the same as that of Lemma~\ref{Lemma:nsc-nonsmooth-neighbor} with $\mu = 0$. 
\end{proof}
\begin{lemma}\label{Lemma:nc-nonsmooth-obj}
The iterates $\{\x_t\}_{t \geq 1}$ generated by Algorithm~\ref{Algorithm:TTGDA} satisfy
\begin{equation*}
\tfrac{1}{T+1}\left(\sum_{t=0}^T \Delta_t\right) \leq \eta_\x B L^2 + \tfrac{D^2}{2B\eta_\y} + \tfrac{\eta_\y L^2}{2}.  
\end{equation*}
The iterates $\{\x_t\}_{t \geq 1}$ generated by Algorithm~\ref{Algorithm:TTSGDA} satisfy
\begin{equation*}
\tfrac{1}{T+1}\left(\sum_{t=0}^T \Delta_t\right) \leq \eta_\x B (L^2+\sigma^2) + \tfrac{D^2}{2B\eta_\y} + \eta_\y(L^2 + \sigma^2). 
\end{equation*}
\end{lemma}
\begin{proof}
Applying the same argument used for proving Lemma~\ref{Lemma:nc-smooth-obj} but with the two inequalities in Lemma~\ref{Lemma:nc-nonsmooth-neighbor}, we can obtain the inequalities for Algorithm~\ref{Algorithm:TTGDA} and~\ref{Algorithm:TTSGDA}. 
\end{proof}
\paragraph{Proof of Theorem~\ref{Thm:nc-nonsmooth}.} We first prove the results for Algorithm~\ref{Algorithm:TTGDA}. Applying the same argument used for proving Theorem~\ref{Thm:nc-smooth} but with the first inequalities in Lemmas~\ref{Lemma:nonsmooth-descent} and~\ref{Lemma:nc-nonsmooth-obj}, we have
\begin{equation*}
\tfrac{1}{T+1}\left(\sum_{t=0}^T \|\grad \Phi_{1/2\rho}(\x_t)\|^2\right) \leq \tfrac{4\Delta_\Phi}{\eta_\x (T+1)} + 8\rho(\eta_\x BL^2 + \tfrac{D^2}{2B\eta_\y} + \tfrac{\eta_\y L^2}{2}) + 4\eta_\x\rho L^2. 
\end{equation*}
Letting $B = \left\lfloor\frac{D}{2L}\sqrt{\frac{1}{\eta_\x\eta_\y}}\right\rfloor + 1$, we have
\begin{equation*}
\tfrac{1}{T+1}\left(\sum_{t=0}^T \|\grad \Phi_{1/2\rho}(\x_t)\|^2\right) \leq \tfrac{4\Delta_\Phi}{\eta_\x (T+1)} + 16\rho LD\sqrt{\tfrac{\eta_\x}{\eta_\y}} + 4\eta_\y\rho L^2 + 12\eta_\x\rho L^2. 
\end{equation*}
By the definition of $\eta_\x$ and $\eta_\y$, we have
\begin{equation*}
\tfrac{1}{T+1}\left(\sum_{t=0}^T \|\grad \Phi_{1/2\rho}(\x_t)\|^2 \right) \leq \tfrac{4\Delta_\Phi}{\eta_\x (T+1)} + \tfrac{3\epsilon^2}{4}. 
\end{equation*}
This implies that the number of gradient evaluations required by Algorithm \ref{Algorithm:TTGDA} to return an $\epsilon$-stationary point is
\begin{equation*}
O\left(\tfrac{\rho L^2\Delta_\Phi}{\epsilon^4}\max\left\{1, \tfrac{\rho^2 L^2D^2}{\epsilon^4}\right\}\right). 
\end{equation*}
We now prove the results for Algorithm~\ref{Algorithm:TTSGDA}. Applying the same argument used for analyzing Algorithm~\ref{Algorithm:TTGDA} but with the second inequalities in Lemmas~\ref{Lemma:nonsmooth-descent} and~\ref{Lemma:nc-nonsmooth-obj}, we have
\begin{equation*}
\tfrac{1}{T+1}\left(\sum_{t=0}^T \EE[\|\grad \Phi_{1/2\rho}(\x_t)\|^2]\right) \leq \tfrac{4\Delta_\Phi}{\eta_\x (T+1)} + 8\rho(\eta_\x B(L^2 + \sigma^2) + \tfrac{D^2}{2B\eta_\y} + \eta_\y(L^2 + \sigma^2)) + 4\eta_\x\rho(L^2 + \sigma^2). 
\end{equation*}
Letting $B = \left\lfloor \frac{D}{2}\sqrt{\frac{1}{\eta_\x\eta_\y(L^2+\sigma^2)}}\right\rfloor + 1$, we have
\begin{equation*}
\tfrac{1}{T+1}\left(\sum_{t=0}^T \EE[\|\grad \Phi_{1/2\rho}(\x_t)\|^2]\right) \leq \tfrac{4\Delta_\Phi}{\eta_\x (T+1)} + 16\rho D\sqrt{\tfrac{\eta_\x (L^2+\sigma^2)}{\eta_\y}} + 8\eta_\y\rho(L^2 + \sigma^2) + 12\eta_\x\rho(L^2 + \sigma^2). 
\end{equation*}
By the definition of $\eta_\x$ and $\eta_\y$, we have
\begin{equation*}
\tfrac{1}{T+1}\left(\sum_{t=0}^T \EE[\|\grad \Phi_{1/2\rho}(\x_t)\|^2]\right) \leq \tfrac{4\Delta_\Phi}{\eta_\x (T+1)} + \tfrac{3\epsilon^2}{4}. 
\end{equation*}
This implies that the number of stochastic gradient evaluations required by Algorithm \ref{Algorithm:TTSGDA} to return an $\epsilon$-stationary point is
\begin{equation*}
O\left(\tfrac{\rho\left(L^2 + \sigma^2\right)\Delta_\Phi}{\epsilon^4}\max\left\{1, \tfrac{\rho^2(L^2 + \sigma^2)D^2}{\epsilon^4}\right\}\right). 
\end{equation*}
This completes the proof.

\end{document}